\documentclass{article}

\usepackage{natbib}
\usepackage{amsmath}
\usepackage{amssymb}
\usepackage{algorithmic}
\usepackage[ruled, linesnumbered]{algorithm2e}
\usepackage{wrapfig}
\usepackage{caption}
\usepackage{lipsum}
\usepackage[export]{adjustbox}
\usepackage{multicol}

\usepackage{amsthm}
\theoremstyle{plain}
\newtheorem{theorem}{Theorem}[section]
\newtheorem{proposition}[theorem]{Proposition}
\newtheorem{lemma}[theorem]{Lemma}
\newtheorem{corollary}[theorem]{Corollary}
\theoremstyle{definition}

\newcommand{\cdiff}{black}

\usepackage[final]{neurips_2022}

\usepackage[utf8]{inputenc} 
\usepackage[T1]{fontenc}    
\usepackage{hyperref}       
\usepackage{url}            
\usepackage{booktabs}       
\usepackage{amsfonts}       
\usepackage{nicefrac}       
\usepackage{microtype}      
\usepackage{xcolor}         
\usepackage{enumitem}
\usepackage{graphicx}
\usepackage{multirow}

\title{Optimal Scaling for Locally Balanced Proposals in Discrete Spaces}

\author{%
  Haoran Sun\thanks{Work done during an internship at Google.} \\
  Georgia Tech \\
   \texttt{hsun349@gatech.edu} \\
   \And
   Hanjun Dai \\
   Google Brain \\
   \texttt{hadai@google.com} \\
   \And
   Dale Schuurmans \\
   Google Brain, U of Alberta \\
   \texttt{schuurmans@google.com} \\
}

\begin{document}

\maketitle

\begin{abstract}
Optimal scaling has been well studied for Metropolis-Hastings (M-H) algorithms in continuous spaces, but a similar understanding has been lacking in discrete spaces.
Recently, a family of locally balanced proposals (LBP) for discrete spaces has been proved to be asymptotically optimal, but the question of optimal scaling has remained open.
In this paper, we establish, for the first time, that the efficiency of M-H in discrete spaces can also be characterized by an asymptotic acceptance rate that is independent of the target distribution. 
Moreover, we verify, both theoretically and empirically, that the optimal acceptance rates for LBP and random walk Metropolis (RWM) are $0.574$ and $0.234$ respectively. 
These results also help establish that LBP is asymptotically $O(N^\frac{2}{3})$ more efficient than RWM with respect to model dimension $N$. 
Knowledge of the optimal acceptance rate allows one to automatically tune the neighborhood size of a proposal distribution in a discrete space, directly analogous to step-size control in continuous spaces.
We demonstrate empirically that such adaptive M-H sampling can robustly improve sampling in a variety of target distributions in discrete spaces, including training deep energy based models.
\end{abstract}



\section{Introduction}
The Markov Chain Monte Carlo (MCMC) algorithm is one of the most widely used methods for sampling from  intractable distributions~\citep{robert2013monte}. An important class of MCMC algorithms is Metropolis-Hastings (M-H) \citep{metropolis1953equation, hastings1970monte}, where  new states are generated from a proposal distribution followed by a M-H test.
The efficiency for M-H algorithms depends critically on the proposal distribution. For example, gradient based methods, such as the Metropolis Adjusted Langevin Algorithm (MALA) \citep{rossky1978brownian}, Hamiltonian Monte Carlo (HMC) \citep{neal2011mcmc}, and their variants \cite{girolami2011riemann, hoffman2014no} substantially improve the performance of M-H algorithms in theory and in practice, compared to naive Random Walk Metropolis (RWM), by leveraging gradient information to guide the proposal distribution \citep{roberts2001optimal}.

Despite many advances, progress in gradient based methods has generally focused on continuous spaces. However, \cite{zanella2020informed} recently proposed a general framework of locally balanced proposals (LBP) for discrete spaces, where a proposal distribution is designed to utilize probability changes between states. Subsequently, \cite{grathwohl2021oops} accelerated the sampler by using gradient information to approximate the probability change. 
In empirical evaluations, similar to gradient based samplers in continuous spaces, LBP significantly outperforms RWM and other samplers in discrete spaces. However, both \cite{zanella2020informed} and \cite{grathwohl2021oops} constrain the proposal distribution to lie within a 1-Hamming ball; i.e., only one site of the state variable is allowed to change per M-H step. Such a restricted update reduces the efficiency of the sampler. 
\cite{sun2021path} noticed this problem and modified the proposal distribution to allow  multiple sites to be changed per M-H step. Although such larger updates significantly improve efficiency, \cite{sun2021path} do not show how to determine the update size, leaving the number of sites  updated in an M-H step as a hyperparameter to tune. 

In continuous spaces, \textcolor{\cdiff}{the scale of the proposal distribution} is known to be a critical hyperparameter for obtaining an efficient M-H sampler. For example, consider a Gaussian proposal $\mathcal{N}(x, \sigma^2)$ for modifying a current state $x$ \textcolor{\cdiff}{with scale $\sigma$}. If $\sigma$ is too small, the Markov chain will converge slowly since its increments will be small. Conversely, if $\sigma$ is too large, the M-H test will reject too high a proportion of proposed updates. A significant literature has studied optimal scaling for gradient based methods in continuous spaces \citep{gelman1997weak, roberts1998optimal, roberts2001optimal, beskos2013optimal}, showing that the optimal scaling can be adaptively tuned w.r.t.\ the acceptance rate, independent of the target distribution. Such results suggest a direction for solving the optimal scaling problem for LBP. However, the underlying techniques for approximating a diffusion process cannot be directly applied to LBP given its discrete nature.

In this work, we consider an asymptotic analysis as the dimension of the discrete model, $N$, converges to infinity. Starting with a product distribution, we prove that the asymptotic efficiency of LBP in discrete spaces is $2R\Phi(-\frac{1}{2}\lambda_1 R^\frac{3}{2}/N)$ with an asymptotic acceptance rate of $2\Phi(-\frac{1}{2}\lambda_1 R^\frac{3}{2}/N)$, \textcolor{\cdiff}{where the scale $R$ represents the number of sites to update per M-H step}. Therefore, the asymptotically optimal scale of the proposal distribution is $R = O(N^\frac{2}{3})$ with an asymptotically optimal acceptance rate of $0.574$, independent of the target distribution. 
Moreover, for RWM in a discrete space, we show that the asymptotic efficiency and acceptance rate are $2R\Phi(-\frac{1}{2}\lambda_2 R^\frac{1}{2})$ and $2\Phi(-\frac{1}{2}\lambda_2 R^\frac{1}{2})$, respectively. Hence, the asymptotically optimal scale is $O(1)$ and the asymptotically optimal acceptance rate is  $0.234$ for RWM. By comparing LBP and RWM at their respective optimal scales, it can be determined that LBP is $O(N^\frac{2}{3})$ more efficient than RWM. 

These asymptotically optimal acceptance rates are robust in the following respects. First, although the initial derivation is established w.r.t.\ product distributions, the result can be expanded to more general distributions. Second, the efficiency is not sensitive around the optimal acceptance rate. For example, whereas $0.574$ is the optimal acceptance rate for LBP, the algorithm retains high efficiency for acceptance rates between $0.5$ and $0.7$. Based on these observations, we propose an adaptive LBP (ALBP) algorithm that automatically tunes the update scale to suit the target distribution.

We validate these theoretical findings in a series of empirical simulations on the Bernoulli model, the Ising model, factorized hidden Markov models (FHMM) and restricted Boltzmann machines (RBM). The experimental outcomes comport with the theory. Moreover, we demonstrate that ALBP can automatically find near optimal scales for these distributions. We also use ALBP to train deep energy based models (EBMs), finding that it reduces the MCMC steps needed in contrastive divergence training \citep{hinton2002training, tieleman2009using}, significantly improving the efficiency of the overall training procedure.

\section{Background}
\textbf{Metropolis-Hastings Algorithm}
Let $\pi$ denote the target distribution.
Given a current state $x^{(n)}$, a M-H sampler draws a candidate state $y$ from a proposal distribution $q(x^{(n)}, y)$. Then, with probability
$\min\Big\{1,\, \frac{\pi(y)q(y, x^{(n)})}{\pi(x^{(n)})q(x^{(n)}, y)}\Big\}$
the proposed state is accepted and $x^{(n+1)} = y$; otherwise, $x^{(n+1)} = x^{(n)}$. In this way, the detailed balance condition is satisfied and the M-H sampler generates a Markov chain $x_0, x_1, ...$ that has $\pi$ as its stationary distribution.

\textbf{Locally Balanced Proposal}.
The locally balanced proposal (LBP) is a special case of the pointwise informed proposal (PIP), which is a class of M-H algorithms for discrete spaces \citep{zanella2020informed} using the proposal distribution $
    Q_g(x,y) \propto g\left(\pi(y) / \pi(x) \right)$
such that $g$ is a scalar weight function.
\citet{zanella2020informed} shows that the family of locally balancing functions $\mathcal{G} = \{g: \mathbb{R}_+ \rightarrow \mathbb{R}_+, g(t) = t g(\frac{1}{t}), \forall t > 0\}$ (e.g.\ $g(t) = \sqrt{t}$ or $\frac{t}{t+1}$) is asymptotically optimal for PIP. Hence, PIP with a locally balanced function for its weight function is referred to as LBP. Despite having good proposal quality, PIP requires the weight $g(\pi(z) / \pi(x))$ to be calculated for all candidate states $z$ in the neighborhood of $x$, which results in its high computational cost. \citet{grathwohl2021oops} propose to estimate the probability change by leveraging the gradient, improving the scalability of LBP.

\textbf{Locally Balanced Proposal with Auxiliary Path}.
\citet{sun2021path} generalize LBP by introducing an auxiliary path sampler, which allows multiple sites to be updated per M-H step. In particular, \citet{sun2021path} sequentially selects the update indices without replacement, and uses these indices as auxiliary variables to keep the proposal distribution tractable while preserving the detailed balance condition. Although this can achieve significant improvements in empirical performance, \citet{sun2021path} manually tune the update size per M-H step, and leave the optimal scale problem  open.

\section{Main Result}
\subsection{Problem Statement}
We establish  asymptotic limit theorems for two M-H algorithms in discrete spaces: the \textit{locally balanced proposal} (LBP) and \textit{random walk Metropolis} (RWM). 
Following previous work~\citep{gelman1997weak, roberts1998optimal, beskos2013optimal, vogrinc2022optimal}, we conduct our analysis on a product probability measure $\pi$. In particular, for a state space $\mathcal{X} = \{0, 1\}^N$, we consider a factored target distribution
\begin{equation}
    \label{eq:bernoulli}
    \pi^{(N)}(x) = \prod_{i=1}^N \pi_i(x_i) = \prod_{i=1}^N p_i^{x_i}(1-p_i)^{1-x_i}
\end{equation}
where each site is assumed to have a sufficiently large probability for being both $0$ and $1$; that is, for a fixed $ \epsilon \in (0, \frac{1}{4})$, we assume the target distribution belongs to:
\begin{equation}
    \mathcal{P}_\epsilon := \{\pi^{(N)}: \epsilon < p_j \land (1 - p_j) < \frac{1}{2} - \epsilon, \forall j = 1, ..., N, N \ge 1\}  \label{eq:smooth_target}
\end{equation}
\textcolor{\cdiff}{where we denote $a \land b = \min \{a, b\}$.} To measure the efficiency of the sampler, an ergodic estimate varies with the objective function considered. Alternatively, we use a natural progress estimate: the expected jump distance (EJD).
Denote $P_\theta$ as the transition kernel, $d(x, y)$ as the Hamming distance between $x$ and $y$. For a M-H sampler parameterized by $\theta$, its expected jump distance $\rho(\theta)$ and corresponding expected acceptance rate $a(\theta)$ are
\begin{equation}
    \rho(\theta) = \sum_{X, Y \in \mathcal{X}} \pi(X) P_\theta(X, Y) d(X, Y), \quad a(\theta) = \sum_{X, Y \in \mathcal{X}} \pi(X) P_\theta(X, Y) 1_{\{X\neq Y\}}
\end{equation}
\textcolor{\cdiff}{
In continuous space, the limit of sampling process is a diffusion process, whose efficiency is determined by the expected squared jump distance (ESJD) \citep{roberts2001optimal}. 
In discrete space, the limit of the sampling process is a jump process, whose velocity is characterized by the EJD. Hence, EJD is the correct metric to measure the efficiency in discrete space; see more details in Appendix \ref{sec:why_ejd}.
}

\subsection{Locally Balanced Proposal}
We consider the M-H sampler LBP-$R$, where $R$ refers to flipping $R$ indices in each M-H step.
Given a current state $x$, LBP-R calculates the weight $w_j$ for flipping index $j$ as in \textcolor{\cdiff}{PIP}. Since we are considering a binary target distribution of the form \eqref{eq:bernoulli}, we have
\begin{equation}
    w_j(x) = w_j(x_j) = g(\frac{\pi_j(1-x_j)}{\pi_j(x_j)}) \label{eq:def_w}
\end{equation}
where $g$ is a locally balanced function. Following \cite{sun2021path}, LBP-R select indices $u_r$ with probability $\mathbb{P}(u_r = j) \propto w_j$ sequentially for $r = 1, ..., R$, \textbf{without} replacement. The new state $y$ is obtained by flipping indices $u_{1:R}$ of $x$. If we consider $u$ as an auxiliary variable, the accept rate $A(x, y, u)$ in the M-H acceptance test can be written as
\begin{equation}
    A(x, y, u) = 1 \land \frac{\pi(y)\prod_{r=1}^R \frac{w_{u_r}(y)}{W(y, u) + \sum_{i=1}^r w_{u_i}(y)}}{\pi(x)\prod_{r=1}^R \frac{w_{u_r}(x)}{W(x, u) + \sum_{i=r}^R w_{u_i}(x)}}, \quad \text{ where } W(x, u) = \sum_{i=1}^N w_i - \sum_{r=1}^R w_{u_r} \label{eq:acc_lbp}
\end{equation}
From theorem 1 in \cite{sun2021path}, the auxiliary sampler LBP-R satisfies detailed balance.
A M-H step of LBP-R is summarized in Algorithm~\ref{alg:lbp}.

\IncMargin{1.5em}
\begin{algorithm}[H]
\newcommand{\lIfElse}[3]{\lIf{#1}{#2 \textbf{else}~#3}}
\caption{A M-H step of LBP-R and \textcolor{blue}{ALBP}}
\label{alg:lbp}
  Given current state $x^{(n)}$, \textcolor{blue}{current $R_t$,} initialize candidate set $\mathcal{C} = \{1, .., N\}$\;
  \For{$r = 1, ..., R$ \textcolor{blue}{ or $r = 1, ..., \text{rounding}(R_t)$}}{
    Sample $u_r$ with $\mathbb{P}(u_r=j) \propto w_j(x^{(n)}) 1_{\{j \in \mathcal{C}\}}$\;
    Pop $u_r$ out of the candidate set: $\mathcal{C} \gets \mathcal{C} \backslash \{u_r\}$\;
  }
  Obtain $y$ by flipping indices $u_1, ..., u_R$ of $x^{(n)}$.\;
  \lIfElse{rand(0,1) $< A(x^{(n)}, y, u)$}{$x^{(n+1)} = y$}{$x^{(n+1)} = x^{(n)}$}
  \textcolor{blue}{\lIf{$t < T_\text{warmup}$}{$R_{t+1} \gets R_t + (A(x^{(n)}, y, u) - 0.574)$}}
\end{algorithm}
\DecMargin{1.5em}


\subsection{Optimal Scaling for Locally Balanced Proposal}
We are now ready to state the first asymptotic theorem.
\begin{theorem}
\label{thm:lb}
For arbitrary sequence of target distributions $\{\pi^{(N)}\}_{N=1}^\infty \subset \mathcal{P}_\epsilon$, the M-H sampler LBP-R with a locally balanced weight function $g$ obtains the following, if $R = \textcolor{\cdiff}{\lfloor} l N^\frac{2}{3}\textcolor{\cdiff}{\rfloor}$, 
\begin{align}
    \lim_{N\rightarrow \infty} a(R) - 2\Phi\left(-\frac{1}{2}\lambda_1 l^\frac{3}{2}\right) = 0 \label{eq:asym_acc_lbp}
\end{align}
where $\Phi$ is the c.d.f. of standard normal distribution and $\lambda_1$ only depends on $\pi^{(N)}$
\begin{equation}
    \label{eq:lambda_lbp}
    \lambda_1^2 = \lambda_1^2(\pi^{(N)}) = \frac{\sum_{j=1}^N p_j w_j(1) (w_j(0) - w_j(1))^2}{4( \mathbb{E}_{x}[\frac{1}{N}\sum_{i=1}^N w_i(x_i)])^2 \sum_{i=1}^N p_i w_i(1)}
\end{equation}
\end{theorem}
The definition of $\lambda_1$ in \eqref{eq:lambda_lbp} explains the motivation of restricting the target distributions in \eqref{eq:smooth_target}. In fact, introducing the $\epsilon$ gives upper and lower bounds of $\lambda_1$. When all $p_j$ are arbitrarily close to $\frac{1}{2}$, $(w_j(0) - w_j(1))^2$ in numerator will be zero, so is $\lambda_1$. As a result, the acceptance rate will always be $1$. Else, when all $p_j$ are arbitrarily close to $0$ or $1$, $\mathbb{E}_{x}[\frac{1}{N}\sum_{i=1}^N w_i(x_i)]$ in denominator will be zero, and $\lambda_1$ will be infinity. As a result, the acceptance rate will always be $0$. So, we have to make the mild assumption in \eqref{eq:smooth_target} to assure the following asymptotic result holds. \textcolor{\cdiff}{A more detailed discussion about $\epsilon$ is given in Appendix \ref{sec:convergence_epsilon}.}
\begin{corollary}
\label{cor:order}
The optimal choice of scale for $R = l N^\frac{2}{3}$ is obtained when the expected acceptance rate is $0.574$, independent of the target distribution.
\end{corollary}
\begin{proof}
When $R = l N^\frac{2}{3}$, denote $z = \lambda_1^\frac{2}{3} l$, we have:

\begin{equation}
    \rho(R) = a(R) R = \textcolor{\cdiff}{ 2lN^\frac{2}{3}\left(\Phi\left(-\frac{1}{2}\lambda_1 l^\frac{3}{2}\right) + o(1)\right)
    = \Big(\frac{N}{\lambda_1}\Big)^{\frac{2}{3}} 2 z \Phi\Big(-\frac{1}{2}z^\frac{3}{2}\Big) + o\left(N^\frac{2}{3}\right)}
\end{equation}
It means the optimal value of $z$ is independent of the target distribution $\pi^{(N)}$. As $\Phi$ is known, we can numerically solve $z=1.081$, and the corresponding expected acceptance rate is $a = 0.574$.
\end{proof}

\subsection{Proof of Theorem \ref{thm:lb}}
Denote the current state as $x$ and a new state proposed in LBP-R as $y$. Consider the acceptance rate $A(x, y, u)$ in \eqref{eq:acc_lbp}. Using the fact that, if index $j$ is not flipped then $w_j(y) = w_j(x)$, we have:
\begin{align}
    \frac{\pi(y)}{\pi(x)} \frac{\prod_{r=1}^R w_{u_r}(y)}{\prod_{r=1}^R w_{u_r}(x)}
    = \frac{\pi(y)}{\pi(x)} \frac{\prod_{i=1}^N w_i(y)}{\prod_{i=1}^N w_i(x)}
    = \prod_{i=1}^N \frac{ {\pi_i(y_i)} / {\pi_i(x_i)} g(\pi_i(x_i) / \pi_i(y_i))}{g(\pi_i(y_i) / \pi_i(x_i))}  = 1 \label{eq:using_lb_property}
\end{align}
where \eqref{eq:using_lb_property} takes advantage of the property of a locally balanced function. Hence, the acceptance rate $A(x, y, u)$ can be simplified to:
\begin{equation}
    1 \land \exp\left(\sum_{r=1}^R \log\Big(\frac{1 + \sum_{i=r}^R w_{u_i}(x) / W(x, u)}{1 + \sum_{i=1}^r w_{u_i}(y) / W(y, u)}\Big) \right) \label{eq:acc_before_expectation}
\end{equation}
From the definition in \eqref{eq:acc_lbp}, we have $W(x, u) = W(y, u)$.
Denote $i\land j = \min\{i, j\}$ and $i \lor j = \max\{i, j\}$, we have the following approximation:
\begin{lemma}
\label{lemma:taylor_sumlog}
Define $W= \mathbb{E}_{x,u}[W(x,u)]$. \textcolor{\cdiff}{We have: $\lim_{N\rightarrow 0} \sum_{r=1}^R \log(\frac{1 + \sum_{i=r}^R w_{u_i}(x) / W(x, u)}{1 + \sum_{i=1}^r w_{u_i}(y) / W(y, u)}) - (A + B) = 0$}, where
\begin{align}
    A = & \frac{1}{W}\sum_{r=1}^R (R-r+1)w_{u_i}(x_{u_i}) - r w_{u_i}(y_{u_i}) \label{eq:A}\\
    B = & - \frac{1}{2} \frac{1}{W^2} \sum_{i, j=1}^R \Big[ i\land j\  w_{u_i}(x_{u_i})w_{u_j}(x_{u_j}) - (R - i \lor j + 1) w_{u_i}(y_{u_i})w_{u_j}(y_{u_j})\Big] \label{eq:B}
\end{align}
\end{lemma}

To analyze $A$ and $B$, we reverse the order of $x$ and $u$. In particular, instead of first sampling $x \sim \pi(x)$, then sampling $u \sim p(x|u)$, we use a reversed order where we first \textcolor{\cdiff}{determine} the indices $u$, then the values of $x_u$, and finally the values of $x_{-u}$.
\begin{lemma}
\label{lemma:posterior}
The joint distribution $p(x, u) = \pi(x) p(u|x)$ can be decomposed in the following form:
\begin{equation}
p(x, u) = \prod_{r=1}^R p(u_r|u_{1:r-1}) \ \prod_{r=1}^R p(x_{u_r}|u, x_{u_{1:r-1}}) \  p(x_{-u}|u, x_u)
\end{equation}
Denote $j \notin u_{1:r-1}$ represents $j \neq u_i$ for $i = 1, ..., r-1$, the conditional probabilities are
\begin{align}
    p(u_r = j|u_{1:r-1}) = \frac{p_j w_j(1) 1_{\{j \notin u_{1:r-1}\}}}{\sum_{i=1}^N p_i w_i(1) 1_{\{i \notin u_{1:r-1}\}}} + O(N^{-\frac{5}{2}})  \label{eq:posterior_u} \\
    p(x_j=1|u, x_{1:j-1}, u_r = j) = \frac{1}{2} + r \frac{w_j(0) - w_j(1)}{W} + O(N^{-\frac{2}{3}}) \label{eq:posterior_x}
\end{align}
\end{lemma}

With the conditional distribution in Lemma \ref{lemma:posterior}, we are able to give a concentration property of the term $B$ and show it is safe to ignore:
\begin{lemma}
\label{lemma:B}
With a probability larger than $1 - O(\exp(-N^\frac{1}{2}))$, $B = O\big(N^{-\frac{1}{12}}\big)$.
\end{lemma}

For term $A$, we use martingale central limit theorem with convergence rate \citep{haeusler1988rate} to bound the Kolmogorov-Smirnov statistic.
\begin{lemma}
\label{lemma:A_good}
When $R = l N^\frac{2}{3}$, $\lambda_1$ defined as \eqref{eq:lambda_lbp}, we have:
\begin{equation}
    |\mathbb{P}(\frac{A - \mu}{\sigma} \ge t) - \Phi(t)| = O\big( N^{-\frac{1}{12}}\big), \quad \mu = - \frac{1}{2} \lambda_1^2 l^3, \quad \sigma^2 = \lambda_1^2 l^3
    \label{eq:mclt_rate}
\end{equation}
\end{lemma}
By \eqref{eq:mclt_rate}, the expectation w.r.t. $A$ asymptotically equals to the expectation on $\mathcal{N}(\mu, \sigma^2)$.
The final step to prove Theorem~\ref{thm:lb} is to exploit a property of the normal distribution.
\begin{lemma}
\label{lemma:normal}
If $Z \sim \mathcal{N}(\mu, \sigma^2)$, then we have:
\begin{equation}
    \mathbb{E}[1\land \exp(Z)] = \Phi\Big(\frac{\mu}{\sigma}\Big) + \exp\Big(\mu + \frac{\sigma^2}{2}\Big) \Phi\Big(-\sigma - \frac{\mu}{\sigma}\Big)
\end{equation}
where $\Phi$ is the c.d.f.\ of the standard normal distribution.
\end{lemma}

By Lemma~\ref{lemma:A_good}, \ref{lemma:normal}, we have the expectation of \eqref{eq:acc_before_expectation}, which is the expected accept rate, equals to:
\begin{align}
    \mathbb{E}[a(R)] &= \Phi\Big(-\frac{1}{2}\lambda_1 l^\frac{3}{2}\Big) + \exp(0)\Phi\Big(-\frac{1}{2}\lambda_1 l^\frac{3}{2}\Big) = 2\Phi\Big(-\frac{1}{2}\lambda_1 l^\frac{3}{2}\Big)
\end{align}

\subsection{Optimal Scaling for Random Walk Metropolis}
We denote the Random Walk Metropolis in discrete space as RWM-$R$, where $R$ refers to flipping $R$ indices in each M-H step. Under the Bernoulli distribution, a site is more likely to stay at high probability position, so if we randomly flip a site, it is more likely to decrease its probability. That is, intuitively, the acceptance rate will decrease exponentially as the scale $R$ increases. Consequently, the optimal scale for RWM-$R$ should be $O(1)$. Though this is not a rigorous proof, the constant scaling indicates that it will be hard to directly prove an asymptotic theorem for RWM-$R$. To address this difficulty, we first restrict our target distribution to a smaller class of Bernoulli distributions $\mathcal{P}^{(\beta)}_\epsilon \subset \mathcal{P}_\epsilon$, which is formally defined as follows.
For a fixed $\epsilon \in (0, \frac{1}{4})$ and a fixed $\beta > 0$, define
\textcolor{\cdiff}{
\begin{align}
    \mathcal{P}^{(\beta)}_\epsilon := \left\{ \pi^{(N)}: \frac{1}{2} - \frac{1}{2N^\beta} + \frac{\epsilon}{N^\beta} < p_j \land (1 - p_j) < \frac{1}{2} - \frac{\epsilon}{N^\beta} \right\} \label{eq:p_beta}
\end{align}
}
When $N$ is large, each $p_j$ will be very close to $\frac{1}{2}$. In this way, the acceptance rate will not drop too fast when $R$ is increased, and a non-constant $R$ will be permitted. This enables us to prove:
\begin{theorem}
\label{thm:rw}
For arbitrary sequence of target distributions $\{\pi^{(N)}\}_{N=1}^\infty \subset \mathcal{P}_\epsilon^{(\beta)}$, the M-H sampler RWM-R obtains the following, if $R = l N^{2\beta}$,
\begin{align}
    \lim_{N\rightarrow \infty} a(R) - 2\Phi\left(-\frac{1}{2}\lambda_2 l^\frac{1}{2}\right) \label{eq:acc_rwm}
\end{align}
where $\Phi$ is the c.d.f.\ of the standard normal distribution and $\lambda_2$ only depends on $\pi^{(N)}$.
\begin{equation}
    \label{eq:lambda_rwm}
    \lambda_2^2 = \lambda_2^2(\pi^{(N)}) = \frac{2}{N}\sum_{i=1}^N N^{2\beta}(2p_i-1)\log \frac{p_i}{1 - p_i}
\end{equation}
\end{theorem}

\begin{corollary}
\label{cor:rw_rate}
The optimal scale $R = l N^{2\beta}$ is obtained when the expected acceptance rate is 0.234, independent of the target distribution.
\end{corollary}
The rate in Corollary~\ref{cor:rw_rate} is proved for arbitrary $\beta>0$. If we let $\beta$ decrease to $0$, at $\beta = 0$ the optimal scale for RWM-$R$ is $O(1)$ while the optimal acceptance rate is $0.234$.
\textcolor{\cdiff}{
Also, we can notice that $\mathcal{P}^{(\beta)}_\epsilon$ converges to $\mathcal{P}_\epsilon$ when $\beta$ decrease to $0$ and we are able to show the optimal scale of RWM in $\mathcal{P}_\epsilon$ is $O(1)$, see details in Appendix \ref{sec:rwm_Peps}.
}
However, this limit is not mathematically rigorous, because Theorem~\ref{thm:rw} and Corollary~\ref{cor:rw_rate} only hold asymptotically, such that a smaller $\beta$ requires a larger $N$. Hence, when $\beta$ decreases to $0$,  $N$ must approach infinity to satisfy the asymptotic theorem. Although there is this minor gap in the analysis, the conclusion nevertheless aligns very well with different target distributions in the  experiment section. 

\section{Adaptive Algorithm}
Given knowledge of the optimal acceptance rate, one can design an adaptive algorithm that automatically tunes the scale of the M-H samplers. For this purpose, we use stochastic optimization \cite{andrieu2008tutorial, robbins1951stochastic} to adjust the scaling parameter $R_t$ to ensure that the statistic $A_t = a_t - \delta$ approaches $0$, where $a_t$ is the acceptance probability for iteration $t$ and $\delta$ is the target acceptance rate ($0.574$ for LBP and $0.234$ for RWM). According to Theorem~\ref{thm:lb} and Theorem~\ref{thm:rw}, the acceptance rate is a decreasing function of the scaling $R_t$. Hence, we use the update rule:
\begin{equation}
    R_{t+1} \leftarrow R_t + \eta_t A_t \label{eq:adaptive}
\end{equation}
with step size $\eta_t = 1$. We follow common practice and adapt the tunable MCMC parameters during a warmup phase before freezing them thereafter \citet{gelman2013bayesian}. The computational cost for \eqref{eq:adaptive} is ignorable comparing the total cost of a M-H step. The algorithm boxes for ALBP and ARWM are given in Appendix~\ref{appendix:algorithm}. More advanced implementations are possible, but it is out of the focus in the paper. We observe below that this simple approach is able to maintain the sampler robustly near the optimal acceptance rate.

\section{Related Work}
Informed proposals for Metropolis-Hastings (M-H) algorithms have been extensively studied for continuous spaces \citep{robert2013monte}. The most famous algorithms are the Metropolis-adjusted Langevin algorithm (MALA) \citep{rossky1978brownian} and Hamiltonian Monte Carlo (HMC) \citep{neal2011mcmc}. MALA, HMC, and their variants \citep{girolami2011riemann, hoffman2014no, welling2011bayesian, titsias2019gradient, hirt2021entropy, hoffman2021adaptive, hird2020fresh, livingstone2019barker} use the gradient of the target distribution to guide the proposal distribution toward high probability regions, which brings substantial improvements in  sampling efficiency compared to uninformed methods, such as random walk Metropolis (RWM) \citep{metropolis1953equation}.

Informed proposals have also demonstrated recent success in  discrete spaces. \cite{zanella2020informed} first gives a formal definition of the pointwise informed proposal (PIP) for discrete spaces, then proves that locally balanced proposals (LBP), using a family of locally balanced functions as the weight function in PIP, are asymptotically optimal for PIP. Following this work, \cite{power2019accelerated} extended the framework to Markov jump processes and introduced non-reversible heuristics to accelerate sampling. 
\citet{sansone2021lsb} parameterize the locally balanced function and tune it by minimizing a mutual information. \citet{grathwohl2021oops} give a more scalable version of LBP for differentiable target distributions by estimating the probability change through the gradient. Despite strong empirical results, the LBP method of \cite{zanella2020informed} only flips one bit per M-H step, since PIP has to restrict the proposal distribution to a small neighborhood, e.g.\ a 1-Hamming ball, due to its computational cost. \cite{sun2021path} generalize LBP to flip multiple bits in a single M-H step, gaining significant improvement in sampling efficiency. However, the scaling of the proposal distribution in \cite{sun2021path} was manually tuned and the optimal scaling problem was left open.

For continuous spaces, the optimal scaling problem for informed proposals has been well studied. A significant literature has already shown that the scale can be tuned with respect to the optimal acceptance rate \citep{roberts2001optimal}, e.g.\ 0.234 for RWM \citep{gelman1997weak}, 0.574 for MALA \cite{roberts1998optimal}, 0.651 for HMC \cite{beskos2013optimal}, and 0.574 for Barker \citep{vogrinc2022optimal}, by decreasing the scale so that the Markov chain converges to a diffusion process. However, such a technique is not directly applicable to LBP given its discrete nature. \cite{roberts1998cube} make an initial attempt on discrete space, however it assumes all dimensions satisfy independent, identical Bernoulli distribution. In this work, we have established for the first time the optimal scale for LBP and RWM in discrete spaces.

\section{Experiments}
The effectiveness of LBP has been extensively demonstrated in previous work, e.g. \cite{zanella2020informed, grathwohl2021oops, sun2021path}, in comparison to other M-H samplers for discrete spaces, such as RWM, Gibbs sampling, the Hamming Ball sampler \citep{titsias2017hamming}, and continuous relaxation based methods \cite{zhang2012continuous, pakman2013auxiliary, nishimura2017discontinuous, han2020stein}. Therefore, we focus on simulating LBP-$R$, with weight function $g(t) = \frac{t}{t+1}$, and RWM-$R$ to validate our theoretical findings. More experiments, including different weight functions and comparison between "with" and "without" replacement versions of LBP are given in Appendix~\ref{appendix:exp}.

Throughout the experiment section, we will use the gradient approximation \citep{grathwohl2021oops}. That is to say, we estimate the change in probability of flipping index $i$ is estimated by:
$\tilde{d} x_i = \exp((1 - 2 x_i) (\nabla \log \pi(x))_i)$
For the Bernoulli distribution, this is still exact and does not hinder the justification of the theoretical results. For more complex models, this approximation makes the algorithms significantly more efficient. In particular,  the gradient approximation only requires two calls of the probability function and two calls of the gradient function. 
Consequently, LBP with gradient approximation will take about twice time per update compared to RWM. In our experiments, we observe that LBP and GWG takes $1.2 \pm 0.2$ and $1.1 \pm 0.1$ more time per update, respectively, than RWM, across all target distributions. We therefore omit reporting the detailed run time for each method.

\subsection{Sampling from different target distributions}
We consider four target distributions: the Bernoulli distribution, the Ising model, the factorial hidden Markov model (FHMM), and the restricted Boltzmann machine (RBM). For each model, we consider three configurations: C1, C2, and C3 for smooth, moderate, and sharp target distributions. To obtain performance curves, we first simulate LBP-$1$ and RWM-$1$ for an initial acceptance rate $a_{\max}$. Then, we adopt $a_{\max} - 0.02$, ..., $a_{\max} - 0.02k$, ... as a target acceptance rate. For each rate, we use the adaptive sampler to obtain an estimated scale $R$, with which we simulate 100 chains and calculate the final real acceptance rate and efficiency. In this way, we collect abundant data points to characterize the relationship between acceptance rate and efficiency to facilitate the following analyses. 
\begin{figure}[!htb]
    \centering
    \begin{minipage}{.5\textwidth}
        \centering
        \includegraphics[width=0.99\textwidth]{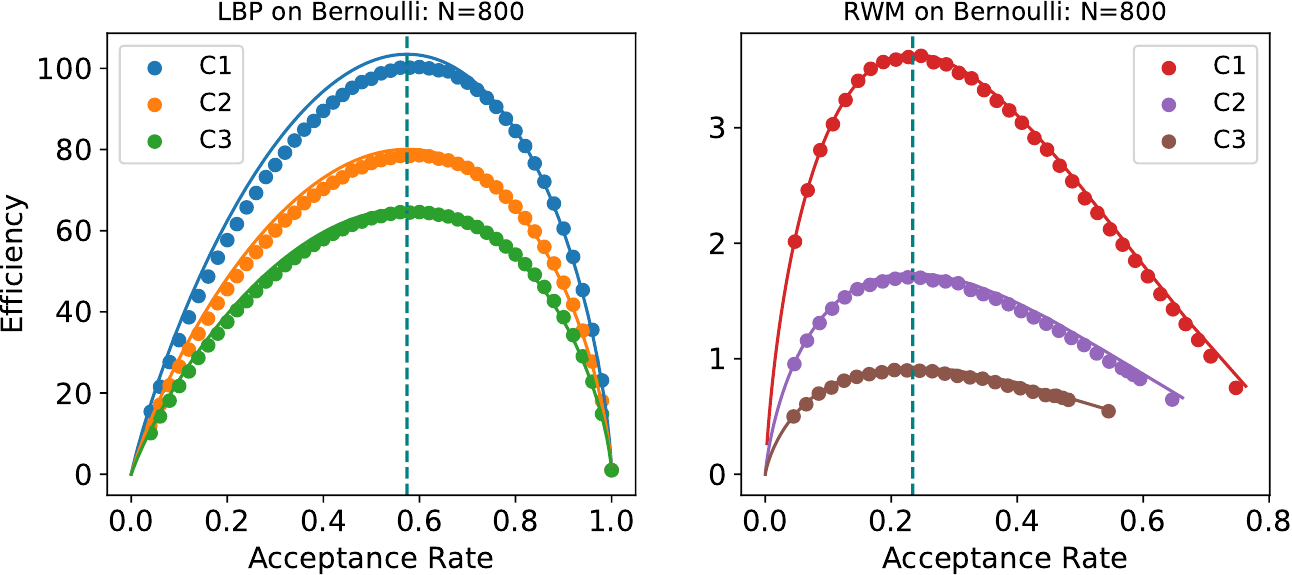}
        \caption{Efficiency Curves on Bernoulli}
        \label{fig:bernoulli-800}
    \end{minipage}%
    \begin{minipage}{0.5\textwidth}
        \centering
        \includegraphics[width=0.99\textwidth]{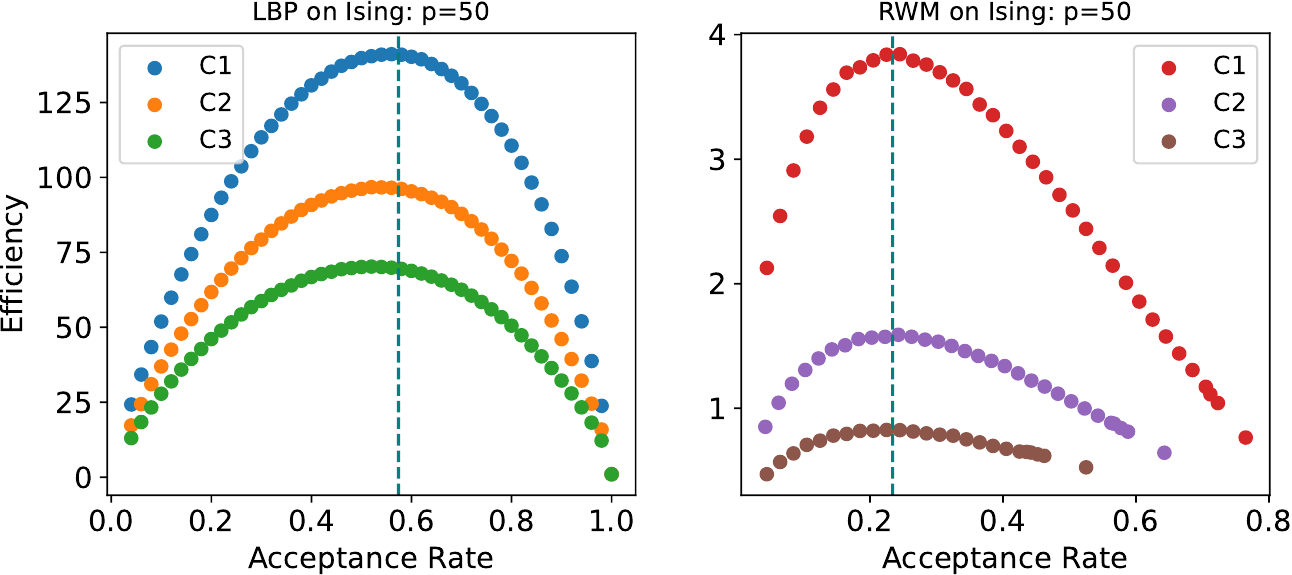}
    \caption{Efficiency Curves on Ising}
    \label{fig:ising-50}
    \end{minipage}
\end{figure}


{\bf Bernoulli Distribution}.
We validate our theoretical results on Bernoulli distribution. The probability mass function is given in \eqref{eq:bernoulli}. For each configuration, we simulate on domains with three dimensionalities: $N=100$, $800$, $6400$. The scatter plot for $N=800$ is reported in Figure~\ref{fig:bernoulli-800}. We also estimate $\lambda$ in \eqref{eq:lambda_lbp} and \eqref{eq:lambda_rwm} and plot the theoretical efficiency curve in \eqref{eq:acc_lbp} and \eqref{eq:acc_rwm}.
From Figure~\ref{fig:bernoulli-800}, we can see that the simulation results align well with the theoretically predicted curves, and the optimal efficiencies were achieved at $0.574$ for LBP and $0.234$ for RWM for all configurations. 

{\bf Ising Model}.
The Ising model \citep{ising1924beitrag} is a classical model in physics defined on a $p\times p$ square lattice graph $(V_p, E_p)$ (details in Appendix~\ref{sec:ising}).
For each configuration, we simulate on three sizes $p = 20$, $50$, $100$. We report the results for $p=50$ in Figure~\ref{fig:ising-50}.
For LBP, the optimal efficiencies are achieved at around $0.5$, which is slightly less than $0.574$, although these values are close. Thus we can say that the asymptotically optimal acceptance rate for LBP still applies to the Ising model. For RWM, $0.234$ perfectly matches the acceptance rate where the optimal efficiencies are obtained.

{\bf Factorial Hidden Markov Model}.
The FHMM \citep{ghahramani1995factorial} uses latent variables $x \in \mathcal{X} = \{0, 1\}^{L\times K}$ to characterize time series data $y \in \mathbb{R}^L$ (details in Appendix~\ref{sec:fhmm}).
Given $y$, we sample the hidden variables $x$ from the posterior $\pi(x) = p(x|y)$. For each configuration, we simulate in three sizes $L=200$, $1000$, $4000$. We report the results for $L=1000$ in Figure \ref{fig:fhmm-1000}. One can observe that these results match the theoretical predictions very well.
\begin{figure}[!htb]
    \centering
    \begin{minipage}{.5\textwidth}
        \centering
        \includegraphics[width=0.99\textwidth]{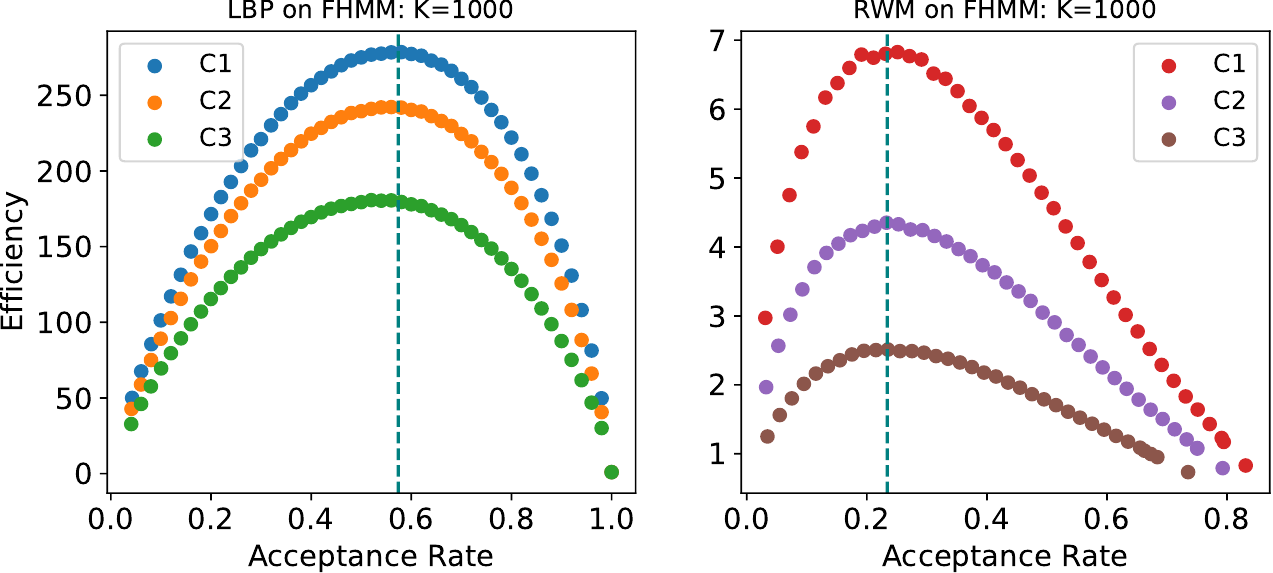}
    \caption{Efficiency Curves on FHMM}
    \label{fig:fhmm-1000}
    \end{minipage}%
    \begin{minipage}{0.5\textwidth}
        \centering
        \includegraphics[width=0.99\textwidth]{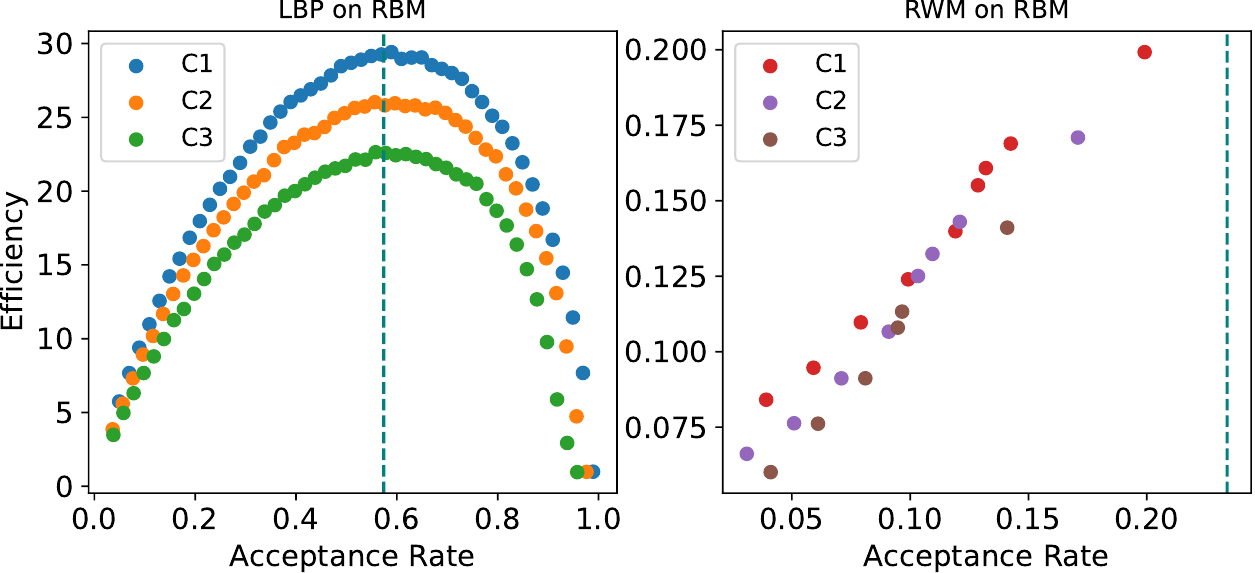}
    \caption{Efficiency Curves on RBM}
    \label{fig:rbm}
    \end{minipage}
\end{figure}


{\bf Restricted Boltzmann Machine}. 
A RBM \citep{smolensky1986information} is a bipartite latent-variable model that defines a distribution over binary data $x \in \{0, 1\}^N$ and latent data $z \in \{0, 1\}^h$ (details in Appendix~\ref{sec:rbm}).
We train an RBM on the MNIST dataset using contrastive divergence \citep{hinton2002training} and sample observable variables $x$. We report the results in Figure~\ref{fig:rbm}.
For LBP, although RBM is much more complex than a product distribution, its efficiency versus acceptance rate curves still match the theoretical predictions very well. For RWM, even using $R=1$ will result in acceptance rates less than $0.234$ for all configurations. Although we cannot check what the optimal value is, we still observe that  efficiency is an increasing function of the acceptance rate when the acceptance rate is less than $0.234$, as predicted by the theory.



{\bf Optimal Scaling and Efficiency}. 
We examine how optimal scaling $R$ for LBP, RWM 
and their relative efficiency ratio grow w.r.t. the model dimension $N$.
In figure~\ref{fig:order}, we can see that both the optimal scaling and efficiency ratio are linear in log-log plot and the slopes are close to $\frac{2}{3}$ across Bernoulli, Ising, and FHMM. The results matches the theories that the optimal scaling $R=O(N^\frac{2}{3})$ for LBP, $R=O(1)$ for RWM, and the relative efficiency ratio LBP over RWM is $O(N^\frac{2}{3})$.

\begin{figure}
    \centering
    \includegraphics[width=.75\textwidth]{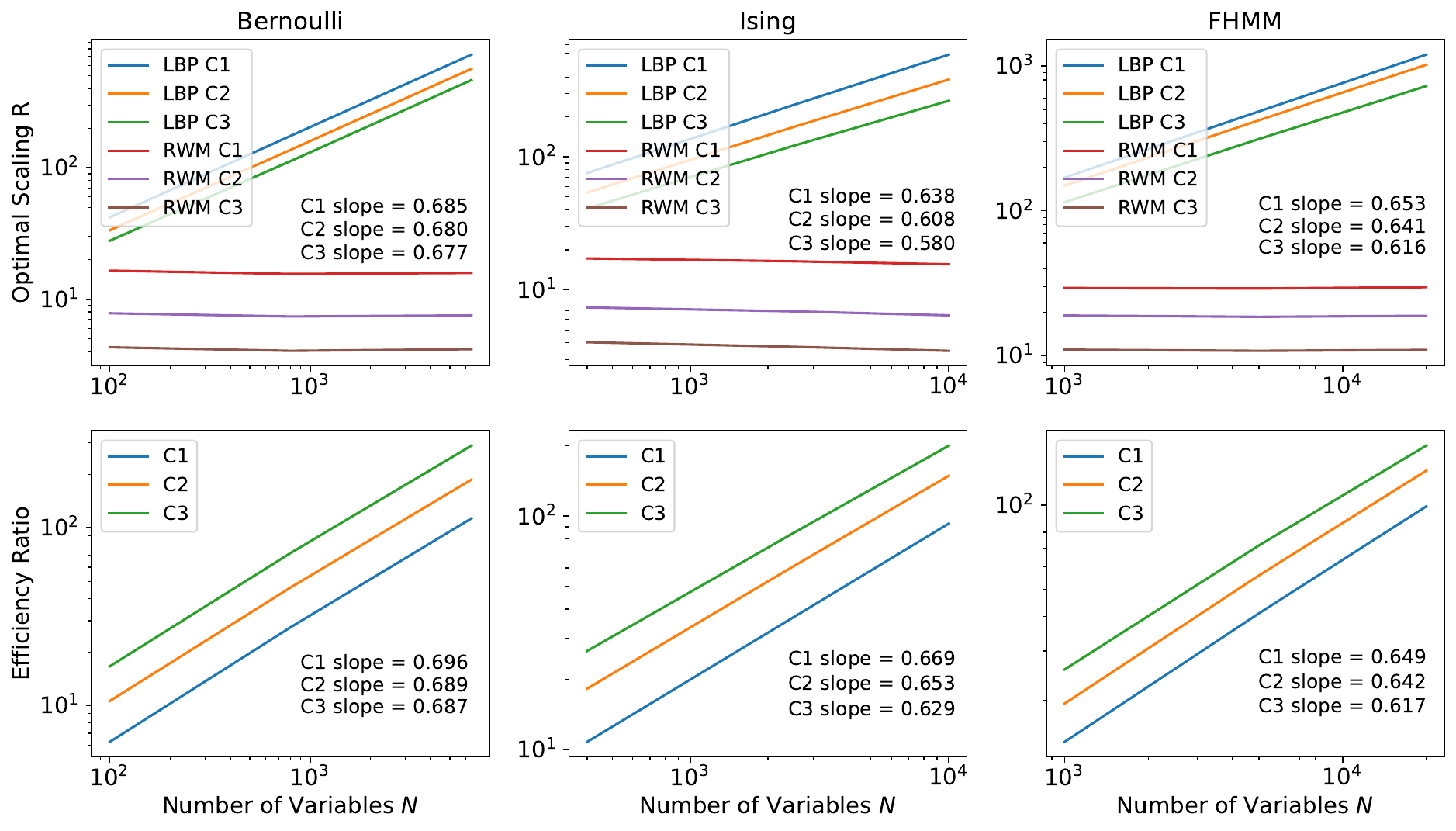}
    \caption{Optimal Scaling $R$ and Efficiency Ratio with respect to model dimension N}
    \label{fig:order}
\end{figure}

\begin{table*}[htb]
    \scriptsize
    \centering
    \caption{Performance of the Samplers on Various Distributions}
    \begin{tabular}{crrrrrrrrrrrr}
    \toprule
    Size & \multicolumn{3}{c}{Bernoulli} & \multicolumn{3}{c}{Ising} & \multicolumn{3}{c}{FHMM} & \multicolumn{3}{c}{RBM} \\
    \cmidrule(lr){1-1} \cmidrule(lr){2-4} \cmidrule(lr){5-7} \cmidrule(lr){8-10} \cmidrule(lr){11-13}
    Sampler & EJD & ESS & Time & EJD & ESS & Time & EJD & ESS & Time & EJD & ESS & Time \\
    \midrule
    RWM-1 & 0.65 & 10.02 & 15.44 & 0.64 & 12.14 & 74.28 & 0.79 & 7.26 & 58.03 & 0.17 & 10.76 & 59.54 \\
    ARWM & 1.70 & 18.44 & 14.90 & 1.58 & 19.60 & 77.45 & 4.32 & 13.32 & 60.02 & 0.17 & 11.13 & 61.24\\
    GRWM & 1.70 & 18.67 & 18.01 & 1.59 & 20.16 & 76.89 &  4.35 & 15.22 & 61.19 & 0.17 & 10.76 & 59.54\\
    LBP-1 & 1.00 & 13.39 & 24.36 & 1.00 & 14.11 & 111.19 & 1.00 & 6.91 & 134.42 & 0.98 & 13.38 & 116.04\\
    ALBP & 78.63 & 622.35 & 28.07 & 96.23 & 821.06 & 124.37 & 242.01 & 129.28 & 487.63 & 26.07 & 25.59 & 144.03 \\
    GLBP & 78.83 & 644.43 & 25.42 & 96.68 & 809.12 & 129.28 & 242.52 & 140.43 & 508.27 & 25.86 & 25.83 &119.38 \\
    \bottomrule
    \end{tabular}
    \label{tab:selected_results}
\end{table*}

\subsection{Adaptive Sampling}
We have validated the theoretical findings regarding the optimal acceptance rates on various distributions. In this section, we examine the performance of the adaptive sampler. In addition to the expected jump distance (EJD), we also report the effective sample size (ESS) \footnote{Computed using \href{https://www.tensorflow.org/probability/api_docs/python/tfp/mcmc/effective_sample_size}{Tensorflow Probability}}. We compare the adaptive sampler ALBP, ARWM with their single step version LBP-1, RWM-1, and grid search version GLBP, GRWM, where we tune the scaling $R$ by grid search.
We give the sampling results on Bernoulli model, Ising model, FHMM, and RBM with medium size and configuration C2 in table~\ref{tab:selected_results}. More results are given in Appendix~\ref{appendix:exp}.
We can see that the adaptive samplers are significantly more efficient than single step samplers, especially for LBP. Also, the adaptive samplers can robustly achieve almost the same performance comparing to using grid search to find the optimal scaling.

\subsection{Training Deep Energy Based Models}
Learning an EBM is a challenge task. Given data sampled from a true distribution $\pi$, we maximize the likelihood of the target distribution $\pi_\theta(x) \propto e^{-f_\theta(x)}$ parameterized by $\theta$. The gradient estimation requires samples from the current model, which is typically obtained via MCMC. The speed of training an EBM is determined by how fast a MCMC algorithm can obtain a good estimate of the second expectation. 

We evaluate adaptive samplers by learning deep EBMs. Following the setting in \citet{grathwohl2021oops}, we train deep EBMs parameterized by Residual Networks \citep{he2016deep} on small binary image datasets using PCD \citep{tieleman2009using} with a replay buffer \citep{du2019implicit}. We compare two single step samplers and two adaptive samplers, where $\text{LBP}_b$ uses $g(t) = \frac{t}{t+1}$ as \textcolor{\cdiff}{weight} function and $\text{LBP}_s$ uses $g(t) = \sqrt{t}$ as weight function. When we allow them to run enough iterations in PCD, they are able to train EBMs in same good quality. To measure the efficiency of these samplers, we compare the minimum number of M-H steps needed in PCD in table~\ref{tab:deepEBM}. We can see that adaptive samplers only need one half or even one fifth iterations compare to single step samplers.
We also present long-run samples from our trained models via $\text{ALBP}_s$ in Figure \ref{fig:deep_samples}.

\begin{table*}[htb]
    \centering
    \caption{Minimum M-H Steps Needed for PCD}
    \begin{tabular}{ccccc}
    \toprule
    Dataset   & $\text{LBP}_b$-1 & $\text{ALBP}_b$ & $\text{LBP}_s$-1 & $\text{ALBP}_s$ \\
    \cmidrule(lr){1-1} \cmidrule(lr){2-3} \cmidrule(lr){4-5} 
    Static MNIST & 90 & 20 & 40 & 15 \\ 
    Dynamic MNIST & 100 & 20 & 40 & 15 \\
    Omniglot & 100 & 60 & 30 & 5 \\
    Caltech & 100 & 60 & 80 & 30 \\
    \bottomrule
    \end{tabular}
\label{tab:deepEBM}
\end{table*}%

\begin{figure}
    \centering
\begin{tabular}{@{}c@{\hskip 1mm}c@{\hskip 1mm}c}
    \includegraphics[width=0.25\textwidth]{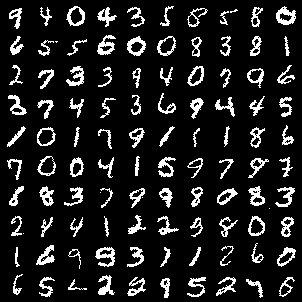} & 
    \includegraphics[width=0.25\textwidth]{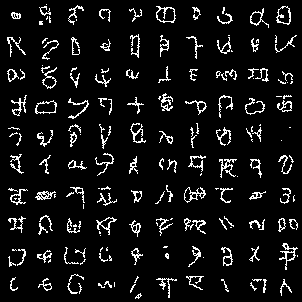} &
    \includegraphics[width=0.25\textwidth]{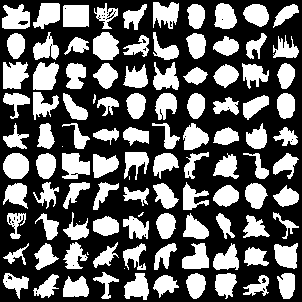} \\
    MNIST & Omniglot & Caltech \\
\end{tabular}
    \caption{Samples from deep EBMs trained by $\text{ALBP}_s$ sampler. \label{fig:deep_samples}}
\end{figure}

\section{Discussion}
\label{sec:discussion}
In this paper, we have addressed the optimal scaling problem for the locally balanced proposal (LBP) in \citep{sun2021path}. We verified, both theoretically and empirically, that the asymptotically optimal acceptance rate for LBP is $0.574$, independent of the target distribution. Moreover, knowledge of the optimal acceptance rate allows one to adaptively tune the neighborhood size for a proposal distribution in a discrete space. We verified the theoretical findings on a diverse set of distributions, and demonstrated that adaptive LBP can improve sampling efficiency for learning  deep EBMs. 

We believe there is considerable room for future work that builds on these results. For theoretical investigation, the theory established under a strong assumption that the target distribution is a product distribution, despite the results applies very well to more complicated distributions. We believe the results still hold under a weaker assumption that the target distribution has no phase transition.
We also believe it is possible to design a HMC style sampler for discrete spaces in the framework of \cite{sun2021path} by using LBP as a block for the auxiliary path. 
For empirical investigation, many real-world problems involve probability models of discrete structured data, such as syntax trees for natural language processing~\citep{tai2015improved}, program synthesis~\citep{dai2020learning}, and graphical models for molecules~\citep{gilmer2017neural}. Efficient discrete samplers should be able to accelerate both learning and inference with such models.

\section*{Acknowledgement}
We thank Vladimir Koltchinskii, Pengcheng Yin, Matthew D. Hoffman, and three anonymous reviewers for their helpful comments to improve the manuscript. Dale Schuurmans gratefully acknowledges the support of a Canada CIFAR AI Chair, NSERC and Amii.

{
\bibliography{reference}

\begin{thebibliography}{43}
\providecommand{\natexlab}[1]{#1}
\providecommand{\url}[1]{\texttt{#1}}
\expandafter\ifx\csname urlstyle\endcsname\relax
  \providecommand{\doi}[1]{doi: #1}\else
  \providecommand{\doi}{doi: \begingroup \urlstyle{rm}\Url}\fi

\bibitem[Andrieu \& Thoms(2008)Andrieu and Thoms]{andrieu2008tutorial}
Andrieu, C. and Thoms, J.
\newblock A tutorial on adaptive mcmc.
\newblock \emph{Statistics and computing}, 18\penalty0 (4):\penalty0 343--373,
  2008.

\bibitem[Beskos et~al.(2013)Beskos, Pillai, Roberts, Sanz-Serna, and
  Stuart]{beskos2013optimal}
Beskos, A., Pillai, N., Roberts, G., Sanz-Serna, J.-M., and Stuart, A.
\newblock Optimal tuning of the hybrid monte carlo algorithm.
\newblock \emph{Bernoulli}, 19\penalty0 (5A):\penalty0 1501--1534, 2013.

\bibitem[Dai et~al.(2020)Dai, Singh, Dai, Sutton, and
  Schuurmans]{dai2020learning}
Dai, H., Singh, R., Dai, B., Sutton, C., and Schuurmans, D.
\newblock Learning discrete energy-based models via auxiliary-variable local
  exploration.
\newblock \emph{arXiv preprint arXiv:2011.05363}, 2020.

\bibitem[Du \& Mordatch(2019)Du and Mordatch]{du2019implicit}
Du, Y. and Mordatch, I.
\newblock Implicit generation and generalization in energy-based models.
\newblock \emph{arXiv preprint arXiv:1903.08689}, 2019.

\bibitem[Gelman et~al.(1997)Gelman, Gilks, and Roberts]{gelman1997weak}
Gelman, A., Gilks, W.~R., and Roberts, G.~O.
\newblock Weak convergence and optimal scaling of random walk metropolis
  algorithms.
\newblock \emph{The annals of applied probability}, 7\penalty0 (1):\penalty0
  110--120, 1997.

\bibitem[Gelman et~al.(2013)Gelman, Carlin, Stern, Dunson, Vehtari, and
  Rubin]{gelman2013bayesian}
Gelman, A., Carlin, J.~B., Stern, H.~S., Dunson, D.~B., Vehtari, A., and Rubin,
  D.~B.
\newblock \emph{Bayesian data analysis}.
\newblock CRC press, 2013.

\bibitem[Ghahramani \& Jordan(1995)Ghahramani and
  Jordan]{ghahramani1995factorial}
Ghahramani, Z. and Jordan, M.
\newblock Factorial hidden markov models.
\newblock \emph{Advances in Neural Information Processing Systems}, 8, 1995.

\bibitem[Gilmer et~al.(2017)Gilmer, Schoenholz, Riley, Vinyals, and
  Dahl]{gilmer2017neural}
Gilmer, J., Schoenholz, S.~S., Riley, P.~F., Vinyals, O., and Dahl, G.~E.
\newblock Neural message passing for quantum chemistry.
\newblock In \emph{International conference on machine learning}, pp.\
  1263--1272. PMLR, 2017.

\bibitem[Girolami \& Calderhead(2011)Girolami and
  Calderhead]{girolami2011riemann}
Girolami, M. and Calderhead, B.
\newblock Riemann manifold langevin and hamiltonian monte carlo methods.
\newblock \emph{Journal of the Royal Statistical Society: Series B (Statistical
  Methodology)}, 73\penalty0 (2):\penalty0 123--214, 2011.

\bibitem[Grathwohl et~al.(2021)Grathwohl, Swersky, Hashemi, Duvenaud, and
  Maddison]{grathwohl2021oops}
Grathwohl, W., Swersky, K., Hashemi, M., Duvenaud, D., and Maddison, C.~J.
\newblock Oops i took a gradient: Scalable sampling for discrete distributions.
\newblock \emph{arXiv preprint arXiv:2102.04509}, 2021.

\bibitem[Haeusler(1988)]{haeusler1988rate}
Haeusler, E.
\newblock On the rate of convergence in the central limit theorem for
  martingales with discrete and continuous time.
\newblock \emph{The Annals of Probability}, pp.\  275--299, 1988.

\bibitem[Han et~al.(2020)Han, Ding, Liu, Torresani, Peng, and
  Liu]{han2020stein}
Han, J., Ding, F., Liu, X., Torresani, L., Peng, J., and Liu, Q.
\newblock Stein variational inference for discrete distributions.
\newblock In \emph{International Conference on Artificial Intelligence and
  Statistics}, pp.\  4563--4572. PMLR, 2020.

\bibitem[Hastings(1970)]{hastings1970monte}
Hastings, W.~K.
\newblock Monte carlo sampling methods using markov chains and their
  applications.
\newblock 1970.

\bibitem[He et~al.(2016)He, Zhang, Ren, and Sun]{he2016deep}
He, K., Zhang, X., Ren, S., and Sun, J.
\newblock Deep residual learning for image recognition.
\newblock In \emph{Proceedings of the IEEE conference on computer vision and
  pattern recognition}, pp.\  770--778, 2016.

\bibitem[Hinton(2002)]{hinton2002training}
Hinton, G.~E.
\newblock Training products of experts by minimizing contrastive divergence.
\newblock \emph{Neural computation}, 14\penalty0 (8):\penalty0 1771--1800,
  2002.

\bibitem[Hird et~al.(2020)Hird, Livingstone, and Zanella]{hird2020fresh}
Hird, M., Livingstone, S., and Zanella, G.
\newblock A fresh take on'barker dynamics' for mcmc.
\newblock \emph{arXiv preprint arXiv:2012.09731}, 2020.

\bibitem[Hirt et~al.(2021)Hirt, Titsias, and Dellaportas]{hirt2021entropy}
Hirt, M., Titsias, M., and Dellaportas, P.
\newblock Entropy-based adaptive hamiltonian monte carlo.
\newblock \emph{Advances in Neural Information Processing Systems}, 34, 2021.

\bibitem[Hoffman et~al.(2021)Hoffman, Radul, and Sountsov]{hoffman2021adaptive}
Hoffman, M., Radul, A., and Sountsov, P.
\newblock An adaptive-mcmc scheme for setting trajectory lengths in hamiltonian
  monte carlo.
\newblock In \emph{International Conference on Artificial Intelligence and
  Statistics}, pp.\  3907--3915. PMLR, 2021.

\bibitem[Hoffman et~al.(2014)Hoffman, Gelman, et~al.]{hoffman2014no}
Hoffman, M.~D., Gelman, A., et~al.
\newblock The no-u-turn sampler: adaptively setting path lengths in hamiltonian
  monte carlo.
\newblock \emph{J. Mach. Learn. Res.}, 15\penalty0 (1):\penalty0 1593--1623,
  2014.

\bibitem[Ising(1924)]{ising1924beitrag}
Ising, E.
\newblock \emph{Beitrag zur theorie des ferro-und paramagnetismus}.
\newblock PhD thesis, Grefe \& Tiedemann, 1924.

\bibitem[Livingstone \& Zanella(2019)Livingstone and
  Zanella]{livingstone2019barker}
Livingstone, S. and Zanella, G.
\newblock The barker proposal: combining robustness and efficiency in
  gradient-based mcmc.
\newblock \emph{arXiv preprint arXiv:1908.11812}, 2019.

\bibitem[Metropolis et~al.(1953)Metropolis, Rosenbluth, Rosenbluth, Teller, and
  Teller]{metropolis1953equation}
Metropolis, N., Rosenbluth, A.~W., Rosenbluth, M.~N., Teller, A.~H., and
  Teller, E.
\newblock Equation of state calculations by fast computing machines.
\newblock \emph{The journal of chemical physics}, 21\penalty0 (6):\penalty0
  1087--1092, 1953.

\bibitem[Neal et~al.(2011)]{neal2011mcmc}
Neal, R.~M. et~al.
\newblock Mcmc using hamiltonian dynamics.
\newblock \emph{Handbook of markov chain monte carlo}, 2\penalty0
  (11):\penalty0 2, 2011.

\bibitem[Nishimura et~al.(2017)Nishimura, Dunson, and
  Lu]{nishimura2017discontinuous}
Nishimura, A., Dunson, D., and Lu, J.
\newblock Discontinuous hamiltonian monte carlo for sampling discrete
  parameters.
\newblock \emph{arXiv preprint arXiv:1705.08510}, 853, 2017.

\bibitem[Pakman \& Paninski(2013)Pakman and Paninski]{pakman2013auxiliary}
Pakman, A. and Paninski, L.
\newblock Auxiliary-variable exact hamiltonian monte carlo samplers for binary
  distributions.
\newblock \emph{arXiv preprint arXiv:1311.2166}, 2013.

\bibitem[Power \& Goldman(2019)Power and Goldman]{power2019accelerated}
Power, S. and Goldman, J.~V.
\newblock Accelerated sampling on discrete spaces with non-reversible markov
  processes.
\newblock \emph{arXiv preprint arXiv:1912.04681}, 2019.

\bibitem[Robbins \& Monro(1951)Robbins and Monro]{robbins1951stochastic}
Robbins, H. and Monro, S.
\newblock A stochastic approximation method.
\newblock \emph{The annals of mathematical statistics}, pp.\  400--407, 1951.

\bibitem[Robert \& Casella(2013)Robert and Casella]{robert2013monte}
Robert, C. and Casella, G.
\newblock \emph{Monte Carlo statistical methods}.
\newblock Springer Science \& Business Media, 2013.

\bibitem[Roberts(1998)]{roberts1998cube}
Roberts, G.~O.
\newblock Optimal metropolis algorithms for product measures on the vertices of
  a hypercube.
\newblock \emph{Stochastics and Stochastic Reports}, 62\penalty0
  (3-4):\penalty0 275--283, 1998.

\bibitem[Roberts \& Rosenthal(1998)Roberts and Rosenthal]{roberts1998optimal}
Roberts, G.~O. and Rosenthal, J.~S.
\newblock Optimal scaling of discrete approximations to langevin diffusions.
\newblock \emph{Journal of the Royal Statistical Society: Series B (Statistical
  Methodology)}, 60\penalty0 (1):\penalty0 255--268, 1998.

\bibitem[Roberts \& Rosenthal(2001)Roberts and Rosenthal]{roberts2001optimal}
Roberts, G.~O. and Rosenthal, J.~S.
\newblock Optimal scaling for various metropolis-hastings algorithms.
\newblock \emph{Statistical science}, 16\penalty0 (4):\penalty0 351--367, 2001.

\bibitem[Rossky et~al.(1978)Rossky, Doll, and Friedman]{rossky1978brownian}
Rossky, P.~J., Doll, J., and Friedman, H.
\newblock Brownian dynamics as smart monte carlo simulation.
\newblock \emph{The Journal of Chemical Physics}, 69\penalty0 (10):\penalty0
  4628--4633, 1978.

\bibitem[Sansone(2021)]{sansone2021lsb}
Sansone, E.
\newblock Lsb: Local self-balancing mcmc in discrete spaces.
\newblock \emph{arXiv preprint arXiv:2109.03867}, 2021.

\bibitem[Smolensky(1986)]{smolensky1986information}
Smolensky, P.
\newblock Information processing in dynamical systems: Foundations of harmony
  theory.
\newblock Technical report, Colorado Univ at Boulder Dept of Computer Science,
  1986.

\bibitem[Sun et~al.(2021)Sun, Dai, Xia, and Ramamurthy]{sun2021path}
Sun, H., Dai, H., Xia, W., and Ramamurthy, A.
\newblock Path auxiliary proposal for mcmc in discrete space.
\newblock In \emph{International Conference on Learning Representations}, 2021.

\bibitem[Tai et~al.(2015)Tai, Socher, and Manning]{tai2015improved}
Tai, K.~S., Socher, R., and Manning, C.~D.
\newblock Improved semantic representations from tree-structured long
  short-term memory networks.
\newblock \emph{arXiv preprint arXiv:1503.00075}, 2015.

\bibitem[Tieleman \& Hinton(2009)Tieleman and Hinton]{tieleman2009using}
Tieleman, T. and Hinton, G.
\newblock Using fast weights to improve persistent contrastive divergence.
\newblock In \emph{Proceedings of the 26th annual international conference on
  machine learning}, pp.\  1033--1040, 2009.

\bibitem[Titsias \& Dellaportas(2019)Titsias and
  Dellaportas]{titsias2019gradient}
Titsias, M. and Dellaportas, P.
\newblock Gradient-based adaptive markov chain monte carlo.
\newblock \emph{Advances in Neural Information Processing Systems},
  32:\penalty0 15730--15739, 2019.

\bibitem[Titsias \& Yau(2017)Titsias and Yau]{titsias2017hamming}
Titsias, M.~K. and Yau, C.
\newblock The hamming ball sampler.
\newblock \emph{Journal of the American Statistical Association}, 112\penalty0
  (520):\penalty0 1598--1611, 2017.

\bibitem[Vogrinc et~al.(2022)Vogrinc, Livingstone, and
  Zanella]{vogrinc2022optimal}
Vogrinc, J., Livingstone, S., and Zanella, G.
\newblock Optimal design of the barker proposal and other locally-balanced
  metropolis-hastings algorithms.
\newblock \emph{arXiv preprint arXiv:2201.01123}, 2022.

\bibitem[Welling \& Teh(2011)Welling and Teh]{welling2011bayesian}
Welling, M. and Teh, Y.~W.
\newblock Bayesian learning via stochastic gradient langevin dynamics.
\newblock In \emph{Proceedings of the 28th international conference on machine
  learning (ICML-11)}, pp.\  681--688. Citeseer, 2011.

\bibitem[Zanella(2020)]{zanella2020informed}
Zanella, G.
\newblock Informed proposals for local mcmc in discrete spaces.
\newblock \emph{Journal of the American Statistical Association}, 115\penalty0
  (530):\penalty0 852--865, 2020.

\bibitem[Zhang et~al.(2012)Zhang, Ghahramani, Storkey, and
  Sutton]{zhang2012continuous}
Zhang, Y., Ghahramani, Z., Storkey, A.~J., and Sutton, C.
\newblock Continuous relaxations for discrete hamiltonian monte carlo.
\newblock \emph{Advances in Neural Information Processing Systems},
  25:\penalty0 3194--3202, 2012.

\end{thebibliography}
}
\bibliographystyle{icml2022}

\section*{Checklist}

The checklist follows the references.  Please
read the checklist guidelines carefully for information on how to answer these
questions.  For each question, change the default \answerTODO{} to \answerYes{},
\answerNo{}, or \answerNA{}.  You are strongly encouraged to include a {\bf
justification to your answer}, either by referencing the appropriate section of
your paper or providing a brief inline description.  For example:
\begin{itemize}
  \item Did you include the license to the code and datasets? \answerYes{See \url{https://github.com/ha0ransun/LBP_Scale.git}.}
  \item Did you include the license to the code and datasets? \answerYes{}
\end{itemize}
Please do not modify the questions and only use the provided macros for your
answers.  Note that the Checklist section does not count towards the page
limit.  In your paper, please delete this instructions block and only keep the
Checklist section heading above along with the questions/answers below.

\begin{enumerate}

\item For all authors...
\begin{enumerate}
  \item Do the main claims made in the abstract and introduction accurately reflect the paper's contributions and scope?
    \answerYes{}
  \item Did you describe the limitations of your work?
    \answerYes{See Section \ref{sec:discussion}}
  \item Did you discuss any potential negative societal impacts of your work?
    \answerNo{}
  \item Have you read the ethics review guidelines and ensured that your paper conforms to them?
    \answerYes{}
\end{enumerate}

\item If you are including theoretical results...
\begin{enumerate}
  \item Did you state the full set of assumptions of all theoretical results?
    \answerYes{}
        \item Did you include complete proofs of all theoretical results?
    \answerYes{See Appendix \ref{app:proof}}
\end{enumerate}

\item If you ran experiments...
\begin{enumerate}
  \item Did you include the code, data, and instructions needed to reproduce the main experimental results (either in the supplemental material or as a URL)?
    \answerYes{See supplemental material}
  \item Did you specify all the training details (e.g., data splits, hyperparameters, how they were chosen)?
    \answerYes{}
        \item Did you report error bars (e.g., with respect to the random seed after running experiments multiple times)?
    \answerNA{}
        \item Did you include the total amount of compute and the type of resources used (e.g., type of GPUs, internal cluster, or cloud provider)?
    \answerYes{}
\end{enumerate}

\item If you are using existing assets (e.g., code, data, models) or curating/releasing new assets...
\begin{enumerate}
  \item If your work uses existing assets, did you cite the creators?
    \answerNA{The work does not use existing assets.}
  \item Did you mention the license of the assets?
    \answerNA{}
  \item Did you include any new assets either in the supplemental material or as a URL?
    \answerYes{Codebase is included in supplemental materials.}
  \item Did you discuss whether and how consent was obtained from people whose data you're using/curating?
    \answerNA{}
  \item Did you discuss whether the data you are using/curating contains personally identifiable information or offensive content?
    \answerNA{}
\end{enumerate}

\item If you used crowdsourcing or conducted research with human subjects...
\begin{enumerate}
  \item Did you include the full text of instructions given to participants and screenshots, if applicable?
    \answerNA{}
  \item Did you describe any potential participant risks, with links to Institutional Review Board (IRB) approvals, if applicable?
    \answerNA{}
  \item Did you include the estimated hourly wage paid to participants and the total amount spent on participant compensation?
    \answerNA{}
\end{enumerate}

\end{enumerate}

\newpage
\appendix

\newpage
\appendix
\section{Complete Proof}
\label{app:proof}
\subsection{A concentration of $W(x, u)$}
\begin{lemma}
\label{lemma:concentration_w(x,u)}
Define $W = \mathbb{E}_{x, u} [W(x, u)]$.  We have:
\begin{equation}
    \mathbb{P}(|W(x, u) - W| > N^{\frac{1}{2}} t) \le 2 e^{-C_2 t^2}
\end{equation}
where $C_2$ is an absolute constant that only depends on the scalar $\epsilon$ in \eqref{eq:smooth_target}.
\end{lemma}

\begin{proof}
Define a martingale $M_n$, $n = 0, 1, ..., N + R$. \textcolor{\cdiff}{We let $M_0 = 0$}. When $n \le N$, it has independent increment
\begin{equation}
    M_n = \sum_{i=1}^n w_i(x) - \mathbb{E}[w_i(x)], \quad n = 1, ..., N
\end{equation}
For $n > N$, it is defined as
\begin{align}
    M_{N+r} 
    & = M_{N+r-1} - w_{u_r}(x) + \mathbb{E}[w_{u_r}(x)|M_1, ..., M_{N+r-1}] \\
    & =  M_{N+r-1} - w_{u_r}(x) + \frac{\sum_{i\notin u_{1:r-1}} w_i^2(x)}{\sum_{i\notin u_{1:r-1}} w_i(x)}
\end{align}
where $i \notin u_{1:r-1}$ means $i \neq u_j$ for $j = 1, ..., r-1$. Since $p_i$ are controlled by $\epsilon$ in \eqref{eq:smooth_target}, we can find a uniform bound
\begin{equation}
    \frac{1}{4C_1} = 2 \sup_{\epsilon < p < 1 - \epsilon} g(\frac{1- p}{p})
\end{equation}
For $1\le n \le N$, we have
\begin{equation}
    |M_n - M_{n-1}| = \textcolor{\cdiff}{\left|w_i(x) - \mathbb{E}[w_i(x)]\right|} \le 2 \max_{x, u} |w_i(x)| \le \frac{1}{4C_1}
\end{equation}
For $1 \le r \le R$, we have
\begin{equation}
    |M_{N+r} - M_{N+r-1}| = \textcolor{\cdiff}{\left|- w_{u_r}(x_{u_r}) + \frac{\sum_{i\neq u_{1:r-1}} w_i^2(x_i)}{\sum_{i\neq u_{1:r-1}} w_i(x_i)}\right|} \le \frac{1}{4C_1}
\end{equation}
Hence, we can apply the Azuma-Hoeffding inequality:
\begin{equation}
    \mathbb{P}(|W(x, u) - W| > \textcolor{\cdiff}{t}N^{\frac{1}{2}}) = \mathbb{P}(|M_n - M_0| > t N^{\frac{1}{2}}) \le 2 e^{\frac{-t^2 N}{2\frac{1}{4C_1}(N+R)}} = 2 e^{-C_1 t^2}. 
\end{equation}
Thus we prove the lemma.
\end{proof}
The lemma indicates with high probability, for arbitrary $\delta > 0$
\begin{equation}
    W(x, u) - W = o(N^{\frac{1}{2} + \delta})
\end{equation}
One observation of the proof is that, the concentration holds for arbitrary $0 \le R \le N$. For example, when $R = N$, $W(x, u) \equiv W \equiv 0$, the concentration is still valid.

\subsection{Lemma~\ref{lemma:taylor_sumlog}}
\begin{proof}
Using Taylor's series, we have
\begin{flalign}
    \log (1 + \sum_{i=r}^R w_{u_i}(x) / W(x, u)) = \frac{\sum_{i=r}^R w_{u_i}(x)}{W(x, u)} - \frac{1}{2} (\frac{\sum_{i=r}^R w_{u_i}(x)}{W(x, u)})^2 + O(\frac{R^3}{N^{3}})
\end{flalign}
\begin{flalign}
    \log (1 + \sum_{i=1}^r w_{u_i}(y) / W(x, u)) = \frac{\sum_{i=1}^r w_{u_i}(y)}{W(x, u)} - \frac{1}{2} (\frac{\sum_{i=1}^r w_{u_i}(y)}{W(x, u)})^2 + O(\frac{R^3}{N^3})
\end{flalign}

Using Lemma \ref{lemma:concentration_w(x,u)} and the property $W(x, u) = W(y, u)$, with high probability, the first order term becomes to:
\begin{align}
    \sum_{r=1}^R \frac{\sum_{i=r}^R w_{u_i}(x)}{W(x, u)} - \frac{\sum_{i=1}^r w_{u_i}(y)}{W(x, u)} 
    & = \sum_{r=1}^R \frac{(R-r+1)w_{u_i}(x) - r w_{u_i}(y)}{W(x, u)} \\
    & = \sum_{r=1}^R \frac{(R-r+1)w_{u_i}(x) - r w_{u_i}(y)}{W} + O(\frac{R^2}{N^{\frac{3}{2}-\delta}})
\end{align}

Similarly, with high probability, the second order term becomes to:
\begin{align}
    & \sum_{r=1}^R (\frac{\sum_{i=r}^R w_{u_i}(x)}{W(x, u)})^2 - (\frac{\sum_{i=1}^r w_{u_i}(y)}{W(x, u)})^2 \\
    = & \frac{1}{W(x, u)^2} \sum_{r=r}^R \Big( \sum_{i, j=r}^R w_{u_i}(x) w_{u_j}(x) - \sum_{i, j=1}^r w_{u_i}(y) w_{u_j}(y) \Big) \\
    = & \frac{1}{W(x, u)^2} \sum_{i=1}^R \sum_{j=1}^R \min\{i, j\} w_{u_i}(x) w_{u_j}(x) - (R - \max\{i, j\} + 1) w_{u_i}(y) w_{u_j}(y) \\
    = & \frac{1}{W^2} \sum_{i=1}^R \sum_{j=1}^R \min\{i, j\} w_{u_i}(x) w_{u_j}(x) - (R - \max\{i, j\} + 1) w_{u_i}(y) w_{u_j}(y) + o(\frac{R^3}{N^{\frac{5}{2}-\delta}})
\end{align}

Since $R = l N^\frac{2}{3}$, denote $i\land j = \min\{i, j\}, i\lor j = \max\{i, j\}$, with high probability, we have
\begin{flalign}
    & \sum_{r=1}^R \log \frac{1 + \sum_{i=r}^R w_{u_i}(x_{u_i}) / W(x, u)}{1 + \sum_{i=1}^r w_{u_i}(y_{u_i})/W(x,u)} \\
    =& \frac{1}{W}\sum_{r=1}^R (R-r+1)w_{u_i}(x) - r w_{u_i}(y) + o(N^{\frac{1}{12}-\delta}) \notag \\
    & \ \  - \frac{1}{2W^2} \sum_{i=1}^R \sum_{j=1}^R i\land j w_{u_i}(x)w_{u_j}(x) - (R - i \lor j + 1) w_{u_i}(y)w_{u_j}(y)
\end{flalign}
\textcolor{\cdiff}{Select $0 < \delta < \frac{1}{12}$, and the corresponding $t = N^\delta$}, we have, for large enough $N$, the above equation does not hold with probability exponentially small, and the term $o(N^{\frac{1}{12}-\delta})$ can be ignored. Hence we prove the weak convergence.
\end{proof}

\subsection{Proof for Lemma~\ref{lemma:posterior}}
\begin{proof}
The distribution $p(u_r|u_{1:r-1})$ can be approximated using the following tricks. First, using lemma \ref{lemma:concentration_w(x,u)}, with high probability, we have:
\begin{align}
    \mathbb{P}(u_r = i|u_{1:r-1}) 
    &= \mathbb{E}_{x \notin u_{1:r}} \left[ \frac{\mathbb{P}(x_i=1) w_i(1)}{W(x_{-i}, x_i = 1, u_{1:r-1})} + \frac{\mathbb{P}(x_i=0) w_i(0)}{W(x_{-i}, x_i=0, u_{1:r-1})} \right] \\
    &= \frac{p_i w_i(1) + (1 - p_i) w_i(0)}{W} + O(N^{-\frac{3}{2}})
\end{align}
Derive the similar result for $\mathbb{P}(u_r = j|u_{1:r-1})$. Since we have $R = l N^\frac{3}{2}$, for arbitrary $1 \le r \le R$, we have $W$ has the same order as $N$. Using the property of locally balanced function, where $p_i w_i(1) = (1 - p_i) w_i(0)$, we have
\begin{align}
    \frac{\mathbb{P}(u_1 = i)}{\mathbb{P}(u_1 = j)} 
    & = \frac{p_i w_i(1)}{p_j w_j(1)} + O(N^{-\frac{5}{2}})
\end{align}
Then, we use the identity:
\begin{align}
    1 
    &= \sum_{i=1}^N \mathbb{P}(u_1 = i) \\
    &= \sum_{j=1}^N \left(\frac{p_i w_i(1)}{p_j w_j(1)} + O(N^{-\frac{5}{2}})\right) \mathbb{P}(u_1 = j) \\
    &= \left(\frac{\sum_{i=1}^N p_i w_i(1)}{p_j w_j(1)} + O(N^{-\frac{3}{2}})\right) \mathbb{P}(u_1 = j)
\end{align}
hence, we have for the first step $u_1$:
\begin{equation}
    \mathbb{P}(u_1 = j) = \frac{p_j w_j(1)}{\sum_{i=1}^N p_i w_i(1)} + O(N^{-\frac{5}{2}})
\end{equation}
Recursively use this trick, for $1 \le r \le R = l N^\frac{2}{3}$ we have:
\begin{equation}
    \mathbb{P}(u_r = j|u_{1:r-1}) = \frac{p_j w_j(1) 1_{\{j \notin u_{1:r-1}\}}}{\sum_{i=1}^N p_i w_i(1) 1_{\{i \notin u_{1:r-1}\}}} + O(N^{-\frac{5}{2}}) 
\end{equation}

Next, we calculate the conditional probability for $x$. To simplify the notation, we denote $\mathbb{P}(x_j=1|u, u_r=j, x_{u_{1:j-1}})$ to represented index $j$ is selected at step $u_r$, and not been selected in all previous steps $u_1, ..., u_{r-1}$. Also, we denote
\begin{equation}
    W(x, u, s, t) = W(x, u) + \sum_{k=s}^t w_{u_k}(x) 
\end{equation}
In this way, the conditional probability for $x$ can be written as
\begin{align}
    &\mathbb{P}(x_j=1|u, u_r=j, x_{u_{1:j-1}}) \\
    = & \mathbb{E}[\frac{\pi_j(1)\prod_{l=1}^{r-1} (1 - \frac{w_j(1)}{W(x_{-j}, x_j=1, u, l, R)}) \frac{w_j(1)}{W(x_{-j}, x_j=1, u, r, R)} }{\sum_{v=0}^1 \pi_j(v)\prod_{l=1}^{r-1} (1 - \frac{w_j(1)}{W(x_{-j}, x_j=v, u, l, R)}) \frac{w_j(1)}{W(x_{-j}, x_j=v, u, r, R)} } |u, u_r=j, x_{u_{1:j-1}}] \\
    = & \mathbb{E}[\frac{\prod_{l=1}^{r-1} (1 - \frac{w_j(1)}{W(x_{-j}, x_j=1, u, l, R)}) \frac{1}{W(x_{-j}, x_j=1, u, r, R)} }{\sum_{v=0}^1 \prod_{l=1}^{r-1} (1 - \frac{w_j(1)}{W(x_{-j}, x_j=v, u, l, R)}) \frac{1}{W(x_{-j}, x_j=v, u, r, R)} } |u, u_r=j, x_{u_{1:j-1}}]
\end{align}
Since $R = l N^\frac{2}{3}$, according to lemma \ref{lemma:concentration_w(x,u)}, with high probability we have:
\begin{align}
    \frac{w_j(1)}{W(x_{-j}, x_j=v, u, l, R)} = \frac{w_j(1)}{W + O(N^\frac{1}{2}) + O(R)} =  \frac{w_j(1)}{W} + O(N^{-\frac{4}{3}})
\end{align}
Using this approximation, we have:
\begin{align}
    &\mathbb{P}(x_j=1|u, u_r=j, x_{u_{1:j-1}}) \\
    = & \mathbb{E}[\frac{\prod_{l=1}^{r-1} (1 - \frac{w_j(1)}{W} + O(N^{-\frac{4}{3}}) ) (\frac{1}{W} + O(N^{-\frac{4}{3}})) }{\sum_{v=0}^1 \prod_{l=1}^{r-1} (1 - \frac{w_j(v)}{W} + O(N^{-\frac{4}{3}}) ) (\frac{1}{W} + O(N^{-\frac{4}{3}}))} |u, u_r=j, x_{u_{1:j-1}}] \\
    = & \mathbb{E}[\frac{\prod_{l=1}^{r-1} (1 - \frac{w_j(1)}{W} + O(N^{-\frac{4}{3}}) )}{\sum_{v=0}^1 \prod_{l=1}^{r-1} (1 - \frac{w_j(v)}{W} + O(N^{-\frac{4}{3}}) )} \textcolor{\cdiff}{(1 + O(N^{-\frac{2}{3}}))} |u, u_r=j, x_{u_{1:j-1}}] \\
    = & \mathbb{E}[\frac{1 - (r-1)\frac{w_j(1)}{W} + (r-1)O(N^{-\frac{4}{3}})}{(1 - (r-1)\frac{w_j(0)}{W}) + (1 - (r-1)\frac{w_j(1)}{W}) + (r-1) O(N^{-\frac{4}{3}}) }|u, u_r=j, x_{u_{1:j-1}}] \\
    = & \mathbb{E}[\frac{1 - (r-1)\frac{w_j(1)}{W}}{(1 - (r-1)\frac{w_j(0)}{W}) + (1 - (r-1)\frac{w_j(1)}{W})} + (r-1) O(N^{-\frac{4}{3}})|u, u_r=j, x_{u_{1:j-1}}] \\
    %
    = & \frac{1}{2} + (r-1) \frac{w_j(0) - w_j(1) }{ 4W} + (r-1) O(N^{-\frac{4}{3}})
\end{align}
Thus we prove the lemma.
\end{proof}


\subsection{A Property for the conditional distribution of $u$}
The following result shows that marginal distribution for $u_1$ is a good approximation of the conditional distribution.
\begin{proposition}
\label{prop:cond}
For $N$ large enough, the conditional distribution for $u_r = j$ given $u_{1:r-1}$ can be approximated by the marginal distribution of $u_1$
\begin{align}
    & p(u_r = j| u_{1:r-1}, j \notin u_{1:r-1}) \\
    = & \mathbb{E}_{u_{1:r-1}}\Big[ \frac{p_j w_j(1)}
    {\sum_{i\notin u_{1:r-1}} p_i w_i(1)} \Big] + O(N^{-\frac{5}{2}})\\
    = & \mathbb{E}_{u_{1:r-1}}\Big[ \frac{p_j w_j(1)}{\sum_{i=1}^N p_i w_i(1)} + 
    \frac{p_j w_j(1) \sum_{i=1}^N p_i w_i(1) (1 - 1_{\{i \notin u_{1:r-1}\}})} 
    {(\sum_{i \notin u_{1:r-1}} p_i w_i(1)) (\sum_{i=1}^N p_i w_i(1))} \Big] + O(N^{-\frac{5}{2}})\\
    = & p(u_1 = j) + O(\frac{r}{N^2}) 
\end{align}
\end{proposition}

\subsection{Proof for Lemma~\ref{lemma:B}}
\begin{proof}
We first calculate its expectation using the conditional distribution derived in lemma~\ref{lemma:posterior}. To simplify the notation, we denote $\delta_w(i) = w_{u_i}(0) - w_{u_i}(1)$ for $i = 1, ..., R$ and 
\begin{align}
    S(i, j, k, l) &= i \land j w_{u_i}(k)w_{u_j}(l) - (R - i \lor j + 1) w_{u_i}(1-k)w_{u_j}(1-l) \\
    P(i, k) &= \frac{1}{2} - (-1)^k (i - 1) \frac{\delta_w(i)}{4W} + (i-1) O(N^{-\frac{4}{3}})
\end{align}
for $i, j = 1, ..., R$, and $k, l = 0, 1$. Then we have
\begin{align}
    & -\frac{1}{2W^{2}} \sum_{i=1}^{R} \sum_{j=1}^{R}[ i\land j w_{u_{i}}(x_{u_{i}}) w_{u_{j}}(x_{u_{j}}) - (R - i\lor j + 1) w_{u_i}(y_{u_i})w_{u_j}(y_{u_j}) | u]\\
    = & -\frac{1}{2W^2} \sum_{i, j=1}^R \sum_{k=0}^1 \sum_{l=0}^1  S(i, j, k, l)P(i, k)P(j, l) \label{eq:3rd}\\
    = & - \frac{1}{2W^2} \sum_{i, j=1}^R (R - (i + j) + 1)(w_{u_i}(0) + w_{u_i}(1))(w_{u_j}(0) + w_{u_j}(1)) + O(\frac{R^2}{N}) \label{eq:2nd} \\
    = & - \frac{1}{2W^2} \sum_{i, j=1}^R (R - (i + j) + 1)(w_{u_i}(0) + w_{u_i}(1))(w_{u_j}(0) + w_{u_j}(1)) + O(\frac{R^4}{N^3})
\end{align}
The remaining expectation is with respect to $u$. From  proposition \ref{prop:cond}, we know that the conditional expectation of $u_i$ can be estimated via the marginal distribution of $u_1$. In fact, when $R = l N^\frac{2}{3}$, we have:
\begin{align}
    & \mathbb{E}[w_{u_r}(0) + w_{u_r}(1)|u_{1:r-1}] \\
    = & \mathbb{E}[\sum_{j=1}^N (w_j(1) + w_j(0)) (\frac{p_jw_j(1)}{\sum_{i=1}^N p_iw_i(1)} + O(\frac{R}{N^2}) | u_{1:r-1}] \\
    = & \mathbb{E}[w_{u_1}(0) + w_{u_1}(1)] + O(N^{-\frac{4}{3}})
\end{align}
and similarly, we have:
\begin{equation}
    \mathbb{E}[(w_{u_r}(0) + w_{u_r}(1))^2|u_{1:r-1}] = \mathbb{E}[(w_{u_1}(0) + w_{u_1}(1))^2] + O(N^{-\frac{4}{3}})
\end{equation}
Using these properties, we have
\begin{align}
    & \mathbb{E}[\sum_{i, j=1}^R (R - (i + j) + 1)(w_{u_i}(0) + w_{u_i}(1))(w_{u_j}(0) + w_{u_j}(1))] \\
    = & \mathbb{E}[ \mathbb{E}[\cdots \mathbb{E}[2\sum_{i=1}^R \sum_{j>i}^R (R - (i + j) + 1)(w_{u_i}(0) + w_{u_i}(1))(w_{u_j}(0) + w_{u_j}(1)) \notag \\
    & \quad + \sum_{r=1}^R (R-2r+1) (w_{u_r}(0) + w_{u_r}(1))^2|u_{1:R-1}] \cdots |u_1] ] \\
    = & \mathbb{E} [2\sum_{i=1}^R \sum_{j>i}^R (R - (i + j) + 1)(w_{u_1}(0) + w_{u_1}(1))(w_{u_1}(0) + w_{u_1}(1))] + O(N^\frac{2}{3}) \notag \\
    & \quad + \mathbb{E}[\sum_{r=1}^R (R-2r+1) (w_{u_1}(0) + w_{u_1}(1))^2] + O(1) \\
    = & (w_{u_1}(0) + w_{u_1}(1))^2 \sum_{i, j=1}^N (R - (i+j) + 1)  + O(N^\frac{2}{3}) \\
    = & O(N^\frac{2}{3})
\end{align}
Hence, we prove that
\begin{equation}
    \mathbb{E}[-\frac{1}{2W^{2}} \sum_{i=1}^{R} \sum_{j=1}^{R}[i \land j w_{u_{i}}(x_{u_{i}}) w_{u_{j}}(x_{u_{j}}) - (R - i \lor j + 1) w_{u_i}(y_{u_i})w_{u_j}(y_{u_j}) ] = O(N^{-\frac{4}{3}})
\end{equation}
The expectation of the $B$ \eqref{eq:B} is small. To show it is save to ignore, we will prove the concentration property. Consider a function of $x$ and $u$:
\begin{equation}
    F(x, u) = -\frac{1}{2} \frac{1}{W^{2}} \sum_{i=1}^{R} \sum_{j=1}^{R}[i \land j w_{u_{i}}(x_{u_{i}}) w_{u_{j}}(x_{u_{j}}) - (R - i\lor j + 1) w_{u_i}(y_{u_i})w_{u_j}(y_{u_j})
\end{equation}
where $y$ is obtained by flipping indices $u$ of $x$. For changing $x$, we have:
\begin{equation}
    |F(x_1, ..., x_j, ..., x_N, u_1, ..., u_R) - F(x_1, ..., x'_j, ..., x_N, u_1, ..., u_R)| \le c_j
\end{equation}
where $c_j = 0$ if $j \notin u$ or $c_j = O(\frac{R^2}{N^2})$ if there exists $r$ and $u_r = j$. For chaning $u$, we have
\begin{equation}
    |F(x_1, ..., \textcolor{\cdiff}{x_N}, u_1, ..., u_i, ... u_R) - F(x_1, ..., x_N, u_1, ..., \textcolor{\cdiff}{u'_i, ..., u_R})| \le d_i
\end{equation}
where $d_i = O(\frac{R^2}{N^2})$ for $i = 1, ..., R$. By McDiarmid's inequality, we have:
\begin{equation}
    \mathbb{P}(|F(x, u) - \mathbb{E}[F(x, u)] \ge t \frac{R^\frac{5}{2}}{N^\frac{7}{4}}) \le 2 \exp(-\frac{2t^2 R^5 / \textcolor{\cdiff}{N^\frac{7}{2}}}{\sum_{j=1}^N c^2_j + \sum_{i=1}^R d^2_i}) \lesssim \exp(-2t^2N^\frac{1}{2})
\end{equation}
Hence, $F(x, u)$ will concentrate to its expectation at scale $O(R^\frac{5}{2} / N^\frac{7}{4})$. Since $R = l N^\frac{2}{3}$, with probability larger than $1 - O(\exp(-N^\frac{1}{2}))$, $B = O(N^{-\frac{1}{12}})$. 
\end{proof}

\subsection{Lemma~\ref{lemma:A_good}}
\begin{proof}
To show that $A$ weakly converges to a normal distribution, we use martingale central limit theorem. Define a martingale $M_n$, for $n=0, 1, ..., 2R$. When $n \le R$, we let the process $M_n = 0$ and the filter $F_n$ as the $\sigma$-algebra determined by $u_1, ..., u_n$. For $R + 1 \le R + n \le 2R$, define
\begin{align}
    M_{R+n} = M_{R + n-1} + \frac{1}{W}\Big( &(R - r + 1) w_{u_n}(x_n) - rw_{u_n}(1 - x_{u_n}) \notag \\
    &- \mathbb{E}[(R - r + 1) w_{u_n}(x_n) - rw_{u_n}(1 - x_{u_n})]\Big)
\end{align}
We first estimate the mean of the increment using the conditional probability derived in lemma~\ref{lemma:posterior}. If $n \le  R$, the mean is obviously $0$, else
\begin{align}
    & \mathbb{E}[\frac{(R - r + 1)w_{u_r}(x_{u_r}) - r w_{u_r}(y_{u_r})}{W}|u_r = j] \\
    =& \frac{(R - r + 1)w_j(1) - r w_j(0)}{W} (\frac{1}{2} + r \frac{w_j(0) - w_j(1)}{W} + O(\frac{R}{N^\frac{3}{2}} + \frac{R^2}{N^2})) \\
    & + \frac{(R - r + 1)w_j(0) - r w_j(1)}{W} (\frac{1}{2} - r \frac{w_j(0) - w_j(1)}{W} + O(\frac{R}{N^\frac{3}{2}} + \frac{R^2}{N^2})) \\
    =& \frac{1}{2}\frac{R-2r+1}{W}(w_j(1) + w_j(0)) - \frac{r(R+1)}{4W^2}(w_j(0) - w_j(1))^2 + O(\frac{R^2}{N^\frac{5}{2}} + \frac{R^3}{N^3})
\end{align}
Then we estimate the variance of $M_n - M_{n-1}$. We start with estimating the 2nd moment.
\begin{align}
    & \mathbb{E}[(\frac{(R - r + 1)w_{u_r}(x_{u_r}) - r w_{u_r}(y_{u_r})}{W})^2|u_r=j] \\
    =& (\frac{(R - r + 1)w_j(1) - r w_j(0)}{W})^2 (\frac{1}{2} + r \frac{w_j(0) - w_j(1)}{W}) + O(\frac{R}{N^\frac{3}{2}} + \frac{R^2}{N^2})) \\
    & + \frac{((R - r + 1)w_j(0) - r w_j(1)}{W})^2 (\frac{1}{2} - r \frac{w_j(0) - w_j(1)}{W}) + O(\frac{R}{N^\frac{3}{2}} + \frac{R^2}{N^2})) \\
    =& \frac{1}{2}((R-r+1)^2 + r^2)\frac{w_j^2(0) + w_j^2(1)}{W^2} - 2r(R-r+1)\frac{w_j(0)w_j(1)}{W^2} + O(\frac{R^3}{N^\frac{7}{2}} + \frac{R^4}{N^4})
\end{align}
Then, we are able to calculate the variance:
\begin{align}
    & \text{var} [\frac{(R - r + 1)w_{u_r}(x_{u_r}) - r w_{u_r}(y_{u_r})}{W}|u_r=j] \\
    = & \mathbb{E}[(\frac{(R - r + 1)w_{u_r}(x_{u_r}) - r w_{u_r}(y_{u_r})}{W})^2|u_r=j] \notag \\
    & - \mathbb{E}^2[\frac{(R - r + 1)w_{u_r}(x_{u_r}) - r w_{u_r}(y_{u_r})}{W}|u_r=j] \\
    = & \frac{(R+1)^2}{4}\frac{w_j^2(0) + w_j^2(1)}{W^2} - \frac{(R+1)^2}{2}\frac{w_j(0)w_j(1)}{W^2} + O(\frac{R^2}{N^\frac{5}{2}} + \frac{R^3}{N^3}) \\
    = & \frac{(R+1)^2}{4W^2}(w_j(0)- w_j(1))^2 + O(\frac{R^2}{N^\frac{5}{2}} + \frac{R^3}{N^3})
\end{align}

We calculate the value of its mean $\mu$ and variance $\sigma^2$. 
\begin{align}
    \mu 
    & = \mathbb{E}[\sum_{r=1}^R \frac{(R - r + 1)w_{u_r}(x_{u_r}) - r w_{u_r}(y_{u_r})}{W}|u] \\
    & = \sum_{r=1}^R \frac{1}{2}\frac{R-2r+1}{W}(w_{u_r}(1) + w_{u_r}(0)) - \frac{r(R+1)}{4W^2}(w_{u_r}(0) - w_{u_r}(1))^2 \\
    \sigma^2
    & = \sum_{r=1}^R \text{var} [\frac{(R - r + 1)w_{u_r}(x_{u_r}) - r w_{u_r}(y_{u_r})}{W}|u] \\
    & = \sum_{r=1}^R \frac{(R+1)^2}{4W^2}(w_{u_r}(0)- w_{u_r}(1))^2
\end{align}

Define $\mu_1 = \mathbb{E}[w_{u_1}(1) + w_{u_1}(0)]$. For the first part in $\mu$, using proposition \ref{prop:cond}, we have
\begin{align}
    & \mathbb{E}[\sum_{r=1}^R \frac{R-2r+1}{W} (w_{u_r}(1) + w_{u_r}(0))] \\
    = & \mathbb{E}[\sum_{r=1}^R \frac{R-2r+1}{W} \mu_1 + O(N^{-\frac{5}{3}})] \\
    = & O(N^{-\frac{2}{3}})
\end{align}

Define $\sigma_1^2 = \mathbb{E}[(w_{u_1}(0) - w_{u_1}(1))^2]$, From lemma \ref{lemma:posterior}, we have
\begin{equation}
    \mathbb{E}[(w_{u_r}(0) - w_{u_r}(1))^2] = \sigma_1^2 + O(N^{-\frac{4}{3}}), \quad \forall r = 1, ..., R
\end{equation}
for the second term in $\mu$, we have
\begin{align}
    \sum_{r=1}^R  - \frac{r(R+1)}{4W^2}(w_{u_r}(0) - w_{u_r}(1))^2 = -\frac{R(R+1)^2}{8W^2}\sigma_1^2 + O(N^{-\frac{4}{3}})
\end{align}
for the variance $\sigma^2$, we have:
\begin{align}
    \sum_{r=1}^R  \frac{(R+1)^2}{4W^2}(w_{u_r}(0) - w_{u_r}(1))^2 = \frac{R(R+1)^2}{4W^2}\sigma_1^2 + O(N^{-\frac{4}{3}})
\end{align}
Finally, we will decouple $R$ with $W$. Specifically:
\begin{align}
    \frac{1}{W^2} = \frac{1}{\mathbb{E}^2[\sum_{k\notin u} w_k(x_k)]} = \frac{1}{\mathbb{E}^2[\sum_{k=1}^N w_k(x_k)]} + O(N^{-\frac{8}{3}})
\end{align}

Combine everything together\textcolor{\cdiff}{,} we have
\begin{align}
    \mu &= -\frac{R(R+1)^2}{8\mathbb{E}^2[\sum_{k=1}^N w_k(x_k)]}\sigma_1^2 + O(N^{-\frac{2}{3}}) \\
    \sigma^2 &= \frac{R(R+1)^2}{4\mathbb{E}^2[\sum_{k=1}^N w_k(x_k)]}\sigma_1^2 + O(N^{-\frac{4}{3}}) 
\end{align}
Since $R = l N^\frac{2}{3}$, we have the sum of the conditional variance is $O(1)$ and the reminder is $o(1)$. For a martingale, we need to check one more step. We know $|M_n - M_{n-1}| = 0$ for $n \le R$. For $n + R > R$, we have:
\begin{align}
    |M_{R+n} - M_{R+n-1}| 
    & = \frac{1}{W}\big((R - r + 1) w_{u_r}(x) - rw_{u_r}(y) - \mathbb{E}[(R - r + 1) w_{u_r}(x) - rw_{u_r}(y)]\big) \\
    & = O(\frac{R}{N}) = O(N^{-\frac{1}{3}})
\end{align}
is uniformly bounded by a constant independent of $N$ and $R$. 
We denote
\begin{equation}
    \lambda_1^2 = \frac{\sum_{j=1}^N p_j w_j(1) (w_j(0) - w_j(1))^2}{4 \mathbb{E}^2[\frac{1}{N}\sum_{k=1}^N w_k(x_k)] \sum_{i=1}^N p_i w_i(1)}
\end{equation}
Then we can rewrite:
\begin{align}
    \mu = -\frac{1}{2}\lambda_1^2 l^3 \\
    \sigma^2 = \lambda_1^2 l^3
\end{align}
By martingale central limit theorem, we have that:
\begin{equation}
    \frac{A - \mu}{\sigma} \longrightarrow_\text{dist.} \mathcal{N}(0, 1)
\end{equation}
Furthermore, we use the convergence rate in \cite{haeusler1988rate}, we have:
\begin{align}
    L_{R, 2 \delta} & \equiv \sum_{r=1}^{2R} E\left(\left|M_r - M_{r-1}\right|^{2+2 \delta}\right) = O(\frac{R^{3 + 2\delta}}{N^{2 + 2\delta}}) \\
    M_{R, 2 \delta} & \equiv \mathbb{E}[|\sum_{r=1}^{2R} \mathbb{E}[(M_r - M_{r-1})^2|F_{r-1}] - 1 |^{1 + \delta}] = O(\frac{R^{4+4\delta}}{N^{4+4\delta}})
\end{align}
Then we have the probability
\begin{equation}
    |\mathbb{P}(\frac{A-\mu}{\sigma} \le t) - \Phi(t)| \le D_R
\end{equation}
where
\begin{equation}
    D_R \le C_\delta (L_{R, 2\delta} + M_{R, 2\delta})^\frac{1}{3 + 2\delta} = O(R / N^\frac{2 + 2\delta}{3 + 2\delta}), \quad \forall \delta > 0
\end{equation}
where $C_\delta$ is an absolute constant that only depends on $\delta$. We select $\delta = \frac{1}{2}$, we have:
\begin{equation}
    |\mathbb{P}(\frac{A-\mu}{\sigma} \le t) - \Phi(t)| \le O(R / N^\frac{3}{4})
\end{equation}
Since we consider $R= l N^\frac{2}{3}$, we prove the lemma.
\end{proof}


\subsection{Proof of Lemma~\ref{lemma:normal}}
\begin{proof}
Assume $Z \sim \mathcal{N}(\mu, \sigma^2)$, then we have:
\begin{align}
    \mathbb{E}\min\{1, e^{Z}\} 
    & = \int_{-\infty}^0 e^z \frac{1}{\sqrt{2\pi}\sigma} e^{-\frac{(z - \mu)^2}{2\sigma^2}} dz + \int_0^\infty \frac{1}{\sqrt{2\pi}\sigma} e^{-\frac{(z - \mu)^2}{2\sigma^2}} dz \\
    &= \int_{-\infty}^0 \frac{1}{\sqrt{2\pi}\sigma} e^{-\frac{z^2 - 2\mu z + \mu^2 - 2 \sigma^2 z}{2\sigma^2}} dz + \int_{-\mu}^\infty \frac{1}{\sqrt{2\pi}\sigma} e^{-\frac{z^2}{2\sigma^2}} dz \\
    &= \exp(\mu + \frac{\sigma^2}{2}) \int_{-\infty}^0 \frac{1}{\sqrt{2\pi}\sigma} e^{-\frac{(z - (\mu + \sigma^2))^2}{2\sigma^2}} dz + \int_{-\mu}^\infty \frac{1}{\sqrt{2\pi}\sigma} e^{-\frac{z^2}{2\sigma^2}} dz \\
    &= \exp(\mu + \frac{\sigma^2}{2}) \int_{-\infty}^{-\mu -\sigma^2} \frac{1}{\sqrt{2\pi}\sigma} e^{-\frac{z^2}{2\sigma^2}} dz + \int_{-\infty}^\mu \frac{1}{\sqrt{2\pi}\sigma} e^{-\frac{z^2}{2\sigma^2}} dz \\
    &= \exp(\mu + \frac{\sigma^2}{2}) \Phi(-\frac{\mu}{\sigma} - \sigma) + \Phi(\frac{\mu}{\sigma}) 
\end{align}
Specially, when $\mu = -\frac{1}{2} \sigma^2$, we have:
\begin{equation}
    \mathbb{E}\min\{1, e^Z\} = 2\Phi(-\frac{1}{2}\sigma)
\end{equation}
\end{proof}

\subsection{Proof for Theorem~\ref{thm:rw}}
\begin{proof}
In RWM-R, the proposal distribution is uniform, hence we only need to consider the probability ratio in the acceptance rate. Given current state $x$ and the picked indices $u$, the proposed state $y$ is obtained by flipping indices $u$ of $x$. The acceptance rate is:
\begin{align}
    A(x, y, u) 
    & = 1 \land \frac{\pi(y)}{\pi(x)} \\
    & = 1 \land \prod_{r=1}^R \frac{\pi_{u_r}(y)}{\pi_{u_r}(x)} \\
    & = 1 \land \prod_{r=1}^R \frac{p_{u_r}^{y_{u_r}} (1 - p_{u_r})^{1 - y_{u_r}}}{p_{u_r}^{x_{u_r}} (1 - p_{u_r})^{1 - x_{u_r}}} \\
    & = 1 \land \prod_{r=1}^R p_{u_r}^{1 - 2 x_{u_r}} (1 - p_{u_r})^{2 x_{u_r} - 1} \\
    & = 1 \land \exp(\sum_{r=1}^R (1 - 2 x_{u_r})\log \frac{p_{u_r}}{1 - p_{u_r}})
\end{align}

Define the martingale $M_n$, $n=1, ..., 2R$. For $r=1, ..., R$, we have $M_r = 0$ and the filtration $F_r$ is determined by the $\sigma$-algebra of $u_1, ..., u_R$. For $ R+1 \le R + n \le 2R$, we have:
\begin{equation}
    M_{R+n} = M_{R + n-1} + (1 - 2x_{u_n}) \log \frac{p_{u_n}}{1 - p_{u_n}} -  \mathbb{E}[(1 - 2x_{u_n}) \log \frac{p_{u_n}}{1 - p_{u_n}}]
\end{equation}
Hence, for $n \le R$, the increment is $0$. For $n + R > R$, denote the mean of the increment is :
\begin{align}
    \mathbb{E}[(1 - 2x_{u_n}) \log \frac{p_{u_n}}{1 - p_{u_n}}] = (1 - 2p_{u_n}) \log \frac{p_{u_n}}{1 - p_{u_n}}
\end{align}
the variance of the increment is:
\begin{align}
    & \mathbb{E}[(M_{R+j} - M_{R+j-1})^2|u, x_{1:j-1}] \\
    = & \mathbb{E}[((1 - 2x_{u_n}) \log \frac{p_{u_n}}{1 - p_{u_n}} -  \mathbb{E}[(1 - 2x_{u_n}) \log \frac{p_{u_n}}{1 - p_{u_n}}])^2 ] \\
    = & \mathbb{E}[((1 - 2x_{u_n}) \log \frac{p_{u_n}}{1 - p_{u_n}})^2] - \mathbb{E}^2[(1 - 2x_{u_n}) \log \frac{p_{u_n}}{1 - p_{u_n}}] \\
    = & (\log \frac{p_{u_n}}{1 - p_{u_n}})^2 - (1 - 2 p_{u_n})^2 (\log \frac{p_{u_n}}{1 - p_{u_n}})^2 \\
    = & 4p_j(1-p_{u_n}) (\log \frac{p_{u_n}}{1 - p_{u_n}})^2
\end{align}
 When $N$ is large, we have $p_{u_n} - \frac{1}{2} = O(N^{-\beta})$, hence
\begin{align}
    & 4p_{u_n}(1-p_{u_n}) (\log \frac{p_{u_n}}{1 - p_{u_n}})^2 \\
    = & 4p_{u_n}(1-p_{u_n}) \log(1 + \frac{2p_{u_n}-1}{p_{u_n}}) \log(\frac{p_{u_n}}{1 - p_{u_n}}) \\
    = & 4(\frac{1}{2} + O(N^{-\beta})) (1-p_{u_n}) (\frac{2p_{u_n}-1}{1 - p_{u_n}} + O(N^{-2\beta}) ) \log(\frac{p_{u_n}}{1 - p_{u_n}}) \\
    = & 2(2p_{u_n} - 1) \log(\frac{p_{u_n}}{1 - p_{u_n}})(1 + O(N^{-\beta}))
\end{align}
is negative twice of the corresponding mean. Since the indices $u$ are uniformly picked, the conditional distribution of $u_r$ is:
\begin{equation}
    \mathbb{P}(u_r = j | u_{1:r-1}) = \frac{1_{\{j \notin u_{1:r-1}\}}}{\sum_{i=1}^N 1_{\{i \notin u_{1:r-1}\}}} = \frac{1}{N} + O(\frac{R}{N^2})
\end{equation}
Hence, we have the mean is
\begin{align}
    \mu &= \mathbb{E}[\sum_{r=1}^R (1 - 2 x_{u_n}) \log \frac{p_{u_n}}{1 - p_{u_n}}] \\
    & = \mathbb{E}[R (1 - 2 x_{u_1}) \log \frac{p_{u_1}}{1 - p_{u_1}} + O(\frac{R^2}{N^2})] \\
    & = \frac{R}{N^{2\beta}} \frac{1}{N} \sum_{i=1}^N N^{2\beta}(1 - 2 p_i) \log \frac{p_i}{1 - p_i} + O(\frac{R^2}{N^2}) 
\end{align}
Similarly, we have the variance is:
\begin{equation}
    \sigma^2 = \mathbb{E}[\sum_{r=1}^R 2(2 x_{u_n} - 1) \log \frac{p_{u_n}}{1 - p_{u_n}}] = \frac{R}{N^{2\beta}} \frac{2}{N} \sum_{i=1}^N N^{2\beta}(2 p_i - 1) \log \frac{p_i}{1 - p_i} + O(\frac{R^2}{N^2}) 
\end{equation}
When $R = O(N^{2\beta}$), the variance is at a constant order. For a martingale, we also need to check the increments are uniformly bounded. When $n \le R$, the increment is always $0$. When $R + 1 \le R + n \le 2R$, we have:
\begin{equation}
    |M_{R+n} - M_{R+n-1}| = |(1 - 2 x_{u_n})\log\frac{p_{u_n}}{1 - p_{u_n}} - \mathbb{E}[(1 - 2 x_{u_n})\log\frac{p_{u_n}}{1 - p_{u_n}}]| \le C(\epsilon)
\end{equation}
where $C(\epsilon)$ is a constant only determined by $\epsilon$. Hence, by martingale central limit theorem, we have the distribution of $M_{2R}$ converges to a normal distribution. Denote
\begin{equation}
    \lambda_2^2 = \frac{2}{N} \sum_{i=1}^N N^{2\beta}(2 p_i - 1) \log \frac{p_i}{1 - p_i}
\end{equation}
Then we can rewrite:
\begin{align}
    \mu &= -\frac{1}{2}\lambda_2^2 \frac{R}{N^{2\beta}} \\
    \sigma^2 &= \lambda_2^2 \frac{R}{N^{2\beta}}
\end{align}
Denote $Z = \sum_{r=1}^R (1 - 2 x_{u_r}) \log \frac{p_{u_r}}{1 - p_{u_r}}$. By martingale central limit theorem, we have
\begin{equation}
    \frac{Z - \mu}{\sigma} \longrightarrow_\text{dist.} \mathcal{N}(0, 1)
\end{equation}
Furthermore, using the convergence rate in \cite{haeusler1988rate}, we have:
\begin{align}
    L_{R, 2 \delta} & \equiv \sum_{r=1}^{2R} E\left(\left|M_r - M_{r-1}\right|^{2+2 \delta}\right) = O(\frac{R}{N^{(4 + 4\delta)\beta}}) \\
    M_{R, 2 \delta} & \equiv \mathbb{E}[|\sum_{r=1}^{2R} \mathbb{E}[(M_r - M_{r-1})^2|F_{r-1}] - 1 |^{1 + \delta}] = O(\frac{R^{2+2\delta}}{N^{2+2\delta}})
\end{align}
Then we have the probability
\begin{equation}
    |\mathbb{P}(\frac{A-\mu}{\sigma} \le t) - \Phi(t)| \le D_R
\end{equation}
where
\begin{equation}
    D_R \le C_\delta (L_{R, 2\delta} + M_{R, 2\delta})^\frac{1}{3 + 2\delta} = O(R^\frac{1}{3 + 2\delta} / N^\frac{4 + 4\delta}{3 + 2\delta}), \quad \forall \delta > 0
\end{equation}
where $C_\delta$ is an absolute constant that only depends on $\delta$. We select $\delta = \frac{1}{2}$, we have:
\begin{equation}
    |\mathbb{P}(\frac{A-\mu}{\sigma} \le t) - \Phi(t)| \le O(R^\frac{1}{4} / N^\frac{5}{4})
\end{equation}

Hence, the expectation w.r.t. $\sum_{r=1}^R (1 - 2 x_{u_r}) \log \frac{p_{u_r}}{1 - p_{u_r}}$ converges to the expectation w.r.t.
\begin{equation}
    \mathcal{N}(-\frac{1}{2}\lambda_2^2 \frac{R}{N^{2\beta}}, \lambda_2^2 \frac{R}{N^{2\beta}}) 
\end{equation}
Using lemma~\ref{lemma:normal}, we have the acceptance rate converges to:
\begin{equation}
    a(R) = 2\Phi(-\frac{1}{2} \lambda_2 \frac{R^\frac{1}{2}}{N^\beta})
\end{equation}
In RWM-R, the distance between the current state $x$ and the proposed state $y$ is always $d(x, y) = R$, hence we have:
\begin{equation}
    \rho(R) = R a(R) = 2R \Phi(-\frac{1}{2} \lambda_2 \frac{R^\frac{1}{2}}{N^\beta})
\end{equation}

When $R = \omega(N^{2\beta})$, we can give a concentration property. Since the selection of $u_r$ is a martingale, we can apply Azuma-Hoeffding inequality:
\begin{align}
    \mathbb{P}(|M_{2R} - \mu| > t \lambda_2 R^\frac{3}{4} / N^{\frac{3}{2}\beta})
    \lesssim 2\exp(- \frac{2t^2 R^\frac{3}{2} /N^{3\beta}}{R N^{-2\beta}}) = 2\exp(-2t^2 R^\frac{1}{2} / N^\beta)
\end{align}
Hence, When $N$ is sufficiently large, with probability larger than $1 - O(\exp(-2t^2 R^\frac{1}{2} / N^\beta))$, we have:
\begin{equation}
    \sum_{r=1}^R (1 - 2 x_{u_r}) \log \frac{p_{u_r}}{1 - p_{u_r}} = -\frac{1}{2}\lambda_2^2 \frac{R}{N^{2\beta}} + \textcolor{\cdiff}{O(\frac{t R^\frac{3}{4}}{N^{\frac{3}{2}\beta}})} = -\frac{C}{2} \lambda_2^2 \frac{R}{N^{2\beta}}
\end{equation}
For $C > 0$ independent with $N, R$.
\end{proof}

\subsection{Proof for Corollary \ref{cor:rw_rate}}
\begin{proof}
When $R = O(N^{2\beta})$, denote $z = R \lambda_2^2 / N^{2\beta}$
\begin{align}
    & \rho(R) = 2 R \Phi(-\frac{1}{2} \lambda_2 \frac{R^\frac{1}{2}}{N^\beta}) \\
    = & 2 (N^{2\beta}/R) (R \lambda_2^2 / N^{2\beta}) \Phi(-\frac{1}{2} ( (R \lambda_2^2 / N^{2\beta})^\frac{1}{2}) \\
    = & 2 (N^{2\beta}/R) z \Phi(-\frac{1}{2} z^\frac{1}{2})
\end{align}
which means the optimal value of $z$ is independent of the target distribution. As $\Phi$ is known, we can numerically solve $z = 5.673$. Hence the corresponding expected acceptance rate $a = 0.234$, independent with the target distribution, and the efficiency is $\Theta(N^{2\beta})$. When $R = \omega(N^{2\beta})$, with probability $1 - O(\exp(-2R / N^\beta))$, the acceptance rate decrease exponentially fast, rendering $o(1)$ jump distance. For the remaining probability $O(\exp(-2R / N^\beta))$, assuming all proposals are accepted, the efficiency is still bounded by:
\begin{equation}
    R \exp(-2R / N^\beta) = o(1)
\end{equation}
Hence, optimal efficiency is achieved when $R = O(N^{2\beta})$.
\end{proof}

\newpage
\section{Discussion}

\subsection{Expected Jump Distance as the Metric to Tune the Scale}
\label{sec:why_ejd}
\textcolor{\cdiff}{
In this section, we want to convince the reader that the expected jump distance (EJD) is the correct metric to evaluate the efficiency for samplers in discrete space. To simplify the derivation, we consider the distribution
\begin{equation}
    \pi^{(N)}(x) = \prod_{i=1}^N \pi_i(x_i) = \prod_{i=1}^N p^{x_i}(1-p)^{1-x_i}
\end{equation}
We can notice that, compared to the target distributions considered in the main text \eqref{eq:bernoulli}, we assume the target distribution is identical in each dimension.
}

\textcolor{\cdiff}{
Let the LBP chain, with $R = lN^\frac{2}{3}$, being denoted as $\{x(1), x(2), ... \}$. Since all dimensions are identical, we only need to focus on the first dimension. Denote $w_1 = g(\frac{\pi_1(x_1=0)}{\pi_1(x_1=1)})$ and $w_0 = g(\frac{\pi_1(x_1=1)}{\pi_1(x_1=0)})$. From Lemma. \ref{lemma:concentration_w(x,u)}, we can see that:
\begin{align}
    \lim_{N\rightarrow \infty} \frac{\mathbb{P}(u, \exists u_j = 1| x_1 = 0)}{\mathbb{P}(u, \exists u_j = 1| x_1 = 1)} = \frac{w_0}{w_1}
\end{align}
That's to say, the probability ratio of $x_1=0$ and $x_1=1$ being flipped equals to their weight ratio. Then we compare the acceptance rate in M-H test. From the proof of the main theorem \ref{thm:lb}, we know the acceptance rate is determined by the term $A$ defined in \eqref{eq:A}
\begin{align}
    A = & \frac{1}{W}\sum_{r=1}^R (R-r+1)w_{u_i}(x_{u_i}) - r w_{u_i}(y_{u_i}) 
\end{align}
We can see that, when the first dimension is flipped in proposal, the difference of $A$ is $O(N^{-\frac{1}{3}})$ for $x_1=0$ and $x_1=1$. As a result, we have:
\begin{equation}
    \lim_{N\rightarrow \infty} \frac{\mathbb{P}(\text{accept }|u, \exists u_j=1,  x_1 = 0)}{\mathbb{P}(\text{accept }|u, \exists u_j=1, x_1 = 1)} = 1
\end{equation}
Now, we consider the one-dimensional process $Z^N_t = x_1(\lfloor tN^\frac{1}{3}\rfloor)$.
The identical assumption implies that, the frequency for a site, for example the first dimension, being selected is $l N^{-\frac{1}{3}}$.
We can easily see that when $N$ is large enough, $Z^N_t$ converges to a jump process $Z_t$, whose generator we denote.
\begin{equation}
    Q = \left[
    \begin{array}{rr}
    -Q_{01} & Q_{01} \\
    Q_{10}   & -Q_{10}
    \end{array}
    \right]
\end{equation}
From the derivation above, we know that
\begin{equation}
    \frac{Q_{01}}{Q_{10}} = \lim_{N\rightarrow \infty} \frac{\sum_u \mathbb{E}_{x_{2:N}}[\mathbb{P}(u, \exists u_j=1| x_1 = 0) \mathbb{P}(\text{accept }|u, \exists u_j=1,  x_1 = 0)]}{\sum_u \mathbb{E}_{x_{2:N}}[\mathbb{P}(u, \exists u_j=1| x_1 = 1) \mathbb{P}(\text{accept }|u, \exists u_j=1,  x_1 = 1)]} = \frac{w_0}{w_1}
\end{equation}
Since the sketch of proof above shows that the ratio is independent with the parameter $l$, we have the following important decomposition
\begin{equation}
    Q = \lambda(l) Q(p)
\end{equation}
where $Q(p)$ is a matrix only depends on $p$ and the locally balanced function $g$ selected, and $\lambda(l)$ is a scalar only depends on the parameter $l$.
}

\textcolor{\cdiff}{
Since $Q(p)$ only depends on the target distribution, for any test functions $f(\cdot)$, the inverse auto-correlation of the jump process is proportional to $\lambda(l)$. When we tune $l$, the coefficient $\lambda(l) = l \cdot 2 \Phi(- \lambda_1 l^\frac{3}{2})$ is the multiplication of the proposal frequency and the acceptance rate. The value $\lambda_1$ is defined in \eqref{eq:lambda_lbp}. As a jump process, we don't have to analytically compute the value of $\lambda(l)$, as $\lambda(l)$ is proportional to the expected jump distance (EJD). So, we can tune $l$ by maximizing the EJD, without having to know the formulation of the target distribution.
}

\textcolor{\cdiff}{Remark 1: The jump process in discrete space is different from the diffusion process in continuous space. For diffusion process, its velocity is characterized by the ESJD. But for jump process, its velocity is characterized by the EJD. That's why Langevin algorithms tunes the step size via ESJD \citep{roberts1998optimal}, but our LBP tunes the path length via EJD.}

\textcolor{\cdiff}{Remark 2: To simplify the derivation, we assume that the target distributions have identical marginals. For target distributions with non-identical marginal distribution, different dimensions $i = 1, ..., N$ can have different velocity $\lambda_i(l)$. But the sampling process will still converge to jump process, and we shall still use EJD to measure the efficiency.
}

\subsection{The Choice of $\epsilon$ and the Optimal Acceptance Rate}
\label{sec:convergence_epsilon}
\textcolor{\cdiff}{
The convergence of \eqref{eq:asym_acc_lbp} does not depend on the value of $\epsilon$ in \eqref{eq:smooth_target}. 
Based on the proof above, we can know \eqref{eq:asym_acc_lbp} converges at the rate $O(N^{-\frac{1}{12}})$. But the convergence of the optimal acceptance rate depends on the $\epsilon$. We can first consider two extreme cases for intuition. When all $p_i$ are close to $\frac{1}{2}$, $\lambda_1$ in \eqref{eq:lambda_lbp} will be close to $0$ and the optimal acceptance rate will be close to $1$; when all $p_i$ are close to $0$ or $1$, $\lambda_1$ in \eqref{eq:lambda_lbp} will be close to $\infty$ and the optimal acceptance rate will be close to $0$. Hence, the main purpose to use fixed $\epsilon$ is to give upper and lower bounds for $\lambda_1$ in \eqref{eq:lambda_lbp}, such that the optimal acceptance rate can converge to $0.574$ as in Corollary \ref{cor:order}.
}

\textcolor{\cdiff}{
Next, we discuss how does the model dimension $N$ in \eqref{eq:bernoulli} needed in terms of $\epsilon$ to make sure the optimal convergence to $0.574$. When all $p_i$ have the extreme value determined by $\epsilon$, using locally balanced function $g(t) = \sqrt{t}$, we can consider the following two situations:
\begin{itemize}
\item 
All $|p_i - 0.5| = \epsilon \rightarrow 0$. Then we have:
\begin{equation}
\lambda_1^2 = \frac{\sum_{i=1}^N \sqrt{\epsilon(1 - \epsilon)} (\sqrt{\frac{\epsilon}{1 - \epsilon}} - \sqrt{\frac{1 - \epsilon}{ \epsilon}})^2}{4 \epsilon (1 - \epsilon) \sum_{i=1}^N \sqrt{\epsilon(1 - \epsilon)}} \approx \frac{\sum_{i=1}^N 0.5 \cdot 4 \epsilon^2}{4 \cdot \sum_{i=1}^N 0.5} = \epsilon^2    
\end{equation}
When the expected acceptance is $0.574$, we need $\lambda_1 l^\frac{3}{2} = O(\epsilon l^\frac{3}{2})$ equals to a constant, which means $l$ has the same order as $\epsilon^{-\frac{2}{3}}$. Since we have $R = l N^\frac{2}{3} \le N$, we need $\epsilon^{-\frac{2}{3}} = O(N^\frac{1}{3})$. As a result, we requires $\epsilon^{-1} = O(N^{\frac{1}{2}})$, which basically means we need $N \ge \epsilon^{-2}$ to have the optimal acceptance rate converges to $0.574$.
\item
All $0.5 - |p_i - 0.5| = \epsilon \rightarrow 0$. We have:
\begin{equation}
     \lambda_1^2 = \frac{\sum_{j=1}^N \epsilon \sqrt{\frac{1 - \epsilon}{\epsilon}} (\sqrt{\frac{1 - \epsilon}{\epsilon}} - \sqrt{\frac{\epsilon}{1 - \epsilon}})^2 }{4 (\sqrt{\epsilon(1-\epsilon)})^2 \sum_{i=1}^N \sqrt{\epsilon (1 - \epsilon)} } \approx \frac{\sum_{j=1}^N \epsilon^\frac{1}{2} \epsilon^{-1} }{4 \epsilon \sum_{j=1}^N \epsilon^{-\frac{1}{2}}} = \frac{1}{4}\epsilon^{-2}
\end{equation}
When the expected acceptance is $0.574$, we need $\lambda_1 l^\frac{3}{2} = O(\epsilon^{-1} l^\frac{3}{2})$ equals to a constant, which means $l$ has the same order as $\epsilon^{\frac{2}{3}}$. Since we have $R = l N^\frac{2}{3} \ge 1$, we have $l^{-1} = O(N^\frac{2}{3})$.  As a result, we requires $\epsilon^{-1} = O(N^{-1})$, which basically means we need $N \ge \epsilon^{-1}$ to have the optimal acceptance rate converges to $0.574$.
\end{itemize}
So, both situations show that we need $N$ increase with $\epsilon$ to make sure the optimal acceptance rate converges. In the main text, we assume $\epsilon$ is a constant and it guarantees Corollary \ref{cor:order} holds.
}

\textcolor{\cdiff}{
We conduct extra numerical simulations to verify our results. To simplify the experiments, we assume all dimensions have the same configurations: $p_i = p$. We report the size of $N$ needed to guarantee that the optimal acceptance rate is $0.574$ in Table \ref{tab:situ1} and Table \ref{tab:situ2}.
}

\begin{table}
\begin{minipage}{.5\linewidth}
    \begin{center}
        \begin{tabular}{c|cccc}
\toprule
$\epsilon$ & $0.1$ & $0.05$ & $0.025$ & $0.0125$ \\
\midrule
$N$     & 10 & 40 & 160 & 640 \\
\bottomrule
\end{tabular}
        \captionof{table}{When $p = 0.5 - \epsilon$ \label{tab:situ1}}
    \end{center}
\end{minipage}
\begin{minipage}{.5\linewidth}
    \begin{center}
        \begin{tabular}{c|cccc}
\toprule
$\epsilon$ & $0.01$ & $0.005$ & $0.0025$ & $0.00125$ \\
\midrule
$N$     & 50 & 100 & 200 & 400 \\
\bottomrule
\end{tabular}
        \captionof{table}{When $p = \epsilon$ \label{tab:situ2}}
    \end{center}
\end{minipage} 
\end{table}


\subsection{Optimal Scale of RWM}
\label{sec:rwm_Peps}
\textcolor{\cdiff}{
When we assume the target distribution belongs to \eqref{eq:smooth_target}, the derivation of the optimal acceptance rate $0.234$ is no longer valid. But we can still show the optimal scale is $R = O(1)$ by proving the acceptance rate decreasing exponentially fast.}

\textcolor{\cdiff}{
In particular, assume we use $R = l N^\beta$ in RWM. Then the acceptance rate can be written as:
\begin{equation}
    A = \min\{1, A' = \frac{\pi(y)}{\pi(x)} = \frac{\prod_{j=1}^R \pi_{u_j}(y)}{\prod_{j=1}^R \pi_{u_j}(x)}\}
\end{equation}
Consider the martingale $M_j, j=0, 1, ..., R$, such that $M_0 = 0$ and 
\begin{equation}
    M_j - M_{j-1} =  \log \frac{\pi_{u_j}(y)}{\pi_{u_j}(x)} - \mathbb{E}[\log \frac{\pi_{u_j}(y)}{\pi_{u_j}(x)}|u_{1:j-1}] = (1 - 2x_{u_j}) \log \frac{p_{u_j}}{1 - p_{u_j}}
\end{equation}
By assumption in \eqref{eq:smooth_target}, we know that 
\begin{align}
    \mathbb{E}[\log \frac{\pi_{u_j}(y)}{\pi_{u_j}(x)}|u_{1:j-1}] 
    & = \mathbb{E} [(1 - 2x_{u_j}) \log \frac{p_{u_j}}{1 - p_{u_j}}|u_{1:j-1}] = (1 - 2p_{u_j}) \log \frac{p_{u_j}}{1 - p_{u_j}} \\
    & \le 2\epsilon \log \frac{1 - 2 \epsilon}{1 + 2 \epsilon} < 0
\end{align}
And we have 
\begin{align}
    |M_j - M_{j-1}| \le 2 \left|(1 - 2 \epsilon) \log \frac{\epsilon}{1 - \epsilon}\right| := K
\end{align}
By Azuma-Hoeffding lemma, we have
\begin{align}
    \mathbb{P}(|\sum_{j=1}^R \log \frac{\pi_{u_j}(y)}{\pi_{u_j}(x)} - \mathbb{E} [\log \frac{\pi_{u_j}(y)}{\pi_{u_j}(x)}] | \ge R^\frac{3}{4} t) \le 2 e^{\frac{-R^frac{1}{2} t^2}{K^2}}
\end{align}
For $\beta > 0$, $R$ increases to infinity when $N \rightarrow \infty$. In this case, $\log A'$ concentrates to a value $T \le R\cdot 2 \epsilon \log \frac{1 - 2\epsilon}{1 + 2\epsilon}$ and $A'$ decreases exponentially fast. Hence, the optimal scaling of RWM is $O(1)$.
}

\newpage
\section{Adaptive Algorithm}
\label{appendix:algorithm}
We give the algorithm box for ALBP:
\begin{algorithm}
\caption{Adaptive Locally Balanced Proposal}
\label{alg:albp}
\begin{algorithmic}[1]
  \STATE Initialize current state $x^{(1)}$.
  \STATE Initialize scale $R_1 = 1$.
  \FOR{t=1, ..., T}
    \STATE Initialize candidate set $\mathcal{C} = \{1, .., N\}$.
    \STATE $R \gets$ probabilistic rounding of $R_t$
    \FOR{r=1, ..., R}
      \STATE Sample $u_r$ with $\mathbb{P}(u_r=j) \propto w_j(x^{(t)}) 1_{\{j \in \mathcal{C}\}}$.
      \STATE $\mathcal{C} \gets \mathcal{C} \backslash \{u_r\}$. \label{eq:pop}
    \ENDFOR
    \STATE Obtain $y$ by flipping indices $u_1, ..., u_R$ of $x^{(t)}$.
    \STATE Compute acceptance rate $A = A(x^{(t)}, y, u)$.
    \IF{rand(0,1) $< A$}
      \STATE $x^{(t+1)} = y$
    \ELSE
      \STATE $x^{(t+1)} = x^{(t)}$
    \ENDIF
    \IF{$t < T_\text{warmup}$}
      \STATE $R_{t+1} \gets R_t + (A - 0.574)$.
    \ENDIF
  \ENDFOR
\end{algorithmic}
\end{algorithm}

We give the algorithm box for ARWM:
\begin{algorithm}
\caption{Adaptive Random Walk Metropolis}
\label{alg:arwm}
\begin{algorithmic}[1]
  \STATE Initialize current state $x^{(1)}$.
  \STATE Initialize scale $R_1 = 1$.
  \FOR{t=1, ..., T}
    \STATE Initialize candidate set $\mathcal{C} = \{1, .., N\}$.
    \STATE $R \gets$ probabilistic rounding of $R_t$
    \STATE Uniformly sample $u_1, ..., u_R$.
    \STATE Obtain $y$ by flipping indices $u_1, ..., u_R$ of $x^{(t)}$.
    \STATE Compute acceptance rate $A = A(x^{(t)}, y, u)$.
    \IF{rand(0,1) $< A$}
      \STATE $x^{(t+1)} = y$
    \ELSE
      \STATE $x^{(t+1)} = x^{(t)}$
    \ENDIF
    \IF{$t < T_\text{warmup}$}
      \STATE $R_{t+1} \gets R_t + (A - 0.234)$.
    \ENDIF
  \ENDFOR
\end{algorithmic}
\end{algorithm}

\newpage
\section{Experiment Details}
\label{appendix:exp}
We consider five samplers:
\begin{itemize}
    \item RWM: random walk Metropolis
    \item GWG($\sqrt{t}$): LBP with replacement, same as algorithm \ref{alg:lbp} except for skipping line 5, weight function $g(t) = \sqrt{t}$
    \item LBP($\sqrt{t}$): LBP given in algorithm \ref{alg:lbp}, weight function $g(t) = \sqrt{t}$
    \item GWG($\frac{t}{t+1}$): LBP with replacement, same as algorithm \ref{alg:lbp} except for skipping line 5, weight function $g(t) = \frac{t}{t+1}$
    \item LBP($\frac{t}{t+1}$): LBP given in algorithm \ref{alg:lbp}, weight function $g(t) = \frac{t}{t+1}$
\end{itemize}
For each sampler, we first start simulating with $R = 1$ to get an initial acceptance rate $a_{\max}$. Then we adopt $a_{\max} - 0.02, a_{\max} - 0.04, ..., a_{\max} - 0.02k$ as the target acceptance rate, until $a_{\max} - 0.02k < 0.03$. For each target acceptance rate $a_\text{target}$, we use our adaptive sampler to get an estimated scaling $R_\text{target}$. Then we simulate 100 chains with scaling $R_\text{target}$ to get the expected acceptance rate, expected jump distance, effective sample size $(a, d, e)$.

To measure the performance of the adaptive sampler, we compare three versions for each sampler above. In particular, for sampler X we have
\begin{itemize}
    \item X-1, represents fixed scaling $R=1$ version of the sampler.
    \item AX, represents the adaptive version of the sampler, whose target acceptance rate is selected as $0.234$ for RWM, and $0.574$ for else.
    \item GX, represents the grid search version of the sampler, where we always use the best results among all simulations for different target acceptance rates we mentioned above.
\end{itemize}

\subsection{Simulation on Bernoulli Model}
\label{sec:bernoulli}
The density function for Bernoulli distribution is:
\begin{equation}
    \pi^{(N)}(x) = \prod_{i=1}^N \pi_i(x_i) = \prod_{i=1}^N p_i^{x_i}(1-p_i)^{1-x_i}
\end{equation}
We consider three configurations
\begin{itemize}
    \item C1: $p_i$ is independently, uniformly sampled from $[0.25, 0.75]$.
    \item C2: $p_i$ is independently, uniformly sampled from $[0.15, 0.85]$.
    \item C3: $p_i$ is independently, uniformly sampled from $[0.05, 0.95]$.
\end{itemize}
For each configuration, we simulate on three sizes:
\begin{itemize}
    \item $N=100$, sample Markov chain $x_{1:10000}$, use $x_{1:5000}$ for burn in, use $x_{5001:10000}$ to estimate expected acceptance rate, expected jump distance, effective sample size.
    \item $N=800$, sample Markov chain $x_{1:40000}$, use $x_{1:20000}$ for burn in, use $x_{20001:40000}$ to estimate expected acceptance rate, expected jump distance, effective sample size.
    \item $N=6400$, sample Markov chain $x_{1:100000}$, use $x_{1:50000}$ for burn in, use $x_{50001:100000}$ to estimate expected acceptance rate, expected jump distance, effective sample size.
\end{itemize}

We give the scatter plot of $(a, d)$ and $(a, r)$ in figure~\ref{fig:bernoulli_all}. And we examine the performance of our adaptive algorithm in table~\ref{tab:bernoulli_d_all} and table~\ref{tab:bernoulli_e_all}.
\begin{figure}
    \centering
    \includegraphics[width=\textwidth]{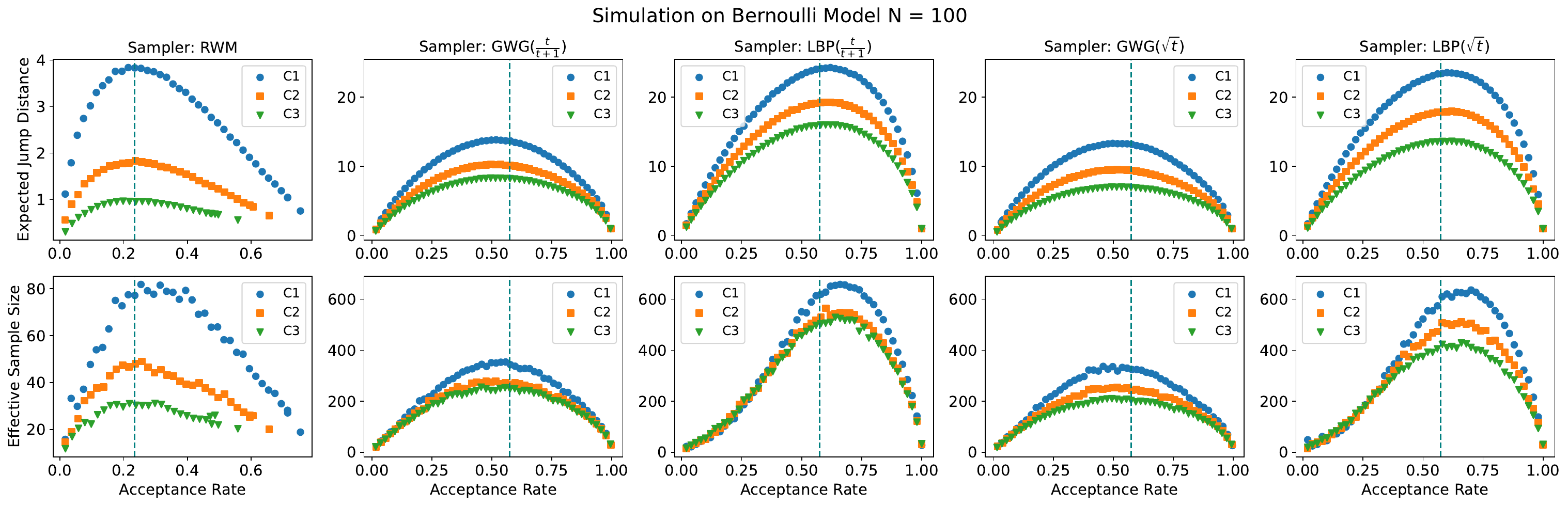}
    \includegraphics[width=\textwidth]{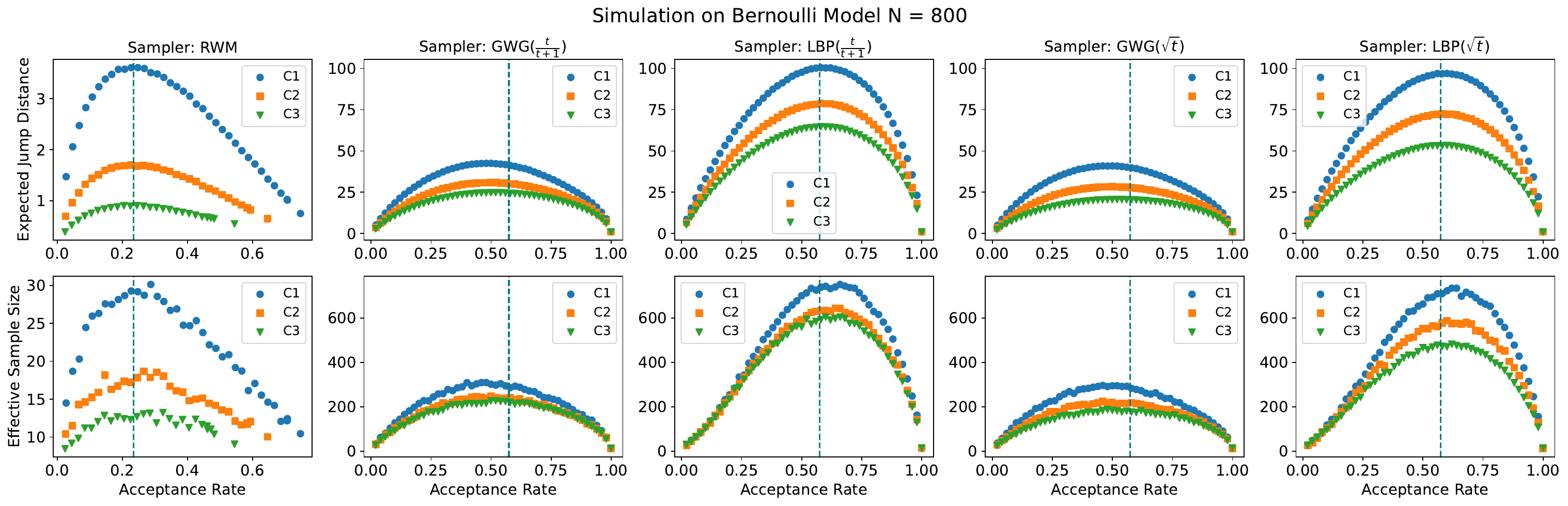}
    \includegraphics[width=\textwidth]{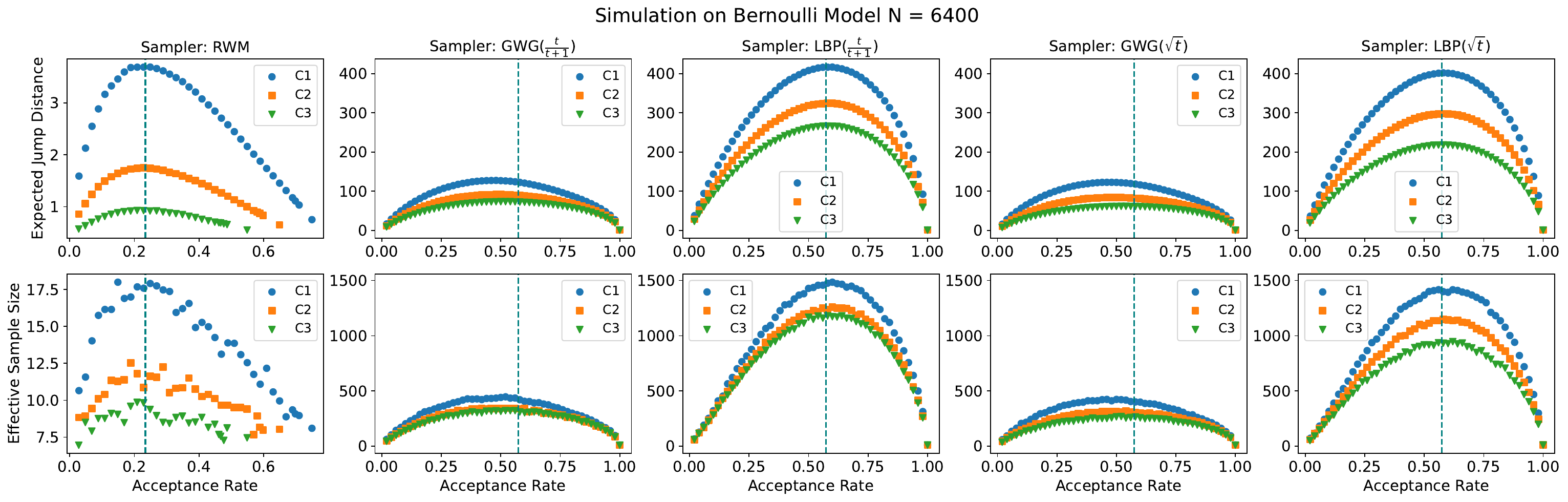}
    \caption{Simulation Results on Bernoulli Model}
    \label{fig:bernoulli_all}
\end{figure}

\begin{table}[htb]
\footnotesize
\centering
\begin{tabular}{crrrrrrrrr}
\toprule
Size  & \multicolumn{3}{c}{$N=100$} & \multicolumn{3}{c}{$N=800$} & \multicolumn{3}{c}{$N=6400$} \\
\cmidrule(lr){1-1} \cmidrule(lr){2-4} \cmidrule(lr){5-7}  \cmidrule(lr){8-10}
Sampler & C1 & C2 & C3 & C1 & C2 & C3 & C1 & C2 & C3 \\
\midrule
RWM-1 & 0.75 & 0.66 & 0.56 & 0.75 & 0.65 & 0.55 & 0.75 & 0.65 & 0.55 \\
ARWM & 3.85 & 1.81 & 0.96 & 3.63 & 1.70 & 0.89 & 3.69 & 1.74 & 0.92 \\
GRWM & 3.84 & 1.83 & 0.96 & 3.61 & 1.70 & 0.90 & 3.69 & 1.75 & 0.93 \\
GWG($\frac{t}{t+1}$)-1 & 1.00 & 1.00 & 1.00 & 1.00 & 1.00 & 1.00 & 1.00 & 1.00 & 1.00 \\
AGWG($\frac{t}{t+1}$) & 13.60 & 10.14 & 8.26 & 41.45 & 30.23 & 24.39 & 123.08 & 89.48 & 72.56 \\
GGWG($\frac{t}{t+1}$) & 13.84 & 10.30 & 8.31 & 42.50 & 30.88 & 24.74 & 127.55 & 91.68 & 73.33 \\
LBP($\frac{t}{t+1}$)-1 & 1.00 & 1.00 & 1.00 & 1.00 & 1.00 & 1.00 & 1.00 & 1.00 & 1.00 \\
ALBP($\frac{t}{t+1}$) & 24.05 & 19.16 & 15.96 & 100.25 & 78.63 & 64.47 & 416.55 & 324.59 & 266.73 \\
GLBP($\frac{t}{t+1}$) & 24.26 & 19.26 & 15.98 & 100.49 & 78.83 & 64.69 & 416.67 & 324.67 & 266.21 \\
GWG($\sqrt{t}$)-1 & 1.00 & 1.00 & 0.99 & 1.00 & 1.00 & 1.00 & 1.00 & 1.00 & 1.00 \\
AGWG($\sqrt{t}$) & 13.14 & 9.41 & 6.92 & 39.88 & 27.81 & 20.46 & 118.44 & 82.52 & 61.27 \\
GGWG($\sqrt{t}$) & 13.31 & 9.52 & 7.04 & 40.92 & 28.34 & 20.60 & 122.74 & 84.08 & 61.52 \\
LBP($\sqrt{t}$)-1 & 1.00 & 1.00 & 1.00 & 1.00 & 1.00 & 1.00 & 1.00 & 1.00 & 1.00 \\
ALBP($\sqrt{t}$) & 23.40 & 17.88 & 13.59 & 96.96 & 72.39 & 53.28 & 401.94 & 297.95 & 218.09 \\
GLBP($\sqrt{t}$) & 23.53 & 17.95 & 13.61 & 96.93 & 72.41 & 53.36 & 401.89 & 297.72 & 218.11 \\
\bottomrule
\end{tabular}
\caption{Expected Jump Distance on Bernoulli Model}
\label{tab:bernoulli_d_all}
\end{table}

\begin{table}[htb]
\footnotesize
\centering
\begin{tabular}{crrrrrrrrr}
\toprule
Size  & \multicolumn{3}{c}{$N=100$} & \multicolumn{3}{c}{$N=800$} & \multicolumn{3}{c}{$N=6400$} \\
\cmidrule(lr){1-1} \cmidrule(lr){2-4} \cmidrule(lr){5-7}  \cmidrule(lr){8-10}
Sampler & C1 & C2 & C3 & C1 & C2 & C3 & C1 & C2 & C3 \\
\midrule
RWM-1 & 18.85 & 20.09 & 20.27 & 10.44 & 10.02 & 9.06 & 8.11 & 8.04 & 7.46 \\
ARWM & 80.86 & 47.95 & 30.49 & 28.54 & 18.44 & 12.97 & 17.97 & 11.25 & 8.66 \\
GRWM & 81.82 & 49.00 & 31.11 & 30.13 & 18.67 & 13.22 & 17.99 & 12.53 & 9.86 \\
GWG($\frac{t}{t+1}$)-1 & 28.11 & 27.89 & 31.56 & 10.80 & 12.75 & 13.61 & 7.93 & 9.17 & 9.00 \\
AGWG($\frac{t}{t+1}$) & 343.74 & 270.00 & 253.53 & 302.65 & 234.94 & 215.58 & 423.86 & 334.06 & 307.75 \\
GGWG($\frac{t}{t+1}$) & 353.97 & 278.47 & 255.11 & 309.04 & 247.38 & 227.08 & 446.20 & 343.49 & 320.93 \\
LBP($\frac{t}{t+1}$)-1 & 27.37 & 30.62 & 33.31 & 12.11 & 13.39 & 14.19 & 8.81 & 9.06 & 9.57 \\
ALBP($\frac{t}{t+1}$) & 604.07 & 528.66 & 511.27 & 754.69 & 622.35 & 594.59 & 1472.86 & 1247.31 & 1185.65 \\
GLBP($\frac{t}{t+1}$) & 658.24 & 564.55 & 529.03 & 751.22 & 644.43 & 604.47 & 1484.93 & 1259.07 & 1179.93 \\
GWG($\sqrt{t}$)-1 & 26.19 & 30.17 & 30.92 & 12.35 & 12.41 & 14.29 & 8.97 & 8.97 & 9.00 \\
AGWG($\sqrt{t}$) & 335.66 & 254.36 & 205.81 & 284.31 & 206.75 & 175.17 & 406.81 & 303.40 & 261.13 \\
GGWG($\sqrt{t}$) & 336.70 & 254.30 & 209.78 & 296.11 & 224.29 & 187.23 & 422.89 & 318.11 & 267.72 \\
LBP($\sqrt{t}$)-1 & 28.36 & 27.88 & 30.93 & 12.11 & 12.50 & 14.12 & 8.48 & 10.07 & 9.36 \\
ALBP($\sqrt{t}$) & 598.66 & 488.24 & 412.74 & 702.91 & 570.50 & 482.58 & 1411.88 & 1135.84 & 935.89 \\
GLBP($\sqrt{t}$) & 636.34 & 510.63 & 428.40 & 734.46 & 588.45 & 482.09 & 1417.15 & 1147.37 & 946.22 \\
\bottomrule
\end{tabular}
\caption{Effective Sample Size on Bernoulli Model}
\label{tab:bernoulli_e_all}
\end{table}

\begin{table}[htb]
\footnotesize
\centering
\begin{tabular}{crrrrrrrrr}
\toprule
Size  & \multicolumn{3}{c}{$N=100$} & \multicolumn{3}{c}{$N=800$} & \multicolumn{3}{c}{$N=6400$} \\
\cmidrule(lr){1-1} \cmidrule(lr){2-4} \cmidrule(lr){5-7}  \cmidrule(lr){8-10}
Sampler & C1 & C2 & C3 & C1 & C2 & C3 & C1 & C2 & C3 \\
\midrule
RWM-1 & 6.26 & 8.85 & 5.15 & 14.45 & 15.44 & 12.48 & 39.10 & 42.70 & 41.07 \\
ARWM & 7.10 & 6.32 & 5.27 & 15.84 & 14.90 & 14.74 & 42.18 & 43.27 & 44.44 \\
GRWM & 7.33 & 6.41 & 5.90 & 13.66 & 18.01 & 14.56 & 43.64 & 44.58 & 42.24 \\
GWG($\frac{t}{t+1}$)-1 & 7.63 & 9.17 & 7.45 & 13.61 & 13.13 & 12.73 & 60.10 & 66.24 & 54.25 \\
AGWG($\frac{t}{t+1}$) & 7.49 & 9.61 & 7.48 & 14.18 & 17.95 & 15.59 & 57.39 & 72.89 & 69.93 \\
GGWG($\frac{t}{t+1}$) & 9.07 & 9.04 & 8.15 & 16.89 & 20.09 & 10.99 & 70.31 & 68.87 & 63.86 \\
LBP($\frac{t}{t+1}$)-1 & 10.60 & 12.68 & 10.35 & 23.58 & 24.36 & 19.74 & 70.11 & 94.35 & 86.53 \\
ALBP($\frac{t}{t+1}$) & 11.30 & 10.83 & 11.83 & 24.81 & 28.07 & 21.57 & 129.27 & 108.42 & 108.88 \\
GLBP($\frac{t}{t+1}$) & 10.76 & 11.57 & 10.97 & 30.71 & 25.42 & 19.33 & 92.47 & 100.16 & 100.06 \\
GWG($\sqrt{t}$)-1 & 10.80 & 11.22 & 7.19 & 17.70 & 26.87 & 19.60 & 58.95 & 59.23 & 61.19 \\
AGWG($\sqrt{t}$) & 9.57 & 13.00 & 7.23 & 18.22 & 19.46 & 20.27 & 78.60 & 67.59 & 65.13 \\
GGWG($\sqrt{t}$) & 9.44 & 7.36 & 8.64 & 18.83 & 20.72 & 20.11 & 65.64 & 50.45 & 65.91 \\
LBP($\sqrt{t}$)-1 & 13.18 & 13.82 & 11.71 & 25.62 & 28.37 & 28.94 & 86.08 & 81.86 & 90.00 \\
ALBP($\sqrt{t}$) & 12.51 & 12.21 & 10.59 & 23.16 & 22.51 & 21.02 & 102.78 & 107.33 & 103.70 \\
GLBP($\sqrt{t}$) & 14.38 & 13.48 & 10.34 & 23.97 & 30.38 & 21.48 & 120.14 & 100.18 & 117.96 \\
\bottomrule
\end{tabular}
\caption{Running Time on Bernoulli Model}
\label{tab:bernoulli_t_all}
\end{table}

\subsection{Simulation on Ising Model}
\label{sec:ising}
Ising model is a classic model in physics defined on a $p\times p$ square lattice graph $(V_p, E_p)$. That's to say, the nodes are indexed by $\{1, ..., p\}^2$ and an edge $((i, j), (k, l))$ exists if and only if one of the following condition holds:
\begin{multicols}{2}
\begin{itemize}
    \item $i = k, \ j = l + 1$
    \item $i = k, \ j = l - 1$
    \item $i = k + 1, \ j = l$
    \item $i = k - 1, \ j = l$
\end{itemize}
\end{multicols}
The state space is $\mathcal{X} = \{-1, 1\}^{V_p}$ and the target distribution is defined as:
\begin{equation}
    \pi(x) \propto \exp\Big(\sum_{i \in V_p} \alpha_i x_i - \lambda \sum_{(i, j) \in E_p} x_i x_j\Big)
\end{equation}
Following \cite{zanella2020informed}, we consider three configurations
\begin{itemize}
    \item C1: $\alpha_v$ is independently and uniformly sampled from $(-0.2, 0.4)$ if $(v_1 - \frac{p}{2})^2 + (v_2 - \frac{p}{2})^2 \le \frac{p^2}{2}$, else $\alpha_v$ is independently and uniformly sampled from $(-0.4, 0.2)$; $\lambda = 0.1$.
    \item C2: $\alpha_v$ is independently and uniformly sampled from $(-0.3, 0.6)$ if $(v_1 - \frac{p}{2})^2 + (v_2 - \frac{p}{2})^2 \le \frac{p^2}{2}$, else $\alpha_v$ is independently and uniformly sampled from $(-0.6, 0.3)$; $\lambda = 0.15$.
    \item C3: $\alpha_v$ is independently and uniformly sampled from $(-0.4, 0.8)$ if $(v_1 - \frac{p}{2})^2 + (v_2 - \frac{p}{2})^2 \le \frac{p^2}{2}$, else $\alpha_v$ is independently and uniformly sampled from $(-0.8, 0.4)$; $\lambda = 0.2$.
\end{itemize}
For each configuration, we simulate on three sizes:
\begin{itemize}
    \item $p=20$, sample Markov chain $x_{1:10000}$, use $x_{1:5000}$ for burn in, use $x_{5001:10000}$ to estimate expected acceptance rate, expected jump distance, effective sample size.
    \item $p=50$, sample Markov chain $x_{1:40000}$, use $x_{1:20000}$ for burn in, use $x_{20001:40000}$ to estimate expected acceptance rate, expected jump distance, effective sample size.
    \item $p=100$, sample Markov chain $x_{1:100000}$, use $x_{1:50000}$ for burn in, use $x_{50001:100000}$ to estimate expected acceptance rate, expected jump distance, effective sample size.
\end{itemize}

We give the scatter plot of $(a, d)$ and $(a, r)$ in figure~\ref{fig:ising_all}. And we examine the performance of our adaptive algorithm in table~\ref{tab:ising_d_all} and table~\ref{tab:ising_e_all}.
\begin{figure}
    \centering
    \includegraphics[width=\textwidth]{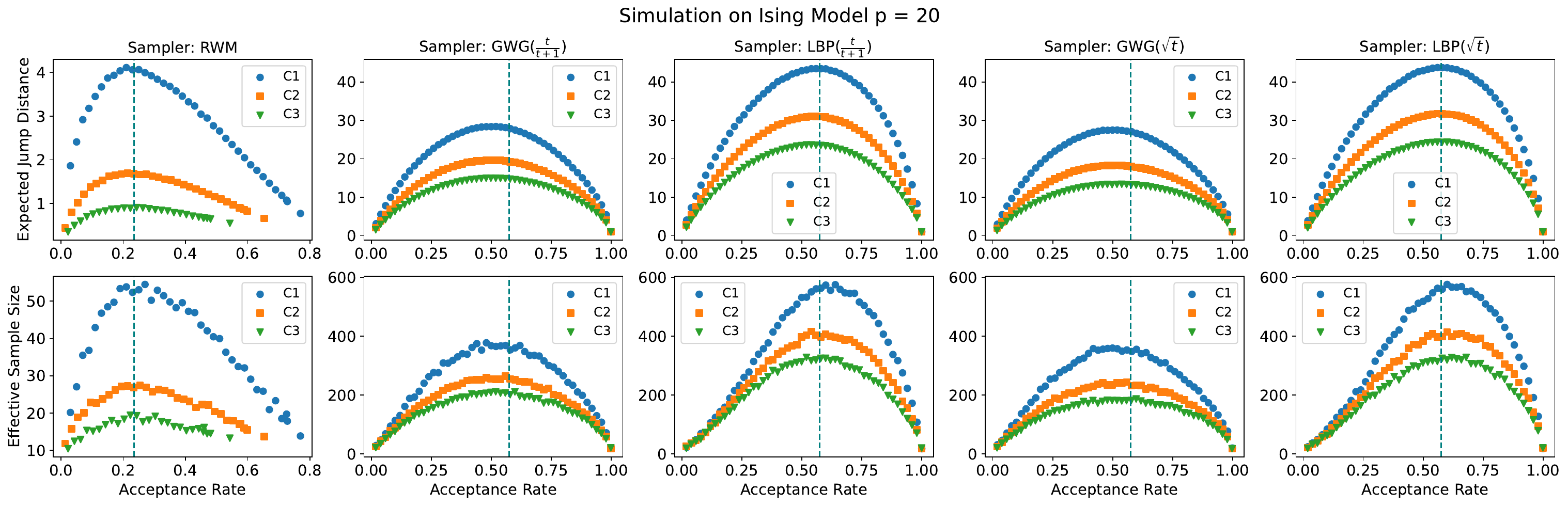}
    \includegraphics[width=\textwidth]{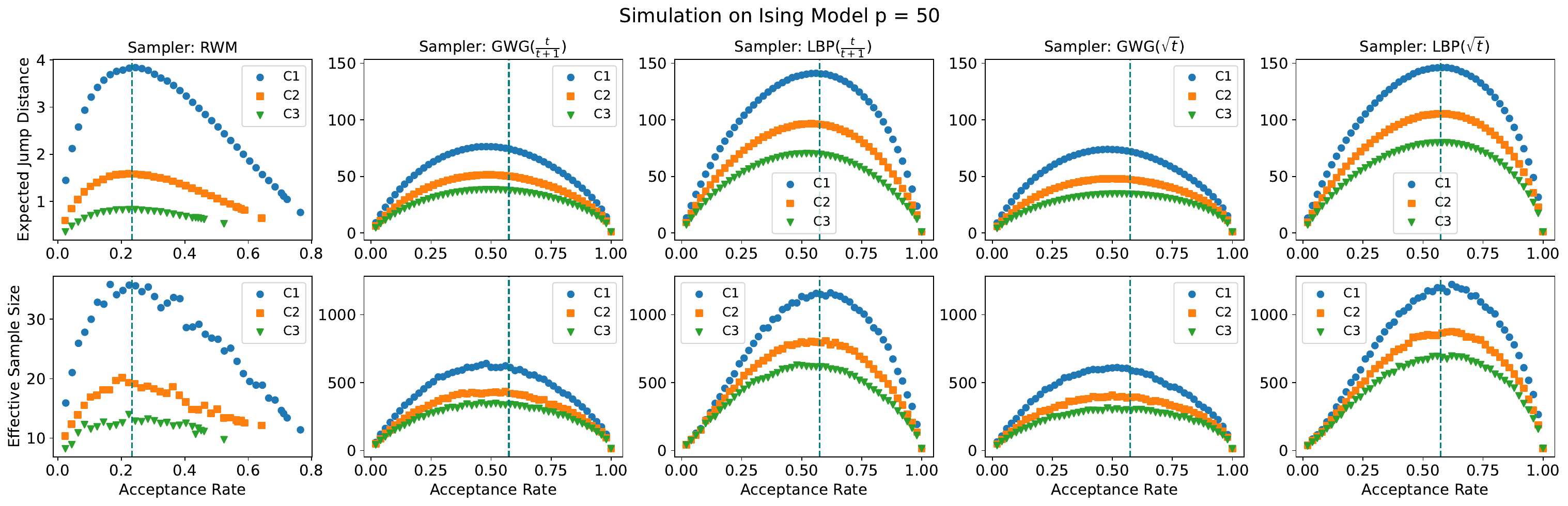}
    \includegraphics[width=\textwidth]{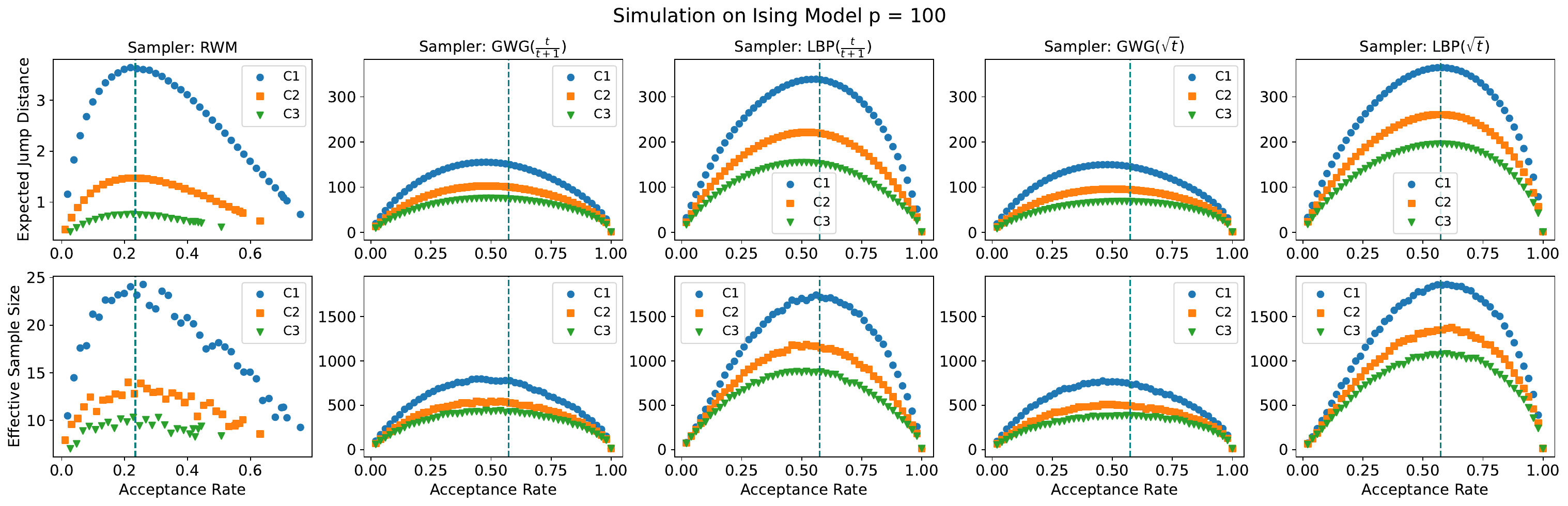}
    \caption{Simulation Results on Ising Model}
    \label{fig:ising_all}
\end{figure}

\begin{table}[htb]
\footnotesize
\centering
\begin{tabular}{crrrrrrrrr}
\toprule
Size  & \multicolumn{3}{c}{$p=20$} & \multicolumn{3}{c}{$p=50$} & \multicolumn{3}{c}{$p=100$} \\
\cmidrule(lr){1-1} \cmidrule(lr){2-4} \cmidrule(lr){5-7}  \cmidrule(lr){8-10}
Sampler & C1 & C2 & C3 & C1 & C2 & C3 & C1 & C2 & C3 \\
\midrule
RWM-1 & 0.77 & 0.65 & 0.54 & 0.77 & 0.64 & 0.52 & 0.76 & 0.63 & 0.51 \\
ARWM & 4.02 & 1.70 & 0.89 & 3.83 & 1.58 & 0.82 & 3.64 & 1.47 & 0.76 \\
GRWM & 4.12 & 1.69 & 0.90 & 3.84 & 1.59 & 0.83 & 3.64 & 1.47 & 0.76 \\
GWG($\frac{t}{t+1}$)-1 & 1.00 & 1.00 & 1.00 & 1.00 & 1.00 & 1.00 & 1.00 & 1.00 & 1.00 \\
AGWG($\frac{t}{t+1}$) & 27.93 & 19.35 & 14.73 & 74.33 & 50.27 & 37.68 & 150.31 & 100.21 & 74.15 \\
GGWG($\frac{t}{t+1}$) & 28.39 & 19.64 & 14.96 & 76.33 & 51.44 & 38.31 & 155.29 & 102.49 & 75.47 \\
LBP($\frac{t}{t+1}$)-1 & 1.00 & 1.00 & 1.00 & 1.00 & 1.00 & 1.00 & 1.00 & 1.00 & 1.00 \\
ALBP($\frac{t}{t+1}$) & 43.42 & 30.99 & 23.50 & 141.01 & 96.23 & 69.54 & 338.14 & 219.73 & 152.05 \\
GLBP($\frac{t}{t+1}$) & 43.45 & 31.10 & 23.62 & 141.20 & 96.68 & 70.16 & 339.11 & 221.49 & 154.52 \\
GWG($\sqrt{t}$)-1 & 1.00 & 1.00 & 1.00 & 1.00 & 1.00 & 1.00 & 1.00 & 1.00 & 1.00 \\
AGWG($\sqrt{t}$) & 26.96 & 18.09 & 13.29 & 72.04 & 47.25 & 34.30 & 145.43 & 94.41 & 68.09 \\
GGWG($\sqrt{t}$) & 27.49 & 18.32 & 13.39 & 73.94 & 48.03 & 34.54 & 149.84 & 95.79 & 68.30 \\
LBP($\sqrt{t}$)-1 & 1.00 & 1.00 & 1.00 & 1.00 & 1.00 & 1.00 & 1.00 & 1.00 & 1.00 \\
ALBP($\sqrt{t}$) & 43.85 & 31.83 & 24.32 & 146.02 & 105.43 & 79.82 & 364.76 & 261.05 & 195.65 \\
GLBP($\sqrt{t}$) & 43.73 & 31.76 & 24.33 & 146.21 & 105.38 & 79.78 & 364.84 & 260.81 & 195.67 \\
\bottomrule
\end{tabular}
\caption{Expected Jump Distance on Ising Model}
\label{tab:ising_d_all}
\end{table}

\begin{table}[htb]
\footnotesize
\centering
\begin{tabular}{crrrrrrrrr}
\toprule
Size  & \multicolumn{3}{c}{$p=20$} & \multicolumn{3}{c}{$p=50$} & \multicolumn{3}{c}{$p=100$} \\
\cmidrule(lr){1-1} \cmidrule(lr){2-4} \cmidrule(lr){5-7}  \cmidrule(lr){8-10}
Sampler & C1 & C2 & C3 & C1 & C2 & C3 & C1 & C2 & C3 \\
\midrule
RWM-1 & 13.85 & 13.70 & 13.25 & 11.39 & 12.14 & 9.73 & 9.27 & 8.58 & 8.36 \\
ARWM & 51.66 & 27.50 & 17.39 & 35.34 & 19.60 & 12.89 & 22.99 & 13.31 & 9.47 \\
GRWM & 54.48 & 27.36 & 19.41 & 35.85 & 20.16 & 13.96 & 24.28 & 13.99 & 10.32 \\
GWG($\frac{t}{t+1}$)-1 & 19.55 & 18.06 & 20.73 & 13.53 & 13.30 & 14.54 & 10.26 & 11.44 & 11.38 \\
AGWG($\frac{t}{t+1}$) & 362.96 & 250.15 & 205.55 & 611.15 & 429.22 & 340.15 & 755.44 & 533.07 & 419.56 \\
GGWG($\frac{t}{t+1}$) & 377.87 & 264.00 & 211.38 & 641.09 & 434.17 & 349.96 & 795.99 & 543.68 & 441.19 \\
LBP($\frac{t}{t+1}$)-1 & 17.78 & 18.76 & 20.24 & 13.69 & 14.11 & 14.88 & 10.04 & 10.32 & 12.30 \\
ALBP($\frac{t}{t+1}$) & 551.81 & 394.65 & 330.29 & 1135.03 & 821.06 & 620.26 & 1733.51 & 1164.64 & 868.62 \\
GLBP($\frac{t}{t+1}$) & 575.40 & 416.56 & 328.42 & 1161.62 & 809.12 & 629.38 & 1742.19 & 1184.58 & 880.69 \\
GWG($\sqrt{t}$)-1 & 19.95 & 17.66 & 17.87 & 13.22 & 13.57 & 14.22 & 9.54 & 10.17 & 11.50 \\
AGWG($\sqrt{t}$) & 356.57 & 236.55 & 176.42 & 569.23 & 399.21 & 306.83 & 727.25 & 501.14 & 379.41 \\
GGWG($\sqrt{t}$) & 359.85 & 244.00 & 186.74 & 611.60 & 407.16 & 312.81 & 774.36 & 508.04 & 384.21 \\
LBP($\sqrt{t}$)-1 & 18.24 & 19.32 & 20.65 & 14.01 & 14.09 & 16.07 & 10.24 & 11.53 & 11.08 \\
ALBP($\sqrt{t}$) & 562.14 & 413.21 & 329.44 & 1197.76 & 867.77 & 680.61 & 1866.85 & 1359.86 & 1078.16 \\
GLBP($\sqrt{t}$) & 576.18 & 414.54 & 328.02 & 1223.24 & 877.04 & 695.96 & 1861.60 & 1374.12 & 1079.32 \\
\bottomrule
    \end{tabular}
    \caption{Effective Sample Size on Ising Model}
    \label{tab:ising_e_all}
\end{table}

\begin{table}[htb]
\footnotesize
\centering
\begin{tabular}{crrrrrrrrr}
\toprule
Size  & \multicolumn{3}{c}{$p=20$} & \multicolumn{3}{c}{$p=50$} & \multicolumn{3}{c}{$p=100$} \\
\cmidrule(lr){1-1} \cmidrule(lr){2-4} \cmidrule(lr){5-7}  \cmidrule(lr){8-10}
Sampler & C1 & C2 & C3 & C1 & C2 & C3 & C1 & C2 & C3 \\
\midrule
RWM-1 & 18.78 & 19.73 & 19.64 & 71.79 & 74.28 & 75.45 & 173.58 & 142.63 & 143.42 \\
ARWM & 18.84 & 19.94 & 19.37 & 76.05 & 77.45 & 78.17 & 134.24 & 149.44 & 150.26 \\
GRWM & 19.20 & 20.08 & 19.99 & 76.09 & 76.89 & 76.90 & 134.70 & 149.48 & 150.13 \\
GWG($\frac{t}{t+1}$)-1 & 29.54 & 31.07 & 31.78 & 89.62 & 92.75 & 97.31 & 228.38 & 224.22 & 228.92 \\
AGWG($\frac{t}{t+1}$) & 31.07 & 32.54 & 40.28 & 97.31 & 99.73 & 104.38 & 271.45 & 304.41 & 273.28 \\
GGWG($\frac{t}{t+1}$) & 31.10 & 32.38 & 32.26 & 96.96 & 99.61 & 104.46 & 271.39 & 267.24 & 273.17 \\
LBP($\frac{t}{t+1}$)-1 & 36.40 & 37.61 & 46.76 & 108.31 & 111.19 & 116.87 & 260.65 & 291.16 & 269.27 \\
ALBP($\frac{t}{t+1}$) & 37.42 & 38.36 & 38.82 & 126.34 & 124.37 & 116.80 & 308.16 & 320.08 & 317.61 \\
GLBP($\frac{t}{t+1}$) & 37.91 & 39.26 & 39.12 & 124.74 & 129.28 & 125.60 & 309.07 & 310.78 & 316.89 \\
GWG($\sqrt{t}$)-1 & 29.93 & 30.59 & 31.06 & 115.42 & 120.13 & 121.43 & 216.42 & 237.03 & 240.41 \\
AGWG($\sqrt{t}$) & 30.17 & 31.68 & 31.53 & 95.37 & 98.34 & 103.86 & 261.55 & 273.21 & 280.95 \\
GGWG($\sqrt{t}$) & 30.57 & 31.46 & 31.44 & 121.69 & 125.99 & 127.66 & 259.48 & 304.54 & 280.64 \\
LBP($\sqrt{t}$)-1 & 36.99 & 37.42 & 36.98 & 106.82 & 110.86 & 117.53 & 258.04 & 275.03 & 283.44 \\
ALBP($\sqrt{t}$) & 37.62 & 38.84 & 37.58 & 125.48 & 128.87 & 121.34 & 303.28 & 306.60 & 312.44 \\
GLBP($\sqrt{t}$) & 36.87 & 38.96 & 37.87 & 125.95 & 130.17 & 174.29 & 403.14 & 305.27 & 339.78 \\
\bottomrule
\end{tabular}
\caption{Running Time on Ising Model}
\label{tab:ising_t_all}
\end{table}

\subsection{Simulation on FHMM}
\label{sec:fhmm}
FHMM uses latent variables $x \in \mathcal{X} = \{0, 1\}^{L\times K}$ to characterize time series data $y \in \mathbb{R}^L$. Denote $p(x)$ as the prior for hidden variables $x$, and $p(y|x)$ for the likelihood:
\begin{align}
    p(x) &= \prod_{l=1}^L p(x_{l, 1}) \prod_{k=2}^K p(x_{l, k}|x_{l, k-1}) \\
    p(y|x) &= \prod_{l=1}^L \mathcal{N}(y_l; w^T x_l + b, \sigma^2)
\end{align}
Specifically, we have $p(x_{l, 1}) = 0.1$, $p(x_{l, k} = x_{l, k-1} | x_{l, k-1}) = 0.8$ independently $\forall l = 1, ..., L$ and $\forall k = 2, ..., K$. And we have all entries in $W$ and $b$ are independent Gaussian random variables. We sample latent variables $x$ and sample $y \sim p(y|x)$. Then we simulate our samplers to sample the latent variables $x$ from the posterior $\pi(x) = p(x|y)$.

We consider three configurations
\begin{itemize}
    \item C1: $\sigma^2 = 2$
    \item C2: $\sigma^2 = 1$
    \item C3: $\sigma^2 = 0.5$
\end{itemize}
For each configuration, we simulate on three sizes:
\begin{itemize}
    \item $L=200, K=5$, sample Markov chain $x_{1:10000}$, use $x_{1:5000}$ for burn in, use $x_{5001:10000}$ to estimate expected acceptance rate, expected jump distance, effective sample size.
    \item $L=1000, K=5$, sample Markov chain $x_{1:40000}$, use $x_{1:20000}$ for burn in, use $x_{20001:40000}$ to estimate expected acceptance rate, expected jump distance, effective sample size.
    \item $L=4000, K=5$, sample Markov chain $x_{1:100000}$, use $x_{1:50000}$ for burn in, use $x_{50001:100000}$ to estimate expected acceptance rate, expected jump distance, effective sample size.
\end{itemize}

We give the scatter plot of $(a, d)$ and $(a, r)$ in figure~\ref{fig:fhmm_all}. And we examine the performance of our adaptive algorithm in table~\ref{tab:fhmm_d_all} and table~\ref{tab:fhmm_e_all}.
\begin{figure}
    \centering
    \includegraphics[width=\textwidth]{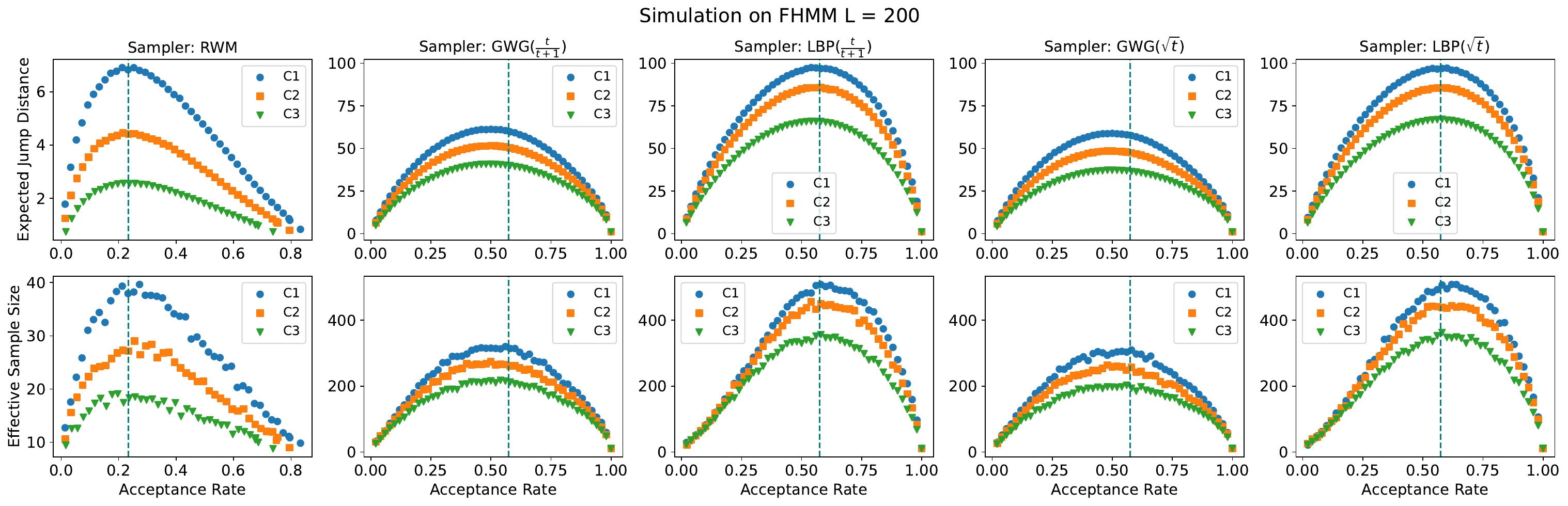}
    \includegraphics[width=\textwidth]{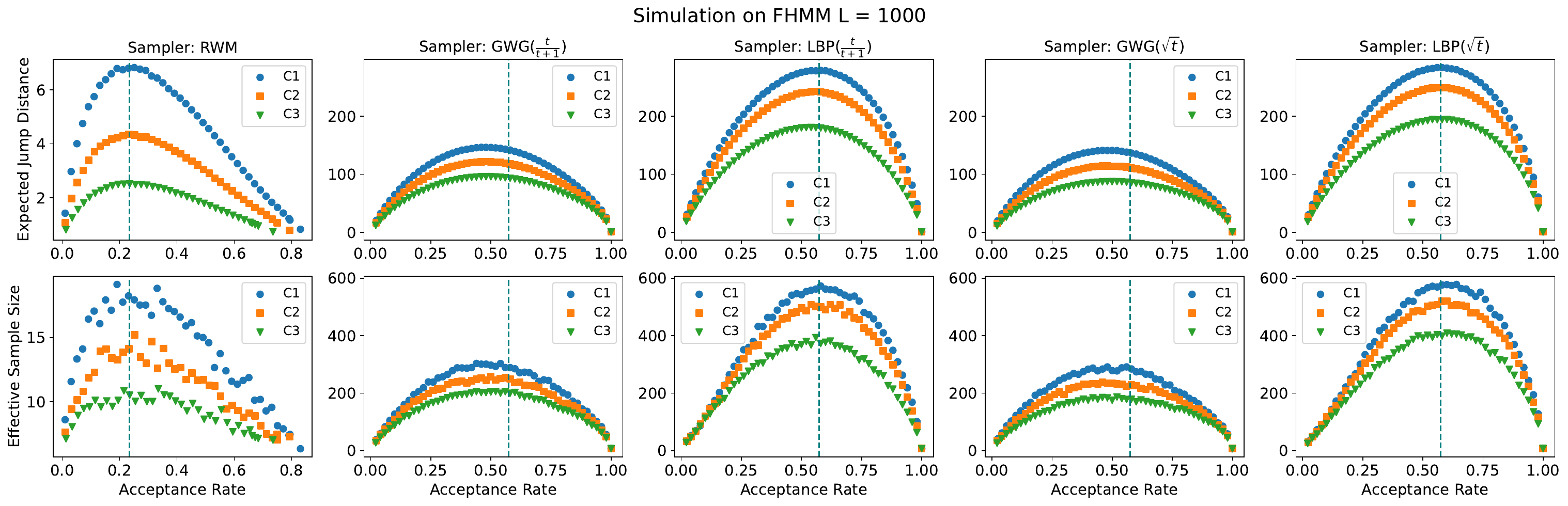}
    \includegraphics[width=\textwidth]{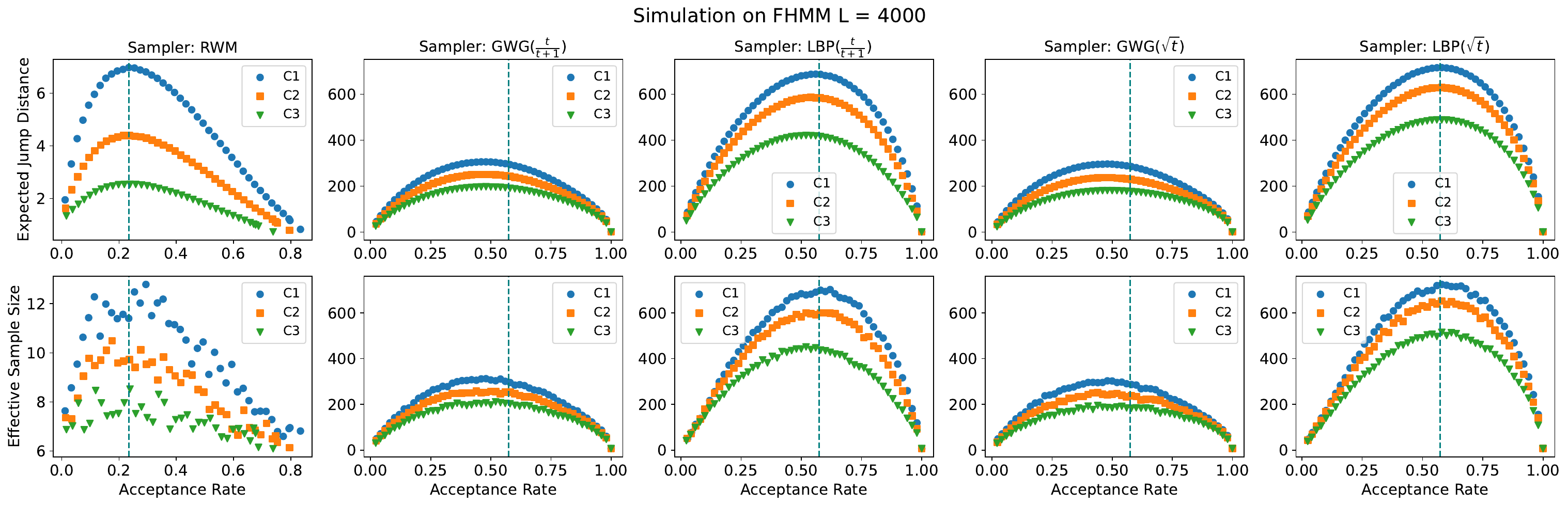}
    \caption{Simulation Results on FHMM}
    \label{fig:fhmm_all}
\end{figure}

\begin{table}[htb]
\footnotesize
\centering
\begin{tabular}{crrrrrrrrr}
\toprule
Size  & \multicolumn{3}{c}{$L=200$} & \multicolumn{3}{c}{$L=1000$} & \multicolumn{3}{c}{$L=4000$} \\
\cmidrule(lr){1-1} \cmidrule(lr){2-4} \cmidrule(lr){5-7}  \cmidrule(lr){8-10}
Sampler & C1 & C2 & C3 & C1 & C2 & C3 & C1 & C2 & C3 \\
\midrule
RWM-1 & 0.83 & 0.79 & 0.74 & 0.83 & 0.79 & 0.73 & 0.83 & 0.80 & 0.74 \\
ARWM & 6.85 & 4.41 & 2.55 & 6.79 & 4.32 & 2.50 & 6.94 & 4.40 & 2.54 \\
GRWM & 6.91 & 4.45 & 2.56 & 6.83 & 4.35 & 2.51 & 6.97 & 4.39 & 2.54 \\
GWG($\frac{t}{t+1}$)-1 & 1.00 & 1.00 & 1.00 & 1.00 & 1.00 & 1.00 & 1.00 & 1.00 & 1.00 \\
AGWG($\frac{t}{t+1}$) & 59.96 & 50.53 & 40.05 & 142.24 & 118.18 & 93.44 & 294.85 & 243.23 & 191.70 \\
GGWG($\frac{t}{t+1}$) & 61.24 & 51.63 & 40.74 & 146.40 & 121.91 & 95.88 & 305.78 & 251.09 & 197.28 \\
LBP($\frac{t}{t+1}$)-1 & 1.00 & 1.00 & 1.00 & 1.00 & 1.00 & 1.00 & 1.00 & 1.00 & 1.00 \\
ALBP($\frac{t}{t+1}$) & 97.11 & 85.70 & 65.77 & 278.20 & 242.01 & 179.51 & 687.29 & 585.02 & 416.32 \\
GLBP($\frac{t}{t+1}$) & 97.47 & 85.78 & 65.97 & 278.59 & 242.52 & 180.60 & 687.74 & 586.70 & 420.53 \\
GWG($\sqrt{t}$)-1 & 1.00 & 1.00 & 1.00 & 1.00 & 1.00 & 1.00 & 1.00 & 1.00 & 1.00 \\
AGWG($\sqrt{t}$) & 57.52 & 47.55 & 36.53 & 137.11 & 111.40 & 85.78 & 286.03 & 230.52 & 177.22 \\
GGWG($\sqrt{t}$) & 58.83 & 48.53 & 37.21 & 141.15 & 114.24 & 87.45 & 296.62 & 237.35 & 180.79 \\
LBP($\sqrt{t}$)-1 & 1.00 & 1.00 & 1.00 & 1.00 & 1.00 & 1.00 & 1.00 & 1.00 & 1.00 \\
ALBP($\sqrt{t}$) & 96.64 & 85.77 & 66.91 & 283.12 & 248.95 & 193.99 & 715.20 & 627.85 & 488.14 \\
GLBP($\sqrt{t}$) & 97.14 & 85.58 & 67.14 & 283.51 & 248.84 & 194.34 & 716.13 & 628.80 & 488.50 \\
\bottomrule
\end{tabular}
\caption{Expected Jump Distance on FHMM}
\label{tab:fhmm_d_all}
\end{table}

\begin{table}[htb]
\footnotesize
\centering
\begin{tabular}{crrrrrrrrr}
\toprule
Size  & \multicolumn{3}{c}{$L=200$} & \multicolumn{3}{c}{$L=1000$} & \multicolumn{3}{c}{$L=4000$} \\
\cmidrule(lr){1-1} \cmidrule(lr){2-4} \cmidrule(lr){5-7}  \cmidrule(lr){8-10}
Sampler & C1 & C2 & C3 & C1 & C2 & C3 & C1 & C2 & C3 \\
\midrule
RWM-1 & 9.83 & 8.98 & 8.78 & 6.33 & 7.26 & 7.01 & 6.82 & 6.14 & 6.09 \\
ARWM & 35.88 & 28.73 & 18.09 & 18.45 & 13.32 & 10.68 & 10.99 & 10.04 & 8.37 \\
GRWM & 39.65 & 29.04 & 19.09 & 19.15 & 15.22 & 11.00 & 12.79 & 10.49 & 8.52 \\
GWG($\frac{t}{t+1}$)-1 & 10.78 & 10.43 & 9.96 & 7.13 & 6.91 & 7.22 & 6.50 & 5.82 & 6.69 \\
AGWG($\frac{t}{t+1}$) & 306.97 & 262.30 & 213.37 & 288.52 & 245.50 & 196.80 & 295.08 & 241.20 & 195.35 \\
GGWG($\frac{t}{t+1}$) & 320.12 & 273.95 & 217.62 & 303.42 & 257.10 & 205.59 & 312.35 & 257.75 & 210.94 \\
LBP($\frac{t}{t+1}$)-1 & 10.70 & 10.26 & 10.58 & 7.25 & 7.25 & 7.05 & 5.97 & 6.34 & 6.76 \\
ALBP($\frac{t}{t+1}$) & 499.13 & 455.94 & 352.52 & 573.35 & 487.63 & 383.97 & 679.06 & 594.70 & 436.70 \\
GLBP($\frac{t}{t+1}$) & 508.67 & 456.24 & 356.10 & 572.88 & 508.27 & 393.32 & 702.34 & 600.21 & 451.36 \\
GWG($\sqrt{t}$)-1 & 10.30 & 10.22 & 11.06 & 6.69 & 7.90 & 7.09 & 6.53 & 6.72 & 6.80 \\
AGWG($\sqrt{t}$) & 294.38 & 251.57 & 190.81 & 278.17 & 227.18 & 186.24 & 289.14 & 232.77 & 184.94 \\
GGWG($\sqrt{t}$) & 309.79 & 264.26 & 202.46 & 291.43 & 238.06 & 186.76 & 302.98 & 251.14 & 194.88 \\
LBP($\sqrt{t}$)-1 & 9.86 & 10.52 & 10.82 & 6.98 & 7.24 & 7.68 & 6.05 & 6.59 & 6.52 \\
ALBP($\sqrt{t}$) & 502.23 & 443.64 & 348.77 & 575.49 & 524.64 & 406.15 & 727.50 & 645.50 & 504.30 \\
GLBP($\sqrt{t}$) & 508.85 & 444.36 & 362.72 & 578.29 & 520.92 & 408.98 & 724.64 & 651.81 & 515.16 \\
\bottomrule
\end{tabular}
\caption{Effective Sample Size on FHMM}
\label{tab:fhmm_e_all}
\end{table}

\begin{table}[htb]
\footnotesize
\centering
\begin{tabular}{crrrrrrrrr}
\toprule
Size  & \multicolumn{3}{c}{$L=200$} & \multicolumn{3}{c}{$L=1000$} & \multicolumn{3}{c}{$L=4000$} \\
\cmidrule(lr){1-1} \cmidrule(lr){2-4} \cmidrule(lr){5-7}  \cmidrule(lr){8-10}
Sampler & C1 & C2 & C3 & C1 & C2 & C3 & C1 & C2 & C3 \\
\midrule
RWM-1 & 136.30 & 129.16 & 30.23 & 58.29 & 58.03 & 61.32 & 112.92 & 112.73 & 110.94 \\
ARWM & 139.53 & 138.36 & 29.94 & 60.08 & 60.02 & 58.61 & 120.83 & 119.95 & 120.38 \\
GRWM & 137.61 & 123.72 & 30.14 & 66.84 & 61.19 & 58.43 & 120.42 & 120.39 & 118.40 \\
GWG($\frac{t}{t+1}$)-1 & 49.54 & 48.86 & 56.84 & 113.86 & 112.89 & 82.67 & 282.79 & 282.29 & 281.53 \\
AGWG($\frac{t}{t+1}$) & 48.80 & 64.40 & 68.73 & 119.27 & 124.07 & 88.57 & 315.22 & 315.52 & 313.65 \\
GGWG($\frac{t}{t+1}$) & 49.76 & 48.60 & 57.24 & 118.77 & 118.59 & 88.57 & 315.45 & 314.77 & 315.27 \\
LBP($\frac{t}{t+1}$)-1 & 53.94 & 69.11 & 75.98 & 129.92 & 134.42 & 91.59 & 295.47 & 294.56 & 292.91 \\
ALBP($\frac{t}{t+1}$) & 43.73 & 57.83 & 64.59 & 92.14 & 129.28 & 93.24 & 315.10 & 327.57 & 308.84 \\
GLBP($\frac{t}{t+1}$) & 57.84 & 57.21 & 63.94 & 136.68 & 140.43 & 100.02 & 315.52 & 314.08 & 309.44 \\
GWG($\sqrt{t}$)-1 & 231.12 & 196.52 & 56.26 & 112.63 & 110.48 & 109.11 & 279.85 & 279.51 & 277.10 \\
AGWG($\sqrt{t}$) & 209.90 & 218.15 & 55.81 & 113.70 & 119.34 & 114.59 & 964.08 & 314.70 & 314.63 \\
GGWG($\sqrt{t}$) & 218.94 & 218.49 & 55.59 & 116.99 & 117.65 & 112.57 & 314.86 & 314.45 & 311.70 \\
LBP($\sqrt{t}$)-1 & 256.23 & 248.33 & 62.95 & 147.23 & 128.65 & 121.14 & 1069.63 & 945.32 & 292.23 \\
ALBP($\sqrt{t}$) & 57.65 & 57.08 & 63.78 & 133.30 & 130.70 & 98.33 & 313.61 & 311.16 & 308.09 \\
GLBP($\sqrt{t}$) & 232.46 & 230.29 & 64.57 & 129.31 & 153.27 & 130.77 & 315.30 & 313.04 & 309.27 \\
\bottomrule
\end{tabular}
\caption{Running Time on FHMM}
\label{tab:fhmm_t_all}
\end{table}

\subsection{Simulation on RBM}
\label{sec:rbm}
RBM is a bipartite latent-variable model, defining a distribution over binary data $x \in \{0, 1\}^N$ and latent data $z \in \{0, 1\}^h$. Given parameters $W \in \mathbb{R}^{h \times N}, b \in \mathbb{R}^N, c \in \mathbb{R}^h$, the distribution of observable variables $x$ is obtained by marginalizing the latent variables $z$:
\begin{equation}
    \pi(x) \propto \exp(b^T x) \prod_{i=1}^h (1 + \exp(W_i x + c_i))
\end{equation}
We train the RBM on MNIST dataset using contrastive divergence \citep{hinton2002training} in three configurations 
\begin{itemize}
    \item C1: $h = 100$
    \item C2: $h = 400$
    \item C3: $h = 1000$
\end{itemize}
For each configuration, we sample Markov chain $x_{1:40000}$, use $x_{1:20000}$ for burn in, use $x_{20001:40000}$ to estimate expected acceptance rate, expected jump distance, effective sample size.

We give the scatter plot of $(a, d)$ and $(a, r)$ in figure~\ref{fig:rbm_all}. And we examine the performance of our adaptive algorithm in table~\ref{tab:rbm_d_all} and table~\ref{tab:rbm_e_all}.
\begin{figure}
    \centering
    \includegraphics[width=\textwidth]{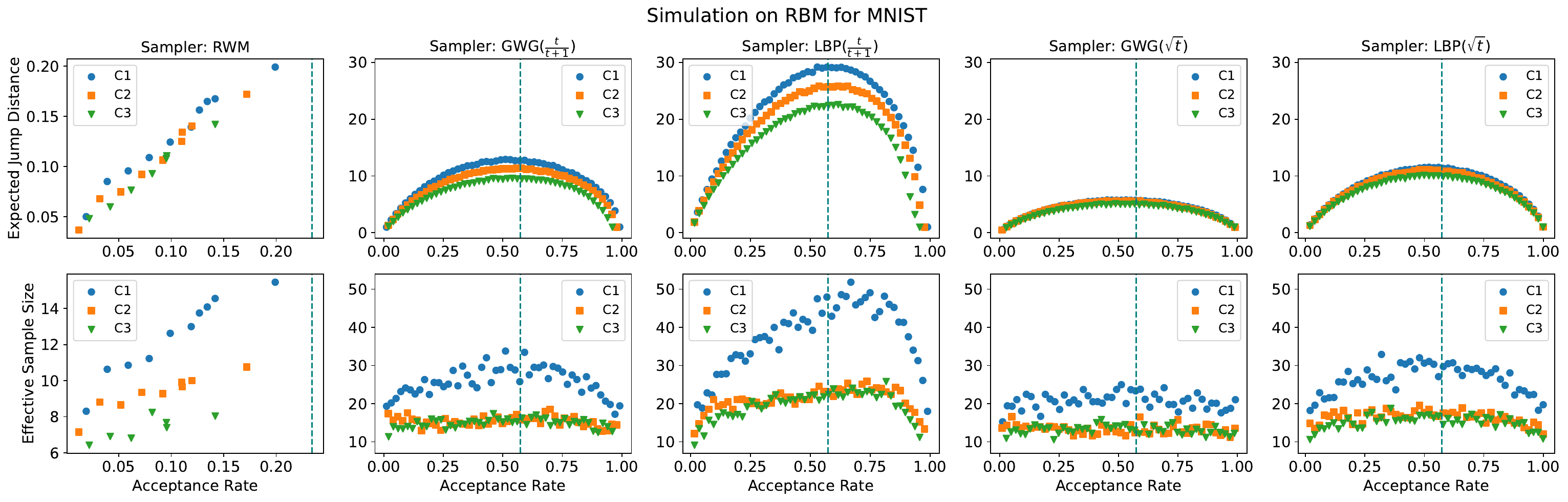}
    \caption{Simulation Results on RBM}
    \label{fig:rbm_all}
\end{figure}

\begin{table}[htb]
\centering
\begin{tabular}{crrr}
\toprule
Size  & $h=100$ & $h=400$ & $h=1000$ \\
\midrule
RWM-1 & 0.20 & 0.17 & 0.14 \\
ARWM & 0.19 & 0.17 & 0.14 \\
GRWM & 0.20 & 0.17 & 0.14 \\
GWG($\frac{t}{t+1}$)-1 & 0.99 & 0.98 & 0.96 \\
AGWG($\frac{t}{t+1}$) & 12.75 & 11.31 & 9.62 \\
GGWG($\frac{t}{t+1}$) & 12.91 & 11.36 & 9.58 \\
LBP($\frac{t}{t+1}$)-1 & 0.99 & 0.98 & 0.96 \\
ALBP($\frac{t}{t+1}$) & 29.03 & 26.07 & 22.47 \\
GLBP($\frac{t}{t+1}$) & 29.19 & 25.85 & 22.55 \\
GWG($\sqrt{t}$)-1 & 0.99 & 0.99 & 0.99 \\
AGWG($\sqrt{t}$) & 5.74 & 5.50 & 5.04 \\
GGWG($\sqrt{t}$) & 5.76 & 5.58 & 5.10 \\
LBP($\sqrt{t}$)-1 & 1.00 & 1.00 & 1.00 \\
ALBP($\sqrt{t}$) & 11.41 & 10.65 & 9.93 \\
GLBP($\sqrt{t}$) & 11.53 & 11.09 & 10.10 \\
\bottomrule
\end{tabular}
\caption{Expected Jump Distance on RBM}
\label{tab:rbm_d_all}
\end{table}

\begin{table}[htb]
\centering
\begin{tabular}{crrr}
\toprule
Size  & $h=100$ & $h=400$ & $h=1000$ \\
\midrule
RWM-1 & 15.46 & 10.76 & 8.04 \\
ARWM & 15.08 & 11.13 & 8.82 \\
GRWM & 15.46 & 10.76 & 8.24 \\
GWG($\frac{t}{t+1}$)-1 & 19.42 & 14.45 & 12.70 \\
AGWG($\frac{t}{t+1}$) & 31.71 & 16.42 & 16.21 \\
GGWG($\frac{t}{t+1}$) & 33.77 & 18.51 & 17.36 \\
LBP($\frac{t}{t+1}$)-1 & 17.99 & 13.38 & 11.16 \\
ALBP($\frac{t}{t+1}$) & 48.20 & 25.59 & 23.61 \\
GLBP($\frac{t}{t+1}$) & 51.82 & 25.83 & 25.78 \\
GWG($\sqrt{t}$)-1 & 21.03 & 13.59 & 12.20 \\
AGWG($\sqrt{t}$) & 21.92 & 13.51 & 15.52 \\
GGWG($\sqrt{t}$) & 24.97 & 16.58 & 15.81 \\
LBP($\sqrt{t}$)-1 & 19.72 & 12.02 & 10.77 \\
ALBP($\sqrt{t}$) & 33.43 & 16.95 & 17.28 \\
GLBP($\sqrt{t}$) & 32.90 & 19.51 & 18.74 \\
\bottomrule
\end{tabular}
\caption{Effective Sample Size on RBM}
\label{tab:rbm_e_all}
\end{table}

\begin{table}[htb]
\centering
\begin{tabular}{crrr}
\toprule
Size  & $h=100$ & $h=400$ & $h=1000$ \\
\midrule
RWM-1 & 67.37 & 59.54 & 43.48 \\
ARWM & 69.71 & 61.24 & 42.28 \\
GRWM & 67.37 & 59.54 & 44.52 \\
GWG($\frac{t}{t+1}$)-1 & 92.75 & 93.28 & 69.18 \\
AGWG($\frac{t}{t+1}$) & 99.34 & 95.26 & 74.14 \\
GGWG($\frac{t}{t+1}$) & 94.70 & 96.48 & 73.09 \\
LBP($\frac{t}{t+1}$)-1 & 118.60 & 116.04 & 87.54 \\
ALBP($\frac{t}{t+1}$) & 84.03 & 144.03 & 90.43 \\
GLBP($\frac{t}{t+1}$) & 121.71 & 119.38 & 91.12 \\
GWG($\sqrt{t}$)-1 & 109.86 & 94.40 & 69.86 \\
AGWG($\sqrt{t}$) & 94.97 & 94.63 & 72.30 \\
GGWG($\sqrt{t}$) & 96.42 & 94.23 & 72.93 \\
LBP($\sqrt{t}$)-1 & 116.63 & 116.15 & 86.23 \\
ALBP($\sqrt{t}$) & 120.47 & 118.61 & 88.35 \\
GLBP($\sqrt{t}$) & 118.45 & 118.86 & 89.34 \\
\bottomrule
\end{tabular}
\caption{Running Time on RBM}
\label{tab:rbm_t_all}
\end{table}

\end{document}